\documentclass{article}

\usepackage{microtype}
\usepackage{graphicx}
\usepackage{subfigure}
\usepackage{booktabs} 

\usepackage{hyperref}

\usepackage[accepted]{icml2024}

\usepackage[utf8]{inputenc} 
\usepackage[T1]{fontenc}    
\usepackage{url,wrapfig}            
\usepackage{amsfonts}       
\usepackage{nicefrac}       
\usepackage{xcolor}         
\usepackage[inline]{enumitem}

\usepackage{bm}
\usepackage{subcaption}
\usepackage{wrapfig, framed}
\usepackage{multirow}
\hypersetup{colorlinks,linkcolor={blue},citecolor={blue},urlcolor={black}}  

\usepackage{dsfont}
\usepackage{amsthm}
\usepackage{amsmath}
\usepackage{amssymb}
\usepackage{mathtools}     
\pdfoutput=1

\theoremstyle{plain}
\newtheorem{theorem}{Theorem}[section]
\newtheorem{proposition}[theorem]{Proposition}
\newtheorem{lemma}[theorem]{Lemma}
\newtheorem{corollary}[theorem]{Corollary}
\theoremstyle{definition}
\newtheorem{definition}[theorem]{Definition}
\newtheorem{assumption}[theorem]{Assumption}
\theoremstyle{remark}
\newtheorem{remark}[theorem]{Remark}
\theoremstyle{condition}
\newtheorem{condition}[theorem]{Condition}
\theoremstyle{Question}
\newtheorem{question}[theorem]{Question}

\DeclareMathOperator*{\argmax}{argmax} 
\DeclareMathOperator*{\argmin}{argmin}

\definecolor{mygreen}{rgb}{0.0, 0.7, 0.0}

\usepackage[capitalize,noabbrev]{cleveref}

\usepackage[textsize=tiny]{todonotes}

\icmltitlerunning{A Fixed-Point Approach for Causal Generative Modeling}

\begin{document}

\twocolumn[
\icmltitle{A Fixed-Point Approach for Causal Generative Modeling}




\begin{icmlauthorlist}
\icmlauthor{Meyer Scetbon}{msr}
\icmlauthor{Joel Jennings}{deepg}
\icmlauthor{Agrin Hilmkil}{msr}
\icmlauthor{Cheng Zhang}{msr}
\icmlauthor{Chao Ma}{msr}
\end{icmlauthorlist}

\icmlaffiliation{msr}{Microsoft Research}
\icmlaffiliation{deepg}{Google DeepMind}

\icmlcorrespondingauthor{Meyer Scetbon}{t-mscetbon@microsoft.com}

\icmlkeywords{Machine Learning, ICML}

\vskip 0.3in
]



\printAffiliationsAndNotice{}  

\begin{abstract}
%
%
%
We propose a novel formalism for describing Structural Causal Models (SCMs) as fixed-point problems on causally ordered variables, eliminating the need for Directed Acyclic Graphs (DAGs), and establish the weakest known conditions for their unique recovery given the topological ordering (TO).
%
Based on this, we design a two-stage causal generative model that first infers in a zero-shot manner a valid TO from observations, and then learns the generative SCM on the ordered variables. 
To infer TOs, we propose to amortize the learning of TOs on synthetically generated datasets by sequentially predicting the leaves of graphs seen during training.
To learn SCMs, we design a transformer-based architecture that exploits a new attention mechanism enabling the modeling of causal structures, and show that this parameterization is consistent with our formalism.
Finally, we conduct an extensive evaluation of each method individually, and show that when combined, our model outperforms various baselines on generated out-of-distribution problems.
The code is available on \href{https://github.com/microsoft/causica/tree/main/research_experiments/fip}{Github}.

\end{abstract}

\setlength{\textfloatsep}{10pt}
\section{Introduction}

Recent machine learning (ML) works have witnessed a flurry of activity around causal modeling~\cite{peters2017elements}. SCMs and their associated DAGs provide a complete framework to describe the data generation process, and enable proactive interventions in this process to generate the effects on the data. Such unique properties offer a comprehensive understanding of the underlying generation process, which have made them popular in various fields such as e.g. economics \citep{zhang2006extensions, battocchi2021estimating}, biology~\cite{van2006application} genetics~\cite{sachs2005causal} or healthcare \citep{bica2019estimating, huang2021diagnosis}.

In most ML settings, only observational data are available, and as a result, the recovery of SCMs and their associated DAGs from observations has become one of the most fundamental tasks in causal ML~\cite{pearl2009causality}. However, this inverse problem suffers from several limitations that arise mainly from its computational and modeling aspects, making it difficult to solve. Computationally, the combinatorial nature of the DAG space makes DAG learning an NP-hard problem~\cite{chickering2004large}. Besides, an SCM relies on functions satisfying a DAG structure to define its causal mechanisms. Consequently, the modeling of these functions depends on an \emph{unknown} DAG making the SCM recovery an ill-posed problem in general~\cite{bongers2021foundations}.

Despite these challenges, numerous approaches have been proposed in the literature. Several works have studied the DAG search problem ~\cite{maxwell1997efficient, zheng2018dags, lachapelle2019gradient,charpentier2022differentiable,kamkari2023ocdaf}, thus grappling with its NP-hard complexity. New methods propose to bypass this search by directly amortizing the inference of DAGs from observations on generated datasets~\cite{lorch2022amortized, ke2022learning}, but do not guarantee to predict DAGs. Prior research has also explored modeling SCMs using deep learning techniques~\cite{kocaoglu2017causalgan,pawlowski2020deep,xia2021causal}, but assumes the knowledge of the causal graph. More recently, causality research has led to the emergence of autoregressive flows~\cite{khemakhem2021causal, NEURIPS2023_b8402301} to learn the generative processes induced by SCMs using triangular monotonic increasing (TMI) maps. However, these flows fail at modeling exactly SCMs, leaving aside the underlying causal system.

\textbf{Our Contributions.} 
In this work, we introduce a new framework to learn SCMs from data without instantiating DAGs. By formulating SCMs as fixed point problems on the causally ordered variables, we design a specific attention-based architecture enabling them to be parameterized and learned from data given only the topological order (TO). To recover TOs, we propose to amortize the learning of a zero-shot TO inference method on generated datasets, thus by-passing the NP-hard search over the permutations and enabling their predictions at scale. When combined, these two models provide a complete framework to learn SCMs from observations. We summarize our contributions below.
\begin{itemize}
[leftmargin=.3cm,itemsep=.0cm,topsep=0cm,parsep=2pt]
    \item In section~\ref{sec:scm}, we introduce a new definition of SCMs that eliminates the need for DAGs by framing them as fixed-point problems on ordered variables, and we show its equivalence with the standard one. We also exhibit three important cases where they can be uniquely recovered from observations given the TO. To the best of our knowledge, we obtain the weakest conditions to ensure their recovery from observations when the TO is known.
    \item Rather than searching for the TO in the set of permutations, we propose in section~\ref{sec:topological} to amortize the learning of a zero-shot TO inference method from observations on synthetically generated datasets. To further reduce the complexity of this task, we learn to sequentially predict the leaves of the graphs seen during training. 
    \item In section~\ref{sec:scm_learning}, we introduce our attention-based architecture to parameterize fixed-point SCMs on the causally ordered nodes. The proposed model is an autoencoder exploiting a new attention mechanism to learn causal structures, and we show its consistency with our formalism. 
    \item Finally, in section~\ref{sec:experiment}, we evaluate the performance of each proposed model individually, and compare our final causal generative model, obtained by combining them, against various SoTA methods on both causal discovery and inference tasks. We show that our approach consistently outperforms others on generated out-of-distribution datasets.
\end{itemize}

\subsection{Related Work}
\textbf{Causal Learning through Amortization.} The unsupervised nature of the inverse problem posed by the SCM recovery task, makes causal learning a non-convex and NP-hard optimization problem~\cite{chickering2004large}. To bypass this limitation, \citet{lorch2022amortized} leverage amortization techniques to predict causal structures from observations in a supervised manner. More specifically, they propose to randomly generate synthetic SCMs to build pairs of observational samples and target DAGs, and train a transformer-based architecture to predict the DAGs from the samples. While amortization circumvents the original graph search problem, acyclicity is not guaranteed. In addition, the method aims to correctly predict full DAGs, thus suffering from a quadratic complexity w.r.t the number of variables. Here, we propose 
to drastically reduce the complexity of the amortized DAG inference approach 
by amortizing the inference of topological orders in a sequential manner instead. More precisely, we propose to sequentially infer the leaves of the DAGs given observational samples, from which we deduce the topological ordering. Our procedure is guaranteed to produce a permutation, while only seeking to infer leaves, thus enabling its application at scale.

\textbf{Causal Normalizing Flows.}~\citet{khemakhem2021causal} first introduced the connections between SCMs and autoregressive flows (AFs). When the variables are causally ordered, the data-generating process of an SCM induces a triangular map that pushes forward the exogenous distribution of the noise to the endogenous distribution of the observations. While~\citet{khemakhem2021causal}  focus on affine AFs with additive noise,~\citet{NEURIPS2023_b8402301} generalize this viewpoint by considering instead triangular monotonic increasing maps (TMI). 
%
%
However, due to the monotonicity constraint, this framework does not provide an exact equivalence with standard SCMs that can in principle induce any triangular maps. In addition, these generating maps lack access to the structural equations defining an SCM. Instead, we propose a strict generalization of the AF setting by modeling directly the system of equations defining an SCM as a fixed-point problem on the ordered nodes. Our formalism is exactly equivalent to standard SCMs, and as a by-product recovers the generating AFs which are not constraint to be monotonic. We also generalize the 
identifiability result of~\citep[Corollary 2]{NEURIPS2023_b8402301} and show that the full SCM can be recovered under TMI assumptions.

\section{Fixed-Point Formulation of SCMs}
\label{sec:scm}
In this section, we introduce our new parameterization of SCMs that does not require DAGs, by viewing them as fixed-point problems on the causally ordered nodes. We start by recalling the standard definition of SCMs. 
Then, we present our definition of fixed-point SCMs independent of DAGs, and  we show their equivalences with standard ones. 
Finally, we exhibit three important cases where fixed-point SCMs can be uniquely recovered given the TO. 

\subsection{Standard SCMs}
\label{standard-scm}
An SCM defines the data-generating process of $d$ endogenous random variables, $\bm{X}\sim \mathbb{P}_{\bm{X}}$, from $d$ exogenous and independent random variables, $\bm{N}\sim\mathbb{P}_{\bm{N}}$, using a function $F$ and a graph $\mathcal{G}:=(\bm{V},\mathcal{E})$ where $\bm{V}:=\{1,\dots, d\}$ is a set of $d$ indices and $\mathcal{E}$ is a subset of $\bm{V}^2$ indicating the edges. More precisely, the endogenous variables $\bm{X}$ are defined by the SCM as the following assignments:
\vspace{-0.1cm}
\begin{align}
    \label{eq:scm-simplified}
    X_i := F_i(\textbf{PA}^{(i)}_{\mathcal{G}}(\bm{X}),N_i),~\quad \forall i\in\{1,\dots,d\}
\end{align}

where $\bm{X} := [X_1,\dots,X_d]$, $\bm{N}:=[N_1,\dots,N_d]$, $F:=[F_1,\dots,F_d]$ and $\textbf{PA}^{(i)}_{\mathcal{G}}(\bm{X})$\footnote{Refer to Appendix~\ref{sec:background} for a complete definition.} denotes the subset of variables in $\{X_1,\dots,X_d\}$ that are the parents of $X_i$ according to the graph $\mathcal{G}$. In the following, we denote $\mathcal{S}(\mathcal{G},\mathbb{P}_{\bm{N}},F)$ the standard SCM associated to $\mathcal{G}$, $\mathbb{P}_{\bm{N}}$, and $F$.

\textbf{Topological Ordering.} In an SCM, the graph $\mathcal{G}$ is assumed to be directed and acyclic (DAG)\footnotemark[\value{footnote}], and for such graphs, it is always possible to causally order the nodes. More formally, there exists a permutation $\pi$, i.e. a bijective mapping  $\pi: \{1,\dots, d \}  \to \{1,\dots, d \}, $ satisfying $\pi(i) < \pi(j)$ if $j$ is a parent of $i$. We call such permutation a topological order (TO). We also denote $P_\pi$ the permutation matrix associated, defined as $[P_\pi]_{i,j}=1$ if $\pi^{-1}(i)=j$ and 0 otherwise, and $\Sigma_d$ the set of permutation matrices of size $d$.

\textbf{Assumptions.} In the following, we assume that (i) all the variables $X_i$ and $N_i$ live in $\mathbb{R}$, (ii) the functions $F_i$ are differentiable, and (iii) structural minimality holds. For a precise statement of these assumptions, refer to Appendix~\ref{sec:background}.

\begin{remark}
Note that under the differentibility assumption, the structural minimality assumption simply ensures that if an edge $(i,j)\in\mathcal{E}$ exists in the graph $\mathcal{G}$, then $\frac{\partial F_j}{\partial x_i} \neq \bm{0}$.
\end{remark}

\subsection{Fixed-Point SCMs Without DAGs}
Before introducing our new definition of SCMs, we need to establish some notations. 

\textbf{Notations.} 
Consider a Polish space $\mathcal{Z}$, we denote $\mathcal{P}(\mathcal{Z})$ the set of Borel probability measures on $\mathcal{Z}$, and for $p\geq 1$ an integer, $\mathcal{P}_p(\mathcal{Z})$ the set of p-integrable probability measures on $\mathcal{Z}$. We also denote  $\mathcal{P}(\mathcal{Z})^{\otimes d}$ the set of $d$ jointly independent distributions over $\mathcal{Z}^d$. For $z\in\mathcal{Z}$, we denote $\delta_z$ the Dirac distribution at $z$.
For $\mathbb{P}\in\mathcal{P}(\mathbb{R}^d)$, and $i\in\{1,2\}$, we denote $\Pi_{i,\mathbb{P}}:=\{\gamma\in \mathcal{P}(\mathbb{R}^d\times\mathbb{R}^d)\colon
p_i\#\gamma=\mathbb{P}\}$ where $p_1: (x,y)\in\mathbb{R}^d\times\mathbb{R}^d\to x\in\mathbb{R}^d$,  $p_2: (x,y)\in\mathbb{R}^d\times\mathbb{R}^d\to y\in\mathbb{R}^d$, and $\#$ is the push-forward operator. Next we introduce a simple structural condition on functions.
\begin{condition}
\label{cond:structure}
$H:\mathbb{R}^d\times\mathbb{R}^d\to\mathbb{R}^d$ is differentiable, and satisfies for all $x,n\in\mathbb{R}^d$
\begin{equation}
\label{eq-struc}
\begin{aligned}
    [\textnormal{Jac}_1 H(x,n)]_{i,j} &= 0,\quad \text{if}\quad j\geq i,\quad \text{and} \\ 
    [\textnormal{Jac}_2 H(x,n)]_{i,j} &= 0,\quad \text{if}\quad i\neq j,
\end{aligned}
\end{equation}
where $\textnormal{Jac}_1 H$ and $\textnormal{Jac}_2 H$ are the Jacobians of H w.r.t the first and second variables, i.e. $x$ and $n$ respectively.
\end{condition}
Let us also define the function space of interest in this paper, $$\mathcal{F}_d:=\{H:\mathbb{R}^d\times\mathbb{R}^d\to\mathbb{R}^d \text{ s.t. $H$ satisfies 
 Cond. }\ref{cond:structure}\}\;.$$ We are now ready to present our new definition of SCMs.

\begin{definition}[\textbf{Fixed-Point SCM}]
\label{def:fp-scm}
Let $P\in\Sigma_d$ a permutation matrix of size $d$, $\mathbb{P}_{\bm{N}}\in\mathcal{P}(\mathbb{R})^{\otimes d}$ a jointly independent distribution over $\mathbb{R}^d$ and $H\in\mathcal{F}_d$. Then we define the fixed-point SCM associated, denoted $\mathcal{S}_{\text{fp}}(P, \mathbb{P}_{\bm{N}}, H)$, as the following fixed-point problem on $\gamma\in \Pi_{2,\mathbb{P}_{\bm{N}}}:$
\begin{align}
\label{eq:dfp_scm}
 (P^{T}H( P\cdot, P\cdot) - p_1(\cdot, \cdot))\# \gamma = \delta_{0}.
\end{align}
\end{definition}
The fixed-point formulation becomes clear when one adopts a random variable perspective. Indeed for $(\bm{X},\bm{N})\sim\gamma\in \Pi_{2,\mathbb{P}_{\bm{N}}}$, we have that $\gamma$ is solution of~\eqref{eq:dfp_scm} i.i.f $\bm{X}$ solves $\bm{X} = P^TH( P\bm{X}, P\bm{N})$. In the following proposition we show that the solution $\gamma$ of the fixed-point SCM is unique.
\begin{proposition}
\label{prop:unique-gamma}
Let  $\mathcal{S}_{\text{fp}}(P, \mathbb{P}_{\bm{N}}, H)$ a fixed-point SCM as defined in definition~\ref{def:fp-scm}. Then the fixed-point problem~\eqref{eq:dfp_scm} on $\gamma\in \Pi_{2,\mathbb{P}_{\bm{N}}}$ admits a unique solution.
\end{proposition}

The above result ensures that a fixed-point SCM entails a unique coupling $\gamma$, and as a direct consequence a unique observational distribution $\mathbb{P}_{\bm{X}}:=p_1\#\gamma$. In the following we denote $\gamma(P, \mathbb{P}_{\bm{N}}, H)\in \Pi_{2,\mathbb{P}_{\bm{N}}}$ the solution of~\eqref{eq:dfp_scm}. 

Now, observe that in our definition of fixed-point SCMs, we do not use a DAG to define the structure of the function $H$. In fact, $H$ has a simple structure given by Cond.~\ref{cond:structure} and we can easily define the causal graph from it.

\begin{definition}[Causal Graph of Fixed-Point SCM] 
\label{def:causal-graph}
Let $\mathcal{S}_{\text{fp}}(P, \mathbb{P}_{\bm{N}}, H)$ a fixed-point SCM. Then we say that $j$ is a parent of $i$ if $(x,n)\to [\text{Jac}_1 P^T H(Px,Pn)]_{i,j} \neq \bm{0}$.
\end{definition}

\begin{remark}
Note that as $H$ has to satisfy Cond.~\ref{cond:structure}, then the graph induced by definition~\ref{def:causal-graph} is necessarily a DAG. 
\end{remark}

\textbf{Equivalent Formulation.} In the next Proposition, we finally show the equivalence between our formalism and the standard definition of SCMs introduced in section~\ref{standard-scm}.

\begin{proposition}
\label{prop:equivalence}
Let $\mathcal{S}(\mathcal{G},\mathbb{P}_{\bm{N}}, F)$ an SCM as defined in~\eqref{eq:scm-simplified}, $\pi$ a topological ordering of $\mathcal{G}$, and $P_{\pi}\in\Sigma_d$ the associated permutation matrix. Then, there exists a unique fixed-point SCM of the form $\mathcal{S}_{\text{fp}}(P_\pi,\mathbb{P}_{\bm{N}},H)$
such that for all $i\in\{1,\dots,d\}$, and $x,n\in\mathbb{R}^d$, 
\begin{align}
\label{eq:unique-equi}
[P_\pi^T H(P_\pi x,P_\pi n)]_i = F_i(\textbf{PA}_{\mathcal{G}}^{(i)}(x),n_i)\; .
\end{align}
Reciprocally, for any fixed-point SCM with TO $P$, there exists a unique SCM as defined in section~\ref{standard-scm} with same noise distribution  such that $P$ is a valid TO of its assoacited DAG and~\eqref{eq:unique-equi} is satisfied.
\end{proposition} 

\subsection{Partial Recovery of Fixed-Point SCMs}
\label{sec:partial-id}

Now, we investigate the partial recovery of fixed-point SCMs from observations, that is when the TO is given. To do so, let us introduce some clarifying notations. For any $\mathcal{A}\subset \mathcal{P}(\mathbb{R}^d)$, we denote $\mathcal{A}_c\subset \mathcal{A}$, the subset of distributions in $\mathcal{A}$ which are absolutely continuous w.r.t Lebesgue, and $\mathcal{A}_{cc}\subset \mathcal{A}_c$ the subset of $\mathcal{A}_c$ with distributions that admit a continuous density. We also denote for $\mathbb{Q}\in\mathcal{P}(\mathbb{R}^d)$ and $P\in \Sigma_d$, $\mathcal{A}_P(\mathbb{Q}):=\{(P,\mathbb{P},H)\colon \mathbb{P}\in\mathcal{P}(\mathbb{R})^{\otimes d},~H\in\mathcal{F}_d,~p_1\#\gamma(P,\mathbb{P},H) = \mathbb{Q}\}$, the set of fixed-point SCMs with TO $P$ generating the observational distribution $\mathbb{Q}$. 

\textbf{Partial Recovery Problem.} Let $\mathcal{S}_{\text{fp}}(P, \mathbb{P}_{\bm{N}}, H)$ a fixed-point SCM generating $\gamma(P, \mathbb{P}_{\bm{N}}, H)$ with left marginal $\mathbb{P}_{\bm{X}}:=p_1\#\gamma(P, \mathbb{P}_{\bm{N}}, H)$. Given $P$ and $\mathbb{P}_{\bm{X}}$, can we recover uniquely the generating fixed-point SCM $\mathcal{S}_{\text{fp}}(P, \mathbb{P}_{\bm{N}}, H)$? Or more formally, is $\mathcal{A}_P(\mathbb{P}_{\bm{X}})$ a singleton?\footnote{$A_P(\mathbb{P}_{\bm{X}})$ is a singleton if and only if there exists a unique fixed-point SCM with topological order $P$ generating $\mathbb{P}_{\bm{X}}$.}




Let us now exhibit two important cases where this recovery problem admits a positive answer.


\begin{proposition}
\label{prop:anm}
Let $P\in\Sigma_d$ and $\mathbb{P}_{\bm{X}}\in\mathcal{P}_1(\mathbb{R}^d)$. Let us also denote $\mathcal{F}_d^{ANM}:=\{H\in\mathcal{F}_d\colon H(x,n)=h(x)+n\}$  and $\mathcal{A}_P^{\text{ANM}}(\mathbb{P}_{\bm{X}}):=\{(P,\mathbb{P},H)\in \mathcal{A}_P(\mathbb{P}_{\bm{X}})\colon \mathbb{P}\in\mathcal{P}_1(\mathbb{R})^{\otimes d},~ H\in\mathcal{F}_d^{ANM},~\mathbb{E}_{\bm{N}\sim \mathbb{P}}(\bm{N})=0_d\}$. Then $\mathcal{A}_P^{\text{ANM}}(\mathbb{P}_{\bm{X}})$ admits at most 1 element $\mathbb{P}_{P\bm{X}}$ a.s.
\end{proposition}

Therefore if we restrict our search to Additive Noise Models (ANMs), then given $\mathbb{P}_{\bm{X}}$ and $P$, we can recover uniquely the fixed-point SCM $\mathbb{P}_{P\bm{X}}$ a.s. It is worth noting that in Proposition~\ref{prop:anm_gene}, we show a more general result where we obtain partial recovery in the ANM case with heteroscedastic noise.

Let us now show a generalized version of the ANM case where we relax the additive form of the model by assuming instead that the functions are monotonic w.r.t the noise.
\begin{theorem}
\label{thm:id-noise}
Let $P\in\Sigma_d$, $\mathbb{P}_{\bm{N}}\in\mathcal{P}(\mathbb{R})^{\otimes d}$, $\mathbb{P}_{\bm{X}}\in\mathcal{P}(\mathbb{R}^d)$ and let us assume that $\mathbb{P}_{\bm{N}}, \mathbb{P}_{\bm{X}}\in\mathcal{P}_c(\mathbb{R}^d)$. Let us also denote $\mathcal{F}_d^{MON}:=
    \{H\in\mathcal{F}_d \colon [\textnormal{Jac}_2 H(\cdot,\cdot)]_{i,i}\geq 0,~\forall i\}$,
 and $\mathcal{A}_P^{\text{MON}}(\mathbb{P}_{\bm{N}},\mathbb{P}_{\bm{X}}):=\{(P,\mathbb{P}_{\bm{N}},H)\in \mathcal{A}_P(\mathbb{P}_{\bm{X}})\colon ~H\in\mathcal{F}_d^{MON}\}$. Then $\mathcal{A}_P^{\text{MON}}(\mathbb{P}_{\bm{N}},\mathbb{P}_{\bm{X}})$ admits at most one element $\mathbb{P}_{P\bm{X}}\otimes \mathbb{P}_{P\bm{N}}$ a.s.
\end{theorem}

\begin{remark}
Note that $\mathcal{F}_d^{ANM}\subset \mathcal{F}_d^{MON}
$ and therefore theorem~\ref{thm:id-noise} generalizes the recovery of proposition~\ref{prop:anm} when the distribution of the exogenous variables is known.
\end{remark}

This result demonstrates that the partial recovery of monotonic fixed-point SCMs is feasible when the exogenous distribution is known.
In fact, we show that, for this class of fixed-point SCMs, fixing the noise distribution $\mathbb{P}_{\bm{N}}$ is also necessary to obtain partial recovery. See Proposition~\ref{prop:existence-scm}.


While in this section we focus on the recovery of each component constituting a fixed-point SCM given the TO, in the following we investigate a weaker partial recovery problem, where we only aim at identifying uniquely the interventional and counterfactual distributions of fixed-point SCMs from observations given the TO.

\subsection{Weak Partial Recovery of Fixed-Points SCMs}
\label{sec:identifiability-general}
Before we state the problem of interest, let us first introduce some additional notations defining the interventional and counterfactual distributions of a fixed-point SCM. 

Let $\mathcal{S}_{\text{fp}}(P, \mathbb{P}_{\bm{N}}, H)$ a fixed-point SCM. For any differentiable and lower-triangular map\footnote{That is any differentiable map satisfying for any $x\in\mathbb{R}^d$, and $i,j\in\{1,\dots,d\}$, $[\text{Jac}T(x)]_{i,j}=0$ if $j>i$.} $T:\mathbb{R}^d\to\mathbb{R}^d$, we define $H_T=T\circ H$. In particular, by defining for any $a\in\mathbb{R}$ and $i\in\{1,\dots,d\}$, the diagonal map $T_{i,a}:x\in\mathbb{R}^d\to[x_1,\dots,x_{i-1},a,x_{i+1},\dots,x_d]\in\mathbb{R}^d$, then the fixed-point SCM $\mathcal{S}_{\text{fp}}(P, \mathbb{P}_{\bm{N}}, H_{T_{i,a}})$ corresponds to the intervened SCM under the intervention $\text{do}([PX]_i=a)$, that is $\mathcal{S}_{\text{fp}}(P, \mathbb{P}_{\bm{N}}, H_{T_{i,a}})=\mathcal{S}^{\text{do}([PX]_i=a)}_{\text{fp}}(P, \mathbb{P}_{\bm{N}}, H)$. We also denote $\mathbb{P}_{\bm{X}}^{\text{do}(T)}(P, \mathbb{P}_{\bm{N}}, H):=p_1\#\gamma(P, \mathbb{P}_{\bm{N}}, H_T)$, and we call it the interventional distribution of $\mathcal{S}_{\text{fp}}(P, \mathbb{P}_{\bm{N}}, H)$ under the intervention $T$. In particular, if $T=T_{i,a}$,
$\mathbb{P}_{\bm{X}}^{\text{do}(T_{i,a})}(P, \mathbb{P}_{\bm{N}}, H)$ is the interventional distribution of $\mathcal{S}_{\text{fp}}(P, \mathbb{P}_{\bm{N}}, H)$ under the intervention $\text{do}([PX]_i=a)$. 

Let us now introduce a mild condition on $H$, enabling the proper definition of the counterfactual distributions.

\begin{condition}
\label{cond:counterfactual}
Let $H\in\mathcal{F}_d$, then $H$ is $C^1$ and $H^{\text{gen}}:n\in\mathbb{R}^d\to H^{\circ d}(\cdot,n)(0_d)\in\mathbb{R}^d$ is bijective. 
\end{condition}

\begin{remark}
Note that $H^{\text{gen}}$ corresponds to the generative process associated to $H$ that maps the (ordered) exogenous distribution $\mathbb{P}_{P\bm{N}}$ to the (ordered) observable one $\mathbb{P}_{P\bm{X}}$.
\end{remark}

\begin{remark}
It is also worth noting that $0_d$ can be arbitrarily replaced by any vector $z\in\mathbb{R}^d$ in the condition above.
\end{remark}

Then, if $H$ satisfies cond.~\ref{cond:counterfactual}, 
we can define 
$$\gamma^{\text{do}(T)}(P, \mathbb{P}_{\bm{N}}, H):=(I_d, P^T \circ H_T^{\text{gen}}\circ (H^{\text{gen}})^{-1} \circ P) \#\mathbb{P}_{\bm{X}}$$ 
the counterfactual distribution\footnote{This coupling is well defined as both $H_T^{\text{gen}}$ and $H^{\text{gen}}$ are Borel measurable, and the inverse of a Borel measurable function is Borel measurable, that is $(H^{\text{gen}})^{-1}$ is also Borel measurable.} of $\mathcal{S}_{\text{fp}}(P, \mathbb{P}_{\bm{N}}, H)$ under the intervention $T$. It is also worth noting that the right marginal of $\gamma^{\text{do}(T)}(P, \mathbb{P}_{\bm{N}}, H)$ is the interventional distribution under the intervention $T$, that is
$p_2\#\gamma^{\text{do}(T)}(P, \mathbb{P}_{\bm{N}}, H)=\mathbb{P}_{\bm{X}}^{\text{do}(T)}(P, \mathbb{P}_{\bm{N}}, H)$. Let us now state the problem of interest.

\textbf{Weak Partial Recovery Problem.}  Let $\mathcal{S}_{\text{fp}}(P, \mathbb{P}_{\bm{N}}, H)$ a fixed-point SCM generating $\mathbb{P}_{\bm{X}}:=p_1\#\gamma(P, \mathbb{P}_{\bm{N}}, H)$. Given $P$ and $\mathbb{P}_{\bm{X}}$, can we recover uniquely all the interventional and counterfactual distributions? That is for any differentiable and lower-triangular transformation $T$, can we recover uniquely $\mathbb{P}^{\text{do}(T)}_{\bm{X}}(P, \mathbb{P}_{\bm{N}}, H)$ and $\gamma^{\text{do}(T)}(P, \mathbb{P}, H)$?

We answer positively to this in the following main result. 
\begin{theorem}
\label{thm-gene-identification}
Let $\mathcal{S}_{\text{fp}}(P, \mathbb{P}_{\bm{N}}, H)$ a fixed-point SCM generating $\mathbb{P}_{\bm{X}}:=p_1\#\gamma(P, \mathbb{P}_{\bm{N}}, H)$. Let us assume that 
$H$ satisfies cond.~\ref{cond:counterfactual}, $\mathbb{P}_{\bm{N}}\in\mathcal{P}_{cc}(\mathbb{R})^{\otimes d}$ and $\mathbb{P}_{\bm{X}}\in\mathcal{P}_c(\mathbb{R}^d)$. Then, 
$\mathcal{A}_P^{\text{INV}}(\mathbb{P}_{\bm{X}}):=\{(P,\mathbb{P},F)\in \mathcal{A}_P(\mathbb{P}_{\bm{X}}):~\mathbb{P}\in\mathcal{P}_{cc}(\mathbb{R})^{\otimes d},~F~\text{satisfies cond. \ref{cond:counterfactual}}\}$ is non-empty, and for any $(P,\mathbb{P},F)\in \mathcal{A}_P^{\text{INV}}(\mathbb{P}_{\bm{X}})$, and for any differentiable and lower-triangular map $T$, we have
\begin{align*}
   \mathbb{P}^{\text{do}(T)}_{\bm{X}}(P, \mathbb{P}_{\bm{N}}, H) &= \mathbb{P}^{\text{do}(T)}_{\bm{X}}(P, \mathbb{P}, F)\\
   \gamma^{\text{do}(T)}(P, \mathbb{P}_{\bm{N}}, H)&=\gamma^{\text{do}(T)}(P, \mathbb{P}, F).
\end{align*}
In addition, for any $\mathbb{Q}\in\mathcal{P}_{cc}(\mathbb{R})^{\otimes d}$, and by denoting $\mathcal{A}_P^{\text{INV}}(\mathbb{Q},\mathbb{P}_{\bm{X}}):=\{(P,\mathbb{Q},F)\in\mathcal{A}_P^{\text{INV}}(\mathbb{P}_{\bm{X}})\}$, then $\mathcal{A}_P^{\text{INV}}(\mathbb{Q},\mathbb{P}_{\bm{X}})$ is not empty. That is there always exists an invertible fixed-point SCM with exogenous distribution $\mathbb{Q}$, and topological order $P$ generating $\mathbb{P}_{\bm{X}}$ (and which has the same causal distributions).
\end{theorem}
The theorem above has two important consequences: (1) that any invertible fixed-point SCM with topological ordering $P$ generating $\mathbb{P}_{\bm{X}}$, that is any element in $\mathcal{A}_P^{\text{INV}}(\mathbb{P}_{\bm{X}})$, recovers the true interventional and counterfactual distributions uniquely, and (2) that we can arbitrarily choose the exogenous distribution to recover these causal distributions, that is we can impose $\mathbb{Q}:=\mathcal{N}(0_d,I_d)$ and recover any fixed-point SCM in $\mathcal{A}_P^{\text{INV}}(\mathbb{Q},\mathbb{P}_{\bm{X}})$ to obtain the causal distributions. These results are, to the best of our knowledge, the weakest conditions obtained to ensure recovery from observations when the TO is known.

In the rest of the paper, we aim at recovering the three components constituting a fixed-point SCM, that are $P$, $\mathbb{P}_{\bm{N}}$ and $H$, from observations only in the ANM setting  and leave the more general cases for future work. In section~\ref{sec:topological}, we describe our approach to amortize the learning of a zero-shot TO inference method, enabling the recovery of $P$ from observations only. In section~\ref{sec:scm_learning}, we present our attention-based parameterization of fixed-point SCMs, and we leverage our partial recovery results in the ANM case to recover uniquely $H$ and $\mathbb{P}_{\bm{N}}$ from observations given $P$.

\section{Amortized Learning of TO}
\label{sec:topological}
We now present the first component of our causal generative model, that aims at inferring in a zero-shot manner the topological ordering of the nodes from observational data. To do so, we propose to amortize the learning of a model trained to sequentially predict the leaves of the graphs seen during training from their corresponding observations.


\begin{algorithm}[t]
    \caption{$\text{d-TOE}(\mathcal{M},(\mathcal{D}_{\text{tr}},\mathcal{G}_{\text{tr}}))$}
   \label{alg:one-forward-pass}
\begin{algorithmic}[1]
\STATE {\bfseries Input:} $\mathcal{M}$, $(\mathcal{D}_{\text{tr}},\mathcal{G}_{\text{tr}})$
\STATE Initialize $\text{d-TOE} = 0$.
   \FOR{$q=1$ {\bfseries to} $d$}
   \STATE  $\bm{p} \gets \mathcal{M}(\mathcal{D}_{\text{tr}}),\quad \bm{y} \gets  \mathcal{L}(\mathcal{G}_{\text{tr}})$
  \STATE $\text{d-TOE} \gets \text{d-TOE} + \textbf{BN} (\bm{p},\bm{y})$ 
   \STATE $\hat{\ell} \gets \argmax_{i} [\bm{p}]_i,\quad \ell \gets \mathcal{B}(\bm{y},\hat{\ell})$
   \STATE $\mathcal{D}_{\text{tr}}\gets \mathcal{R}_1(\mathcal{D}_{\text{tr}},\ell), \quad  \mathcal{G}_{\text{tr}}\gets \mathcal{R}_2(\mathcal{G}_{\text{tr}},\ell)$
   \ENDFOR
\STATE Return $\text{d-TOE}$
\end{algorithmic}
\end{algorithm}

\textbf{Training Setting.} Given $K\geq 1$ training datasets and their associated DAGs $(\mathcal{D}_{\text{tr}}^{(1)},\mathcal{G}_{\text{tr}}^{(1)}),\dots,(\mathcal{D}_{\text{tr}}^{(K)},\mathcal{G}_{\text{tr}}^{(K)})$, obtained from $K$ synthetically generated SCMs, our goal here is to optimize a learnable architecture $\mathcal{M}$ that given the observations $\mathcal{D}_{\text{tr}}^{(k)}$ can predict a valid TO of $\mathcal{G}_{\text{tr}}^{(k)}$, and so for all $k\in\{1,\dots,K\}$. 

\textbf{Architecture.} We use the exact same encoder $\textbf{En}$ as the one proposed in~\cite{lorch2022amortized} in order to map a dataset $\mathcal{D}\in\mathbb{R}^{n\times d}$ with $n$ observational samples of $d$ endogenous variables, to a latent representation of the nodes $\textbf{En}(\mathcal{D})\in\mathbb{R}^{d\times d_h}$ where $d_h$ is the latent dimension. As we only need to predict whether a node is a leaf, we use a simple linear classifier $f$ to predict the logits of each node, given by $\mathcal{M}(\mathcal{D}):=f(\textbf{En}(\mathcal{D}))\in\mathbb{R}^d$.

\textbf{Training Procedure.} To train the model $\mathcal{M}$ to infer TOs in a zero-shot manner, we propose to successively infer the leaves of the graphs seen during training in the topological order. To formalize the procedure, let us first introduce some operators. We define $\mathcal{R}_1$ the operator that for any dataset $\mathcal{D}\in\mathbb{R}^{n\times d}$ and index $q\in\{1,\dots,d\}$ returns the same dataset where the $q$-th column has been removed, denoted $\mathcal{R}_1(\mathcal{D},q)\in\mathbb{R}^{n\times (d-1)}$. Similarly, we denote $\mathcal{R}_2$ the operator such that for graph $\mathcal{G}\in\{0,1\}^{d\times d}$ and index $q\in\{1,\dots,d\}$ returns the same graph where the $q$-th row and the $q$-th column have been removed, denoted $\mathcal{R}_2(\mathcal{G},q)\in\{0,1\}^{(d-1)\times (d-1)}$. We define also $\mathcal{L}$ the operator that for any DAG $\mathcal{G}\in\{0,1\}^{d\times d}$, returns a binary vector of size $d$ indicating its leaves $\mathcal{L}(\mathcal{G})\in\{0,1\}^d$, i.e. $\mathcal{L}(\mathcal{G})_k=1$ i.f.f $k$ is a leaf. We
denote for any binary vector $v\in\{0,1\}^d$ and index $q\in\{1,\dots,d\}$, the set $S_{v,q}:=\{k\in\{1,\dots,d\}\colon v_k=1\}$, and we 
define $\mathcal{B}$ the operator that returns a sampled index $\mathcal{B}(v,q)$ from either the Dirac distribution $\delta_{q}$ if $v_q=1$ or from the uniform distribution over $S_{v,q}$ otherwise. Finally we define the binary loss between some logits $\bm{p}:=[p_1,\dots,p_d]\in\mathbb{R}^d$ and a binary vector $\bm{y}:=[y_1,\dots,y_d]\in\{0,1\}^d$ as
$\text{BN}(\bm{p},\bm{y}):=
-\sum_{k=1}^{d} (y_k\log(\sigma(p_k)) +(1-y_k)\log(\sigma(-p_k))$
where $\sigma(x):=1/(1+\exp(-x))$ is the sigmoid function. 

We are now ready to present our training loss to learn $\mathcal{M}$. Given any pair $(\mathcal{D}_{\text{tr}}, \mathcal{G}_{\text{tr}})$, we introduce the differentiable topological ordering error (d-TOE) defined in  Algorithm~\ref{alg:one-forward-pass}, and we propose to learn $\mathcal{M}$ by minimizing 
\vspace{-0.2cm}
\begin{align*}
    \sum_{k=1}^K \text{d-TOE}(\mathcal{M},(\mathcal{D}_{\text{tr}}^{(k)}, \mathcal{G}_{\text{tr}}^{(k)}))\; .
\end{align*}
Note that the model $\mathcal{M}$ as well as all the operators involved in Alg.~\ref{alg:one-forward-pass} are fully parallelizable w.r.t the number of datasets, which allow us to compute d-TOE per batch of datasets.

\begin{remark}
While d-TOE requires $d$ successive calls of $\mathcal{M}$, the memory and time complexities of the backward passes are still linear w.r.t these, since we are not considering the gradient of either $\ell$ or $\hat{\ell}$ defined in line~$\bm{6}$ of Alg.~\ref{alg:one-forward-pass}. 
\end{remark}


\textbf{Sub-sampling of d-TOE.} In order to improve the scalability of the training procedure, we propose to compute d-TOE only on a subset of indices randomly sampled. More formally, let us denote $1\leq d_{\text{max}}\leq d$, the maximum number of indices to keep during training for computing the loss. Then for each training pair $(\mathcal{D}_{\text{tr}},\mathcal{G}_{\text{tr}})$ (or batch of pairs), we randomly sample a set of $d_{\text{max}}$ indices in $\{1,\dots,d\}$, and we only update the d-TOE in line $\bm{5}$ of Alg.~\ref{alg:one-forward-pass} if the current index $q$ of the $\textbf{for}$ loop is in the set. Note that if we choose $d_{\text{max}}=1$ the backward computation is equivalent to the one where only a single call of $\mathcal{M}$ is performed.


\textbf{Inference.} We summarize the zero-shot TO inference of the amortized model $\mathcal{M}$ in  algorithm~\ref{alg:inference}. Note that when $\mathcal{M}$ predicts a valid TO, Alg.~\ref{alg:one-forward-pass} and~\ref{alg:inference} return the same TO.

Next, we present the second component of our causal generative model, that is a new transformer-based architecture enabling the parametrization of fixed-point SCMs.
\section{Fixed-Point SCM Learning}
\label{sec:scm_learning}

Let $\mathcal{S}_{\text{fp}}(P, \mathbb{P}_{\bm{N}}, H)$ a fixed-point SCM generating $\gamma(P, \mathbb{P}_{\bm{N}}, H)$ with left marginal $\mathbb{P}_{\bm{X}}$ and let $\mathcal{D}:=[\bm{X}^{(1)},\dots, \bm{X}^{(n)}]^T$, a dataset of $n$ i.i.d. samples drawn from $\mathbb{P}_{\bm{X}}$. Our goal is to learn a generating fixed-point SCM of $\mathcal{A}_P(\mathbb{P}_{\bm{X}})$, given the samples $\mathcal{D}$ and the topological ordering $P$. To do so, we propose to parameterize $\mathcal{F}_d$ using an attention-based autoencoder. 



\subsection{Proposed Architecture}
To model $\mathcal{F}_d$ in a learnable fashion, we design an attention-based architecture that comprises four stages: (i) a high dimensional causal embedding of the ordered samples that preserves the causal structure of the generating SCM, (ii) a causal attention mechanism to model ordered DAGs, (iii) a causal encoder that parameterizes $\mathcal{F}_d$ in a latent space, and (iv) a causal decoder that brings back the encoded samples to the original space while preserving the causal structure.

\textbf{Causal Embedding.} This layer maps the samples $(\bm{X},\bm{N})\sim \gamma(P, \mathbb{P}_{\bm{N}}, H)$ into a higher dimensional space  without modifying the causal structure of the generating SCM $\mathcal{S}_{\text{fp}}(P, \mathbb{P}_{\bm{N}}, H)$. To do so, we embed the ordered samples by considering two diagonal maps $E_1, E_2:\mathbb{R}^d\to\mathbb{R}^{d\times D}$ , with $D\gg 1$ the embedding dimension, defined as
\vspace{-0.02cm}
\begin{align*}
    E_{i}(w):=[w_1 * \theta_{i,1},\dots,w_d *\theta_{i,d}]^T + \textbf{Pos}\in\mathbb{R}^{d\times D}
\end{align*}
\vspace{-0.05cm}
where $w:=[w_1,\dots,w_d]\in\mathbb{R}^d$, $\bm{\theta}_i:=[\theta_{i,1},\dots,\theta_{i,d}]^T\in\mathbb{R}^{d\times D}$ with $\theta_{i,q}\in\mathbb{R}^{D}$ some learnable parameters for $i\in\{1,2\}$ and $\textbf{Pos}\in\mathbb{R}^{d\times D}$ a learnable matrix. Then we define our embedded samples as $\bm{X}_{\text{emb}}:=E_1(P\bm{X})$ and $\bm{N}_{\text{emb}}:=E_2(P\bm{N})$. We show the law of $(\bm{X}_{\text{emb}},\bm{N}_{\text{emb}})$ is the solution of a latent fixed-point SCM with the same causal structure as the generating one $\mathcal{S}_{\text{fp}}(P, \mathbb{P}_{\bm{N}}, H)$ in proposition~\ref{prop:enc}.

\textbf{Causal Attention.} 
We propose to encode the causal graph of the ordered nodes with an attention matrix. Before doing so, let us first recall the definition of the standard attention. Given a key, and a query, denoted respectively $K,Q\in\mathbb{R}^{d\times d_{\text{head}}}$ with $d_{\text{head}}$ the dimension of a single head, and a (potential) mask $M\in\{0,+\infty\}^{d\times d}$, the attention matrix is defined as
 $\text{A}_M(Q,K):=\text{softmax}((QK^{T} - M)/\sqrt{D})$
where for $W\in\mathbb{R}^{d\times d}$, $[\text{softmax}(W)]_{i,j}:=\exp(W_{i,j})/\sum_{k}\exp(W_{i,k})$. By viewing $Q$ and $K$ as two sequences of $d$ nodes living in $\mathbb{R}^{d_{\text{head}}}$, the attention matrix can be interpreted as a continuous graph explaining the relationships between the nodes. However, the $\text{softmax}$ forces all the rows to sum to $1$, thus preventing the use of attention to model DAGs, since each node would have at least one parent. We propose to relax this constraint in the following.

\begin{definition}[Causal Attention]
For $Q,K\in\mathbb{R}^{d\times d_{\text{head}}}$, and $M\in\{0,+\infty\}^{d\times d}$  the causal attention matrix $\text{CA}_M(Q,K)$ is defined
as 
\begin{align*}
\text{CA}_M(Q,K):=\frac{\exp((QK^T - M)/\sqrt{D})}{\mathcal{V}(\exp((QK^T - M)/\sqrt{D})\mathbf{1}_d)}
\end{align*}
where for $v:=[v_1,\dots,v_d]\in\mathbb{R}^d_{+}$ and $i\in\{1,\dots,d\}$, $[\mathcal{V}(v)]_i:=v_i$ if $v_i\geq 1$ and $[\mathcal{V}(v)]_i:=1$ otherwise.
\end{definition}

Now the rows of $ \text{CA}_M(Q,K)$ can sum to any values between $[0,1]$, and therefore can be used to model any DAGs. In the following, we consider a specific masking, that is $M_{i,j}:=0$ if $i<j$ and $M_{i,j}:=+\infty$ otherwise, to encode Cond.~\ref{cond:structure} on $\text{Jac}_1 H$ in the causal attention.

\begin{remark}
It is worth noting our causal attention is a strict relaxation of the standard attention viewed as the solution of a partial and entropic optimal transport problem~\cite{cuturi2013sinkhorn}. See Appendix~\ref{sec:arch} for more details.
\end{remark}
\begin{remark}
The model uses multi-head attention, but we have presented a single head for better readability.
\end{remark}

\textbf{Causal Encoder.} To encode the embedded samples, we consider a transformer-like encoder~\cite{vaswani2017attention} using our causal attention. More formally, given the embedded samples, $(\bm{X}_{\text{emb}}, \bm{N}_{\text{emb}})$, we consider the following encoder layer defined as:
\begin{equation*}
\label{def:DL}
\begin{gathered}
\mathcal{C}(\bm{X}_{\text{emb}}, \bm{N}_{\text{emb}}):=\\
h(\text{CA}_M(\bm{N}_{\text{emb}}W_Q,\bm{X}_{\text{emb}}W_K)\bm{X}_{\text{emb}}W_V + \bm{N}_{\text{emb}})
\end{gathered}
\end{equation*}
where $W_Q, W_K, W_V\in\mathbb{R}^{D\times D}$ are learnable parameters, $h(x):= \text{LN} \circ (\text{I}_{D} + \text{MLP})\circ\text{LN}(x)$ applied point-wise on each row, and $\text{LN}$ and $\text{MLP}$ denote a layer norm operator and a multi-layer perceptron respectively. Then starting from $\bm{N}^{(0)}_{\text{emb}}:=\bm{N}_{\text{emb}}$, we compute for $k\in\{0,\dots,L-1\}$:
\begin{align*}
    \bm{N}^{(k+1)}_{\text{emb}}:=\mathcal{C}( \bm{X}_{\text{emb}}, \bm{N}^{(k)}_{\text{emb}}),
\end{align*}
where for each $k$ a new encoder layer $\mathcal{C}$ is instantiated. The causal decoder is then defined as $\mathcal{C}_L:(x,n)\to \mathcal{C}(x,\cdot)^{\circ L}(n)$, and we show that it defines a valid fixed-point SCM in the latent space. See proposition~\ref{prop-causal-enc}.

\textbf{Causal Decoder.} To decode the samples without affecting the causal structure, we propose to consider a simple parametric function $\mathcal{J}$ defined for $x:=[x_1,\dots,x_d]^T\in\mathbb{R}^{d\times D}$ as $\mathcal{J}(x):=[\langle x_1,w_1\rangle,\dots,\langle x_d,w_d\rangle]$, where, $\langle \cdot,\cdot\rangle$ denotes the inner product, and $w_i\in\mathbb{R}^D$ are learnable parameters. We show that $\mathcal{J}$ allows to preserve the causal structure of the encoded samples in proposition~\ref{prop:cond-dec}.

Finally, we can introduce our final architecture $\mathcal{T}$ defined for $x,n\in\mathbb{R}^d$ as:
 \begin{align}
 \label{eq:final-arch}
     \mathcal{T}(x,n):=\mathcal{J}\circ\mathcal{C}_L(E_{1}(x),E_{2}(n))\in\mathbb{R}^d,
 \end{align}
which is guaranteed to be in $\mathcal{F}_d$ and to preserve the causal structure of SCMs during embedding and decoding phases.

\subsection{Training and Generation}

While our architecture can parameterize complex functions in $\mathcal{F}_d$, we consider a restricted setting in this paper. More formally, we consider the learning of additive noise models of the form $\mathcal{T}_{\text{ANM}}(x,n):=\mathcal{T}(x,0) + n$ where $\mathcal{T}$ is defined as in~\eqref{eq:final-arch} and we leave the general case for future works.

\textbf{Training.} To train such a model, we propose to minimize the mean squared error (MSE), that is:
\begin{align}
\label{eq:MSE-training}
    \mathbb{E}_{\bm{X}\sim \hat{\mathbb{P}}_{\bm{X}}}\Vert \bm{X} - P^{T}\mathcal{T}_{\text{ANM}}(P\bm{X},0_d)\Vert_2^2.
\end{align} 
where $\hat{\mathbb{P}}_{\bm{X}}:=1/n\sum_{i=1}^n\delta_{\bm{X}^{(i)}}$ is the empirical distribution of the observations and $P$ is the TO of the generating SCM $\mathcal{S}_{\text{fp}}(P, \mathbb{P}_{\bm{N}}, H)$. Thanks to our partial recovery result obtained in Proposition~\ref{prop:anm}, We show in proposition~\ref{prop:anm-id-train} that in the limit of infinite samples, the minimization of our objective~\eqref{eq:MSE-training} enables to recover uniquely the generating function $H$ if $H\in\mathcal{F}_d^{ANM}$.

\begin{algorithm}[tb]
   \caption{Generative Procedure of $\mathcal{T}_{\text{ANM}}$}
   \label{alg:sampling}
\begin{algorithmic}
\STATE {\bfseries Input:} 
$\mathcal{T}_{\text{ANM}}$, $P$, $g_1,\dots,g_k$,
\STATE Initialize $\tilde{\bm{X}}=0_d$, and $U_1\dots,U_k$ $d$ i.i.d from $\mathbb{U}$
\STATE $\tilde{N}_k \gets g_k(U_k)$, 
$\forall k\in\{1,\dots,d\},\quad \tilde{\bm{N}}\gets [\tilde{N}_1,\dots,\tilde{N}_d]$
\STATE $\tilde{\bm{X}}\gets \mathcal{T}_{\text{ANM}}(\cdot,P\bm{\tilde{N}})^{\circ d}(\tilde {\bm{X}})$
\STATE $\tilde{\bm{X}}\gets P^T\tilde{\bm{X}}$
\STATE Return $(\tilde{\bm{X}}, \tilde{\bm{N}})$
\end{algorithmic}
\end{algorithm}

\textbf{Generative Modeling.} Once $\mathcal{T}_{\text{ANM}}$ is trained, we propose to estimate simple 1-dimensional functions to generate new samples from our learned SCM. To do so, we first estimate the marginals of $\mathbb{P}_{\bm{N}}$, denoted $\hat{\mathbb{P}}_{N_k}$for $k\in\{1,\dots,d\}$, using the predicted noises $\tilde{\bm{N}}^{(i)}:=\bm{X}^{(i)} - P^T\mathcal{T}_{\text{ANM}}(P\bm{X}^{(i)},0_d)$. 
Then we propose to solve the following 1-d problems:
\vspace{-0.1cm}
\begin{align}
\label{eq:simple-1-d}
\min_{g_k:\mathbb{R}\to\mathbb{R}} \text{OT}( g_k\#\hat{\mathbb{U}},\hat{\mathbb{P}}_{N_k}),\quad \forall~k\in\{1,\dots,d\},
\end{align}
where $\text{OT}$ is the optimal transport distance~\cite{villani2009optimal} and $\hat{\mathbb{U}}$ is the empirical measure of the uniform on $[0,1]$. These problems can be solved by estimating the quantile functions of each $\mathbb{P}_{N_k}$~\cite{peyre2019computational}. We summarize the generative process in Alg.~\ref{alg:sampling} and show in proposition~\ref{prop:generation-noise}, that in the limit of infinite samples, this procedure generates samples from $\gamma(P,\mathbb{P}_{\bm{N}}, H)$ if $H\in\mathcal{F}_d^{ANM}$.


\section{Experiments}
\label{sec:experiment}

We start by evaluating individually the performances of each component of our causal generative model, that are, the zero-shot TO inference method $\mathcal{M}$ in section~\ref{sec:eval-amortization}, and the fixed-point SCM parameterization $\mathcal{T}_{\text{ANM}}$ in section~\ref{sec:eval-fp-scm}. Finally we benchmark our final causal model, obtained by combining them, against various baselines in section~\ref{sec:benchmark}. 

\subsection{Evaluation of $\mathcal{M}$}
\label{sec:eval-amortization}

\textbf{Data-generating Process.} We reproduce the procedure proposed in~\cite{lorch2022amortized} to generate synthetic datasets and their associated DAGs using randomly sampled SCMs. More precisely, we consider two distributions of SCMs denoted $\mathbb{P}_{\text{IN}}$ and $\mathbb{P}_{\text{OUT}}$. In $\mathbb{P}_{\text{IN}}$, the graphs are sampled according  Erdos-Renyi~\cite{erdds1959random} and scale-free models~\cite{barabasi1999emergence}, while in $\mathbb{P}_{\text{OUT}}$, we consider Watts-Strogatz~\cite{watts1998collective} and stochastic block models ~\cite{holland1983stochastic}. We simulate homoscedastic Gaussian noise
in $\mathbb{P}_{\text{IN}}$ but consider heteroscedastic Laplacian noise in $\mathbb{P}_{\text{OUT}}$. Finally both $\mathbb{P}_{\text{IN}}$ and $\mathbb{P}_{\text{OUT}}$ use randomly sampled linear (LIN) and nonlinear functions of random Fourier features (RFF) to model functional relationships, but in $\mathbb{P}_{\text{OUT}}$, we sample the parameters of these functions from a range different from that of $\mathbb{P}_{\text{IN}}$. Finally, we use $\mathbb{P}_{\text{IN}}$ to amortize the training of $\mathcal{M}$ and $\mathbb{P}_{\text{OUT}}$ to evaluate its out-of-distribution (O.O.D) performances. 

\begin{algorithm}[tb]
   \caption{$\text{TOS}(\hat{P},\mathcal{G})$}
   \label{alg:tos}
\begin{algorithmic}
\STATE {\bfseries Input:} $\hat{P}$, $\mathcal{G}$
\STATE Initialize $\mathcal{G}_{\hat{P}}=\hat{P}\mathcal{G}^T\hat{P}^T$, and $M\in\{0,1\}^{d\times d}$ s.t. $M_{i,j}=0$ if $i\leq j$ and $M_{i,j}=1$ otherwise.
\STATE $\bm{\ell}\gets (M \odot \mathcal{G}_{\hat{P}})\mathbf{1}_d$,\quad $\text{TOS}\gets \sum_{i=1}^{d} \mathds{1}_{ \bm{\ell}_i\geq 1}$
\STATE Return $1 -\text{TOS} / (d-1) $
\end{algorithmic}
\end{algorithm}

\begin{figure}[!h]
\centering
\includegraphics[width=.9\linewidth]{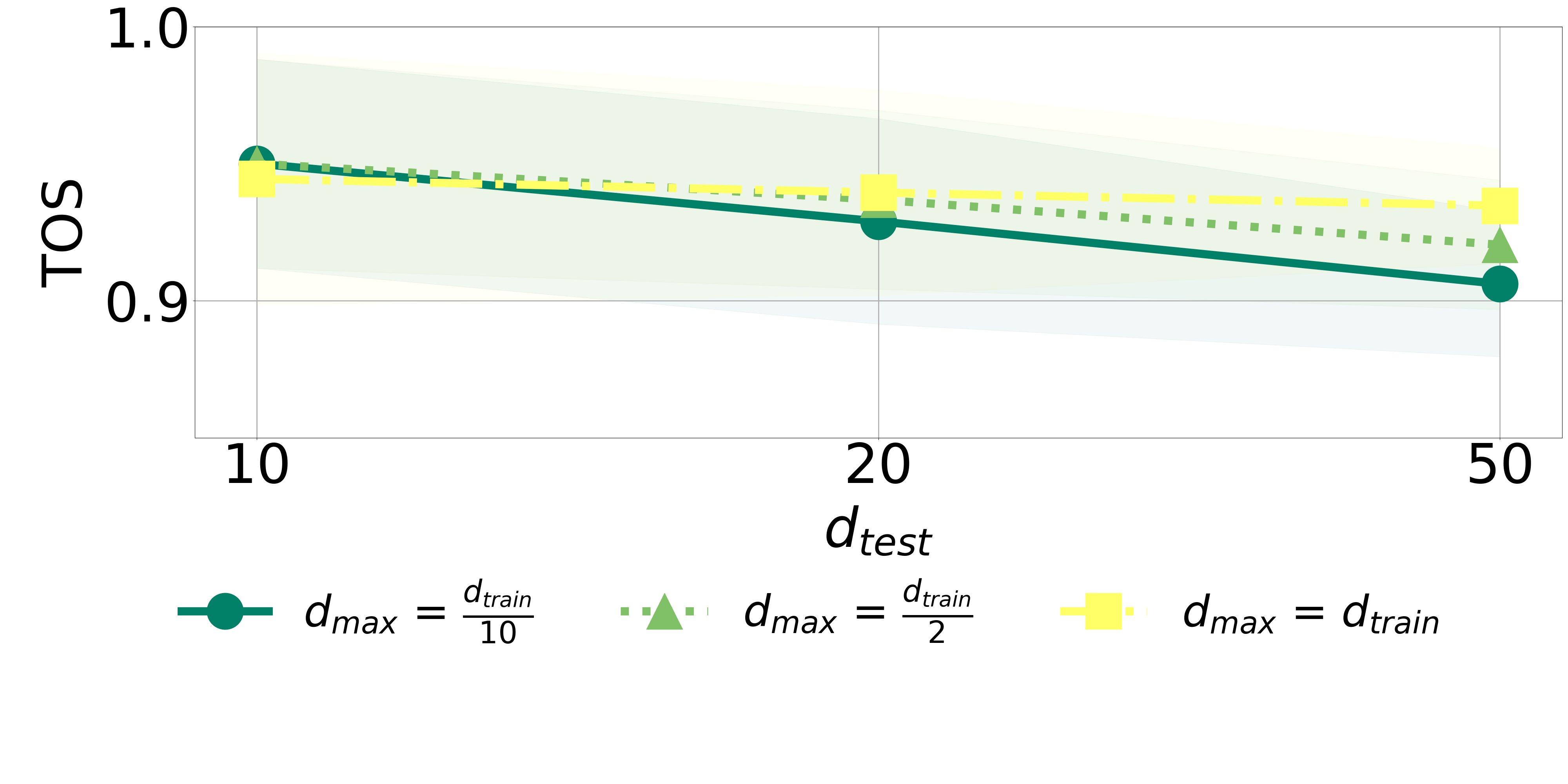}
\caption{We compare the performances of three models $\mathcal{M}$ trained on datasets of $n_{\text{train}}=200$ samples in $d_{\text{train}}=20$, but with different value for $d_{\text{max}}$. We measure their $\text{TOS}$ on the aggregation of both O.O.D metadatasets $\text{LIN}~\textbf{OUT}$ and $\text{RFF}~\textbf{OUT}$ for $d_{\text{test}}\in\{10,20,50\}$ and show as well the standard deviations. Note that we test on larger instance problems than seen during training when $d_{\text{test}}=50$.\label{fig:leaf-ablation-out}}
\vskip -0.1in
\end{figure}

\textbf{Train Datasets.} 
During training, we generate $K\simeq 200k$ datasets with their DAGs according to $\mathbb{P}_{\text{IN}}$, each consisting of $n_{\text{train}}=200$ i.i.d samples with $d_{\text{train}}=100$ dimensions. 

\textbf{Test Datasets.} To test our model, we build 2 in and 2 out-of-distribution metadatasets, each consisting of several datasets. More precisely, $\text{LIN}~\textbf{IN}$ consists of 27 synthetic datasets newly generated according to $\mathbb{P}_{\text{IN}}$, and using only linear functions. For each dimension $d_{\text{test}}\in \{10,20,50\}$, and possible choice for the graph distribution, we randomly generate 3 datasets with $n_{\text{test}}=10k$. Similarly, $\text{RFF}~\textbf{IN}$ is generated from $\mathbb{P}_{\text{IN}}$ with only RFF functions. Finally  $\text{LIN}~\textbf{OUT}$ and $\text{RFF}~\textbf{OUT}$ are generated with the same splitting of the functional relationships but according to $\mathbb{P}_{\text{OUT}}$.

\textbf{Evaluation Metrics.} To evaluate the inferred TO, we introduce the topological ordering score (TOS), a measure that quantifies precisely the quality of the TO inferred. More formally, for a predicted TO $\hat{P}\in\Sigma_d$ and a DAG $\mathcal{G}\in\{0,1\}^{d\times d}$, we define the topological ordering score as presented in Algorithm~\ref{alg:tos}. This score counts exactly the number of nodes that are correctly ranked topologically. 

\begin{table}[!h]
\caption{Training Memory Usage of $\mathcal{M}$ with $d_{\text{train}}=20$.}
\label{table:mem}
\begin{center}
\begin{small}
\begin{sc}
\begin{tabular}{lccc}
\toprule
$d_{\text{max}}$ & $d_{\text{train}}/10$ &  $d_{\text{train}}/2$ & $d_{\text{train}}$ \\
\midrule
Memory (GiB) & $ 3.35$& $6.59$& $8.77$ \\
\bottomrule
\end{tabular}
\end{sc}
\end{small}
\end{center}
\end{table}

\textbf{Results (Effect of Sub-sampling).} Firstly, we investigate the effect of the sub-sampling strategy introduced in section~\ref{sec:topological} to compute $\text{d-TOE}$ (Alg.~\ref{alg:one-forward-pass}) during training. For this experiment, we train smaller models $\mathcal{M}$ on datasets with $n_{\text{train}}=200$ samples and $d_{\text{train}}=20$ dimensions, and compare their performances when varying $d_{\text{max}}\in\{2,10,20\}$. In figure~\ref{fig:leaf-ablation-out}, we show that with only $10\%$ of \text{d-TOE}, we obtain similar performance as the full training $d_{\text{max}}=d_{\text{train}}$. In addition, in Table~\ref{table:mem}, we show empirically the linear dependency of the memory usage during training w.r.t $d_{\text{max}}$.

\textbf{Results (Generalization Performance).} Next, we evaluate the performance of our larger model $\mathcal{M}$ trained on datasets of $n_{\text{train}}=200$ samples in $d_{\text{train}}=100$, and we use $d_{\text{max}}=50$. We also generate larger instances of each test problem, allowing $d_{\text{test}}\in\{100,200\}$. 
In Figure~\ref{fig:leaf-pred-eval}, we show that our model is able to generalize on O.O.D datasets of smaller or equal size, and even to significantly larger problems. Note that on the O.O.D problems, our models show a drop of $10\%$ in performance when increasing the size of the problems. This is mostly due to the fact that the datasets considered in these metadatasets are out-of-distribution datasets sampled according to a distribution substantially different from that used for training $\mathcal{M}$.

\begin{figure}[!t]
\centering
\includegraphics[width=.9\linewidth]{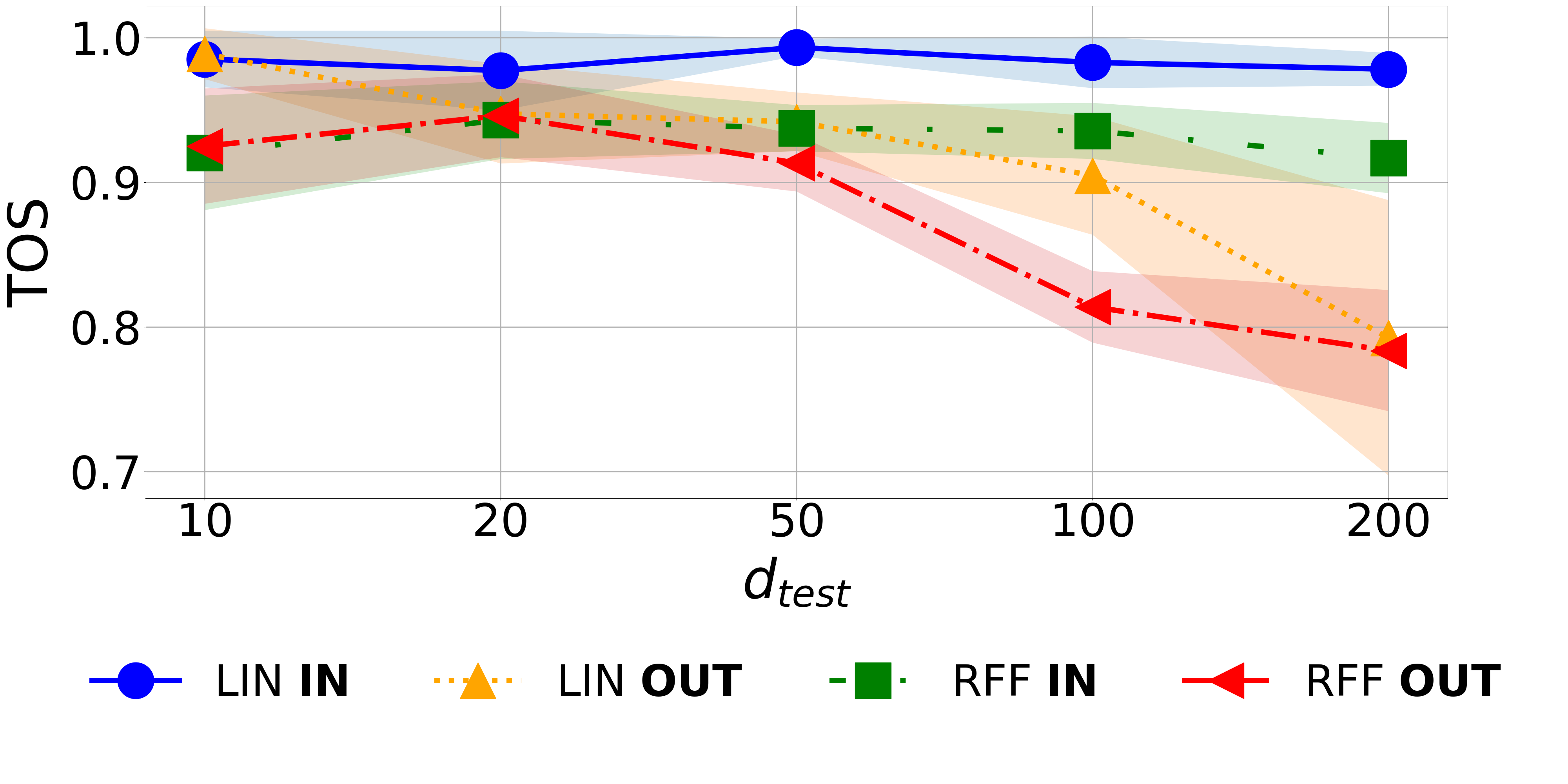}
\caption{For each of the four test metadatasets, we plot the $\text{TOS}$ obtained against the dimension $d_{\text{test}}$. For each curve, each point is obtained by averaging over all the test datasets of a given dimension. We also show the standard deviations.\label{fig:leaf-pred-eval}}
\end{figure}

\subsection{Evaluation of $\mathcal{T}_{\text{ANM}}$}
\label{sec:eval-fp-scm}

We evaluate the performances of our parameterization $\mathcal{T}_{\text{ANM}}$ for learning fixed-point SCMs when a true topological ordering is given on both causal discovery and inference tasks. 

\textbf{Datasets.} Besides reusing the synthetic metadatasets of section~\ref{sec:eval-amortization}, we consider three other settings where the causal graphs are accessible, namely C-Suite~\cite{geffner2022deep}, SynTReN~\cite{van2006syntren, lachapelle2019gradient} and the real-world dataset of protein measurements from~\cite{sachs2005causal}. They all consist of multiple datasets with continuous variables, except C-Suite where we discard the discrete and mixed type problems.

\textbf{Model Configuration.} We consider $\mathcal{T}_{\text{ANM}}$ with an embedding dimension of $D=128$, and $L=2$ layers. The causal attention mechanism uses $8$ heads with an embedding dimension of $d_{\text{head}}=32$. We do not hyper-tune the model on each specific instance, but use the same training configuration for all experiments. See appendix~\ref{sec:add-experiment} for details.

\vspace{-0.1cm}
\begin{table}[!h]
\caption{We compare the 
counterfactual predictions of $\mathcal{T}_{\text{ANM}}$ when trained with the true graph or the TO on various settings. We measure the re-scaled $\ell_2$ distance between the predicted counterfactual samples and the ground truth ones. The results presented are of the form $x/y~(z)$ where $x$ is the median, $y$ the mean and $z$ the standard deviation (std) w.r.t the number of datasets of the averaged errors.}
\label{table:cf-pred-self}
\begin{center}
\begin{small}
\begin{sc}
\begin{tabular}{lcc}
\toprule
Datasets & True $P$ & True $\mathcal{G}$ \\
\midrule
LIN \textbf{IN} &  0.037 / 0.066 (0.057) &  0.012 / 0.039 (0.060)\\
LIN \textbf{OUT} &  0.065 / 0.11 (0.084) &  0.017 / 0.034 (0.048)\\
RFF \textbf{IN} & 0.065 / 0.10 (0.089) & 0.033 / 0.059 (0.075)\\
RFF \textbf{OUT} & 0.11 / 0.12 (0.088) &  0.033 / 0.042 (0.040)\\
C-Suite & 0.026 / 0.032 (0.030) & 0.022 / 0.025 (0.022)\\
\bottomrule
\end{tabular}
\end{sc}
\end{small}
\end{center}
\vskip -0.1in
\end{table}

\textbf{Results (Graph Prediction).} We test the performances of $\mathcal{T}_{\text{ANM}}$ on DAG recovery problems. To obtain a binary graph from our model, we first estimate a continuous graph defined as the mean of its  absolute value Jacobian over samples, i.e. $\hat{\mathcal{G}_c}:=\mathbb{E}_{\bm{X}\sim \hat{P}_{\bm{X}}}|\text{Jac}_1 P^T \mathcal{T}_{\text{ANM}}(P\bm{X},0_d)|$. Then, we obtain the binary graph by applying a naive uniform threshold $\tau>0$, i.e. $\hat{\mathcal{G}}(\tau):=(\hat{\mathcal{G}_c}>\tau)$, thus discarding values smaller than $\tau$. In practice we propose to use $\tau=0.1$. To evaluate our prediction and the proposed rule, we compare the F1 scores obtained by $\mathcal{T}_{\text{ANM}}$ when trained given either the topological ordering or the causal graph, the latter being considered as the gold standard of our model. In figure~\ref{fig:f1-score-eval}, we show that, our model trained given the TO is still able to compete with the gold standard one, and the proposed rule, while not optimal, is able to keep most of the true non-zeros.

\begin{figure}[!t]
\centering
\includegraphics[width=.9\linewidth]{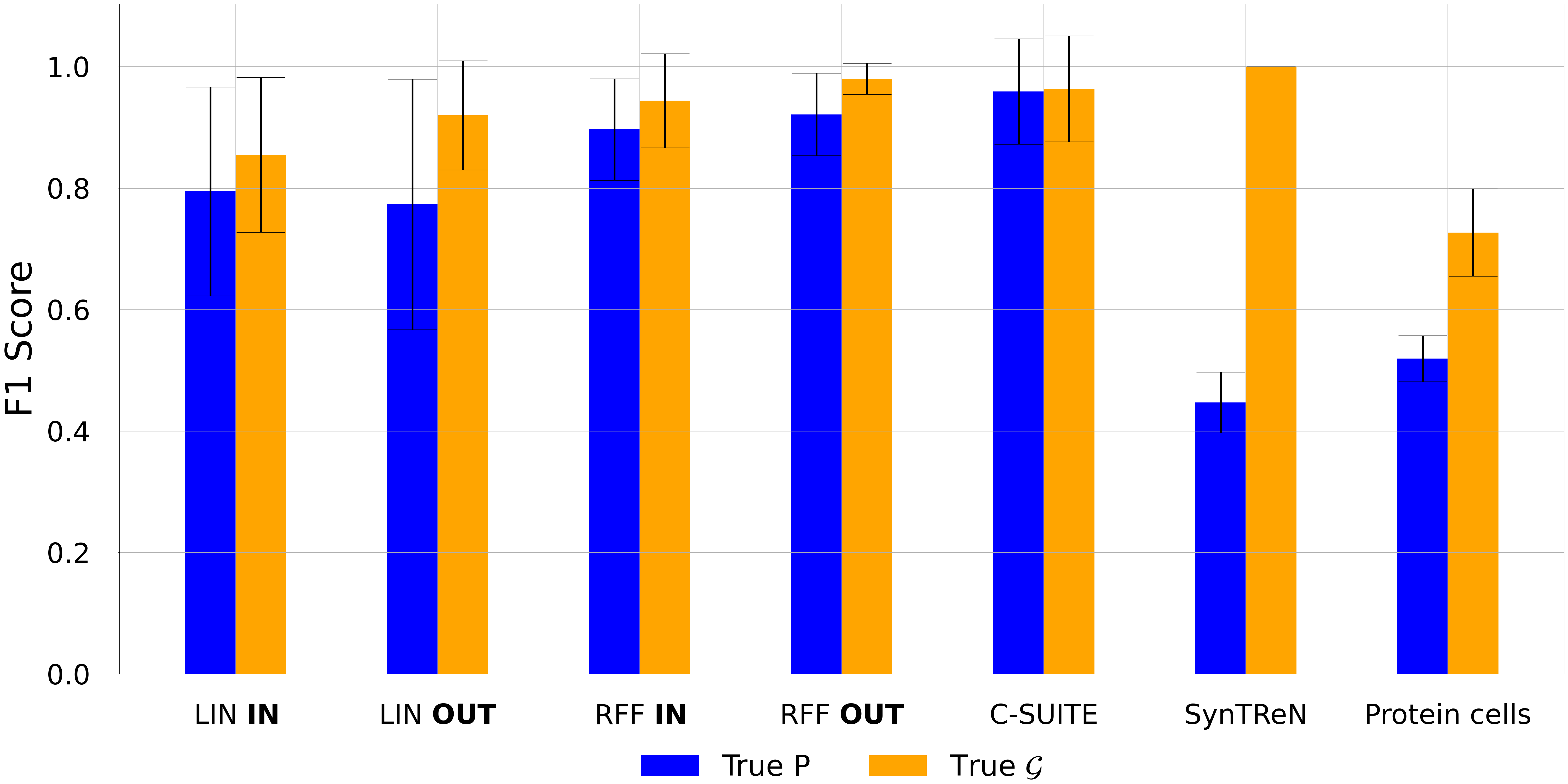}
\caption{We compare the F1 scores obtained by learning $\mathcal{T}_{\text{ANM}}$ with either the full graph or the TO on various settings. These score are obtained by comparing $\hat{\mathcal{G}}(0.1)$ with the ground truth graph and we show the averaged score over all instances of a given setting with their standard deviations.\label{fig:f1-score-eval}}
\end{figure}

\textbf{Results (Counterfactual Prediction).} Next, we evaluate the counterfactual predictions of $\mathcal{T}_{\text{ANM}}$. As we need to generate ground truth counterfactual samples (CF), we only focus on the test metadatasets of section~\ref{sec:eval-amortization} and C-Suite. For each setting, and each dataset, we randomly generate as many interventions as the number of nodes in the dataset, where each intervention is performed on 100 generated test samples. To measure the quality of a generated counterfactual sample, we measure the re-scaled $\ell_2$ distance to the ground truth, that is $r-\ell_2(x,\hat{x}):=\sqrt{\frac{1}{d}\sum_{i=1}^d \left(\frac{x_i - \hat{x}_i}{\sigma_i}\right)^2}$ where the $\sigma_i$ are the standard deviations of the variables $X_i$. Note that we divide the metric by $\sqrt{d}$ as we compare the results across various dimension choices. In table~\ref{table:cf-pred-self}, we show that $\mathcal{T}_{\text{ANM}}$ can recover almost perfectly the ground truth CF samples when learned with true DAGs and obtain $12\%$ mean errors at worst using true TOs.


\subsection{Full Pipeline Benchmarking}
\label{sec:benchmark}
Finally we evaluate our final causal generative modeling pipeline, obtained by combining $\mathcal{M}$ and $\mathcal{T}_{\text{ANM}}$, against various baselines on both causal discovery and counterfactual prediction tasks. More formally, given a new instance problem $\mathcal{D}$, we use our pre-trained zero-shot TO inference model $\mathcal{M}$ to predict a TO $\hat{P}$ from $\mathcal{D}$, and then learn $\mathcal{T}_{\text{ANM}}$ by minimizing~\eqref{eq:MSE-training} where $P$ is replaced by the predicted TO $\hat{P}$. We call our method FiP standing for Fixed-Point model.

\textbf{Baselines.} On causal discovery tasks, we compare our model with 
AVICI~\cite{lorch2022amortized},
PC~\cite{kalisch2007estimating}, GES~\cite{chickering2002optimal}, GOLEM~\cite{ng2020role}, DAG-GNN~\cite{yu2019dag}, GraN-DAG~\cite{lachapelle2019gradient}, DP-DAG~\cite{charpentier2022differentiable}, and DECI~\cite{geffner2022deep}. On counterfactual prediction tasks, we only compare with DECI and DoWhy~\cite{dowhy_gcm} trained with the predicted causal graph of AVICI, as other baselines do not provide this functionality in their codes.

\textbf{Results.} We show in tables~\ref{table:f1-pred-baseline} and~\ref{table:cf-pred-benchmark} that our model outperforms consistently all the other baselines on both causal discovery and counterfactual predictions tasks over the generated O.O.D test datasets LIN \textbf{OUT} and RFF \textbf{OUT}.

\begin{table}[!h]
\caption{We compare the directed F1 scores obtained by our model against various baselines on the out-of-distribution test metadatasets introduced in section~\ref{sec:eval-amortization}. The values reported are obtained by taking for each setting the mean over all the datasets as well as their standard deviations.}
\label{table:f1-pred-baseline}
\begin{center}
\begin{small}
\begin{sc}
\begin{tabular}{lcc}
\toprule
Datasets & LIN \textbf{OUT} & RFF \textbf{OUT} \\
\midrule
PC & 0.47 (0.14) & 0.40 (0.12)\\
GES & 0.56 (0.12) & 0.37 (0.060)\\
GOLEM  & 0.73 (0.29) & 0.31 (0.13)\\
DECI & 0.36 (0.13) & 0.74 (0.14)\\
GraN-DAG & 0.29 (0.19) & 0.50 (0.26)\\
DAG-GNN & 0.61 (0.19) & 0.44 (0.15)\\
DP-DAG & 0.17 (0.074) & 0.16 (0.067)\\
AVICI & 0.73 (0.16) & 0.74 (0.17)\\
FiP (Ours) & $\bm{0.76 (0.20)}$ & $\bm{0.81 (0.15)}$\\
\bottomrule
\end{tabular}
\end{sc}
\end{small}
\end{center}
\end{table}

\vspace{-0.2cm}
\begin{table}[!h]
\caption{We compare 
the counterfactual predictions obtained by our model against other baselines on the O.O.D metadatasets. We report the re-scaled $\ell_2$ errors to the ground truth. Note that we only report the mean and the standard deviation  w.r.t the number of datasets for space reason.}
\label{table:cf-pred-benchmark}
\begin{center}
\begin{small}
\begin{sc}
\begin{tabular}{lcc}
\toprule
Datasets & LIN \textbf{OUT} & RFF \textbf{OUT} \\
\midrule
DECI &  0.39 (0.29) &  0.18 (0.12)\\
DoWhy - AVICI &  0.20 (0.18) & 0.16 (0.096)\\
FiP (Ours) & $\bm{0.13~(0.10)}$ & $\bm{0.13~(0.096)}$\\
\midrule
FiP w. $\mathcal{G}$ &0.034~(0.048) & 0.042~(0.040)\\
DoWhy w. $\mathcal{G}$ & 0.0017~(0.0017) & 0.088~(0.072)\\
\bottomrule
\end{tabular}
\end{sc}
\end{small}
\end{center}
\vskip -0.1in
\end{table}
\section{Conclusion}
In this work, we introduce a new framework that defines SCMs as fixed-point problems on the ordered nodes. Based on this, we train a two-stage causal generative model, that infers in a zero-shot manner the TO, and learn the fixed-point SCM given the predicted ordering. We show that our model addresses causal discovery and inference tasks, and outperforms various baselines on O.O.D generated problems. For future work we aim to extend our method to a fully zero-shot SCM learning method, enabling paradigmatic shift towards assimilation of causal knowledge across domains, akin to the formulation of modern foundation models.

\nocite{langley00}

\newpage
\clearpage
\section*{Impact Statement}
This paper presents work whose goal is to advance the field of Machine Learning. There are many potential societal consequences of our work, none of which we feel must be specifically highlighted here. Nonetheless, the ability to accurately model and infer causal relationships is a powerful tool that could inform decision-making processes and policy formulation. At the same time, the misuse or misinterpretation of causal models could lead to incorrect conclusions and actions, especially in high-stakes domains such as public health, social science, and automated systems. Therefore, we stress the importance of critical examination, domain expertise, and the inclusion of fairness in the development and application of such models. 

\bibliography{biblio}
\bibliographystyle{icml2024}

\newpage
\appendix
\onecolumn

\section{Background}
\label{sec:background}
Structural Causal Models (SCMs) are widely used in the causal literature to express causal functional relationships between random variables. As they require a graph to represent the causal structure, we first review some basic graphical terminologies.

\textbf{Graph Terminology.} Let $d\geq 1$ an integer, $\bm{V}:=\{1,\dots, d\}$ a set of indices and $\mathcal{E}$ a subset of $\bm{V}^2$. Then, $\mathcal{G}:=(V,\mathcal{E})$ is called a graph on $d$ nodes $\bm{V}$ with edges $\mathcal{E}$. An edge $(i,j)\in\mathcal{E}$ is called directed if $(j,i)\notin \mathcal{E}$. The graph $\mathcal{G}$ is called directed if all its edges are directed. A node $i$ is called a parent of $j$ if $(i,j)\in \mathcal{E}$ and $(i,j)$ is directed, that is $(j,i)\notin \mathcal{E}$. We denote the set of parents of a node $j$ as $\textbf{PA}_{\mathcal{G}}(j)$ and its cardinal as $\textit{c}_j$. To refer explicitly to the parents of a node $j$,  we denote $(\text{pa}_1(j),\dots, \text{pa}_{\textit{c}_j}(j))$ the sequence of its parents ranked in the increasing order\footnote{That is: $\text{pa}_k(j) < \text{pa}_q(j)$ if $k < q$.}. A sequence of at least two nodes $(i_1,\dots, i_m)$ with $m\geq 2$, is called a directed path from $i_1$ to $i_m$ of $\mathcal{G}$ if for all $k\in[|1,m-1|]$, $(i_k,i_{k+1})$ is a directed edge of $\mathcal{G}$. A directed path from a node $i$ to itself is called a directed cycle. Finally $\mathcal{G}$ is called a directed acyclic graph (DAG) if it is directed and does not contain directed cycle.

\textbf{Topological Ordering.} An important notion from the graph terminology that is extensively used in this paper is the notion of topological ordering. When the graph $\mathcal{G}$ is a DAG, it is always possible to order the nodes in a specific manner. We call $j$ a descendant of $i$ in $\mathcal{G}$ if there exists a directed path from $i$ to $j$ in $\mathcal{G}$. We denote the set of all the descendants of a node $i$ as $\textbf{DE}_{\mathcal{G}}(i)$. If a node does not have any descendants, we call it a leaf node. If a node does not have any parents, we call it a root node. When $\mathcal{G}$ is a DAG, there exists a permutation $\pi$, that is a bijective mapping  $$\pi: \{1,\dots, d \}  \to \{1,\dots, d \}, $$ satisfying $\pi(i)< \pi(j)$ if $j\in\textbf{DE}_{\mathcal{G}}(i)$. We call such a permutation a topological ordering of $\mathcal{G}$ and it does not have to be unique. Note that the node $\pi^{-1}(1)$ is a root node, and the node $\pi^{-1}(d)$ is a leaf node. In the following we also denote $P_{\pi}\in\{0,1\}^{d\times d}$ the permutation matrix associated, that is $[P_{\pi}]_{i,j}=1$ if $\pi^{-1}(i)=j$ and $0$ otherwise, and $\Sigma_d$ the set of permutation matrices of size $d$.

\textbf{Structural Causal Models.} A structural causal model is a generative model that aims at modeling the causal relationships between random variables. The model consists of three main components: (i) a jointly independent distribution $\mathbb{P}_{\bm{N}}$ modeling the distribution of $d$ \emph{jointly independent} exogenous random variables $\bm{N}:=(N_1,\dots,N_d)$, (ii) a DAG $\mathcal{G}$ on $d$ nodes, and (iii) a sequence of $d$ measurable functions $(f_1,\dots,f_d)$. The definition of these functions depends both on the exogenous variables and, most importantly, on the graph. More formally, let $n_1,\dots,n_d\geq 1$, $d$ integers, $N_1,\dots, N_d$, $d$ jointly independent random variables on respectively $\mathbb{R}^{n_1},\dots, \mathbb{R}^{n_d}$ s.t. $(N_1,\dots,N_d)\sim\mathbb{P}_{\bm{N}}$, and $\mathcal{G}$ a DAG on $d$ nodes. Let also $t_1,\dots, t_d\geq 1$, $d$ integers, and for all $i\in \{1,\dots,d\}$, let $f_i$ a measurable function satisfying $f_i: \mathbb{R}^{p_i}\times\mathbb{R}^{n_i}\to \mathbb{R}^{t_i}$, where $p_i:=\sum_{k\in \textbf{PA}_{\mathcal{G}}(i)} t_k$ if $\textbf{PA}_{\mathcal{G}}(i)\neq \emptyset$,  $p_i:=0$ otherwise. Then, the SCM associated to $(\mathcal{G}, \mathbb{P}_{\bm{N}}, (f_1,\dots,f_d))$, is the 4-tuple  $(\mathcal{G}, \mathbb{P}_{\bm{N}}, (f_1,\dots,f_d), \mathcal{S}(\mathcal{G}, \mathbb{P}_{\bm{N}}, (f_1,\dots,f_d)))$, where $\mathcal{S}(\mathcal{G}, \mathbb{P}_{\bm{N}}, (f_1,\dots,f_d))$ is defined as the collection of the following $d$ (structural) equations on the $X_i$'s:
\begin{align}
\label{eq:scm-standard}
    X_i = f_i(\textbf{PA}_{\mathcal{G}}^{(i)}(\bm{X}), N_i),~\quad \forall i\in\{1,\dots,d\}
\end{align}
where we denote $\textbf{PA}_{\mathcal{G}}^{(i)}(\bm{X}):=[X_{\text{pa}_1(i)},\dots, X_{\text{pa}_{\textit{c}_i}(i)}]\in \mathbb{R}^{p_i}$. The random variables $X_i$ are implicitly defined as the solution of the system~\eqref{eq:scm-standard} which is unique thanks to the DAG structure of $\mathcal{G}$. In the following, we denote the set of all possible standard SCMs $\mathcal{S}$, that is the set of all possible 4-tuples of the form $(\mathcal{G}, \mathbb{P}_{\bm{N}}, (f_1,\dots,f_d), \mathcal{S}(\mathcal{G}, \mathbb{P}_{\bm{N}}, (f_1,\dots,f_d)))$. We also allow ourselves an abuse of notations and represent an SCM only by $\mathcal{S}(\mathcal{G}, \mathbb{P}_{\bm{N}}, (f_1,\dots,f_d))$, while we mean to represent the full 4-tuple  $(\mathcal{G}, \mathbb{P}_{\bm{N}}, (f_1,\dots,f_d), \mathcal{S}(\mathcal{G}, \mathbb{P}_{\bm{N}}, (f_1,\dots,f_d)))$.

\begin{remark}
Observe that here, we move from the original representation of parents viewed as a \textbf{set} of variables (e.g. as defined in~\cite{peters2017elements}), to the equivalent representation where $\textbf{PA}_{\mathcal{G}}^{(i)}(\bm{X})$ is viewed as a finite sequence of variables (or vector of variables) ordered in the increasing order w.r.t the natural order of the indices in the graph $\mathcal{G}$. 
\end{remark}

This general definition of SCM allows the existence of an edge $(j,i)$ in $\mathcal{G}$ that has no influence, meaning that the function $f_i$ can be independent of the variable $x_j$. In order to exclude such situations, we will assume in the following that the $f_i$'s always depend on all its variables. More formally, we consider the following assumption.

\begin{assumption}[Sturctural Minimality]
\label{ass-minimality}
We assume that for all $i\in\{1,\dots,d\}$, there does not exist a $k\in\{1,\dots,\textit{c}_i\}$ and a function $g_i:\mathbb{R}^{p_i-t_{\text{pa}_k(i)}}\times \mathbb{R}^{n_i}\to\mathbb{R}$, such that for all $z_1\in\mathbb{R}^{t_{\text{pa}_1(i)}},\dots,z_{c_i}\in\mathbb{R}^{t_{\text{pa}_{\textit{c}_i}(i)}}$ and $n_i\in\mathbb{R}^{n_i}$:
\begin{align*}
&f_i(z_{1},\dots, z_{c_i}, n_i) =\\ &g_i(z_{1},\dots,z_{k-1},z_{k+1},\dots,z_{c_i},n_i).
\end{align*}
\end{assumption}

Therefore a SCM defines a causal generative process of the $X_i$'s obtained from the exogenous variables $N_i$'s, where the causal structure is given by $\mathcal{G}$, and the functional relationships are given by $f_i$'s. In the following we often refer the exogenous variables $N_i$ as the noise variables. Let us now present two mild assumptions that we hold in the paper.
\begin{assumption}
\label{ass-1d}
Let $\mathcal{S}(\mathcal{G}, \mathbb{P}_{\bm{N}}, (f_1,\dots,f_d))$ an SCM as defined in~\eqref{eq:scm-standard}. Assume that for all $i\in\{1,\dots,d\}$, $n_i=t_i=1$. 
\end{assumption}
This assumption restricts our framework to the case where both the exogenous variables $N_i$ and the generated variables $X_i$ are real-valued. It is made mostly to simplify our notations. 

\begin{assumption}
\label{ass-diff}
Let $\mathcal{S}(\mathcal{G}, \mathbb{P}_{\bm{N}}, (f_1,\dots,f_d))$ an SCM as defined in~\eqref{eq:scm-standard}. Assume that the $f_i$'s are differientiable.
\end{assumption}
This assumption is used in section~\ref{sec:scm} where we revisit the definition of an SCM and propose another formalism on which we build our proposed causal generative model.

\paragraph{Background on Triangular Maps.} Let us now recall some basic facts on triangular maps, which are functions that will be largely exploited in the rest of the paper. A differentiable and triangular map $T$ from $\mathbb{R}^d$ to $\mathbb{R}^d$ is a function such that its Jacobian $Jac_x T$ is a lower (or upper) triangular matrix, and so for any $x\in\mathbb{R}^d$. This property ensures that for all $i\in\{1,\dots,d\}$,  $x\in\mathbb{R}^d \to [T(x)]_i\in\mathbb{R}$ is a function of $[x_1,\dots,x_i]\in\mathbb{R}^{i}$. Such maps have been studied in various domains of ML, especially for generative modeling \cite{kingma2016improved,papamakarios2017masked, pmlr-v151-irons22a}, as the determinant of the Jacobian of these functions can be efficiently computed.

In fact, these maps are also closely related to SCMs, as once the nodes in the graph are ordered in a topological order, there always exists a triangular function that maps the exogenous variables $P_{\pi}\bm{N}$ to the endogenous ones $P_{\pi}\bm{X}$, that is there exists $T$ triangular such that $T(P_{\pi}\bm{N})=P_{\pi}\bm{X}$. This can be seen by unrolling the equations~\eqref{eq:scm-standard} in a topological ordering $\pi$ associated to the underlying graph $\mathcal{G}$. However, the function $T$ obtained by unrolling these equations can only provide an access to the generative process of the causal system (or the SCM), but not the causal system itself, which is of capital interest if one wants to perform counterfactual operations.

In this work, we generalize this viewpoint by considering the effect of the ordering on the causal system itself, which leads us to the proposed definition of fixed-point SCMs.
\section{Fixed-Point SCMs}

\subsection{An Equivalent Representation of Standard SCMs}
\label{sec:random-var-fp}
The goal of this section is to obtain an equivalent formulation of standard SCMs viewed as fixed-point problems. However, we will see that this reparameterization will create a strong dependency between the functional relationship involved in the fixed-point problem and the graph $\mathcal{G}$. 

Following Appendix~\ref{sec:background}, let $\mathcal{S}(\mathcal{G}, \mathbb{P}_{\bm{N}}, (f_1,\dots,f_d))$ an SCM as defined in~\eqref{eq:scm-standard}. Let us now define
$F^{\mathcal{G}}:\mathbb{R}^d\times \mathbb{R}^d\to \mathbb{R}^d$, satisfying $\forall i\in\{1,\dots,d\},~~x,n\in\mathbb{R}^d$:
\begin{align*}
F_i^{\mathcal{G}}(x,n):= f_i(x_{\text{pa}_1(i)},\dots,x_{\text{pa}_{\textit{c}_i}(i)},n_i),
\end{align*}
where $F^{\mathcal{G}}(x,n):=[F_1^{\mathcal{G}}(x,n),\dots,F_d^{\mathcal{G}}(x,n)]$, $x:=[x_1,\dots, x_d]$ and $n:=[n_1,\dots, n_d]$. Then the system of equations introduced in~\eqref{eq:scm-standard} can be equivalently reformulated as the following fixed-point problem on $\bm{X}$:
\begin{align}
\label{scm_fp}
\bm{X}=F^{\mathcal{G}}(\bm{X},\bm{N}).
\end{align}
Observe that here, we have made explicit the dependence between $F^{\mathcal{G}}$ and the graph $\mathcal{G}$. This is because $F^{\mathcal{G}}$ is not allowed to be any function from $\mathbb{R}^d\times\mathbb{R}^d\to\mathbb{R}^d$, but it has to satisfy the structure of the graph $\mathcal{G}$. More simply put, $F^{\mathcal{G}}$ is the composition of the unconstrained functions $(f_1,\dots,f_d)$ and the operators $\textbf{PA}_{\mathcal{G}}^{(i)}(x):=[x_{\text{pa}_1(i)},\dots, x_{\text{pa}_{\textit{c}_i}(i)}]\in \mathbb{R}^{c_i}$ that projects $x$ to its coordinates given by the parents of $i$ in $\mathcal{G}$. In addition, one can recover the graph $\mathcal{G}$ from $F^{\mathcal{G}}$ as shown in the following Lemma.
\begin{lemma}
Let $\mathcal{S}(\mathcal{G}, \mathbb{P}_{\bm{N}}, (f_1,\dots,f_d))$ an SCM and $F^{\mathcal{G}}$ as defined in~\eqref{scm_fp}. Then $(i,j)\in\mathcal{E}$ i.i.f we have $(x,n)\to\frac{\partial F_j^{\mathcal{G}}}{\partial x_i}(x,n)\neq \bm{0}$.
\end{lemma}
\begin{proof}
This results follows directly from the structure of $\mathcal{F}^{\mathcal{G}}$: $(j,i)$ is an edge i.f.f there exists $k$ such that $pa_k(i)=j$, and by minimimality assumption~\ref{ass-minimality}, i.i.f $(x,n)\to\frac{\partial f_i}{\partial x_j}(\cdot,\cdot)\neq \bm{0}$ and therefore i.f.f $\frac{\partial F_j^{\mathcal{G}}}{\partial x_i}(\cdot,\cdot)\neq \bm{0}$.
\end{proof}

Therefore, by applying this reparameterization on the functional relationships, we obtain an equivalent representation of standard SCMs as fixed-point problems. More formally, any standard SCM of the $(\mathcal{G},\mathbb{P}_{\bm{N}},(f_1,\dots,f_d),\mathcal{S}(\mathcal{G}, \mathbb{P}_{\bm{N}}, (f_1,\dots,f_d))$, can be equivalently reparameterized as the 4-tuple $(\mathcal{G},\mathbb{P}_{\bm{N}}, F^{\mathcal{G}},\mathcal{S}_{\text{eq}}(\mathbb{P}_{\bm{N}}, F^{\mathcal{G}}))$, where $\mathcal{S}_{\text{eq}}( \mathbb{P}_{\bm{N}}, F^{\mathcal{G}})$ denotes the fixed-point problem $\bm{X}=F^{\mathcal{G}}(\bm{X},\bm{N})$. In the following, we denote $\mathcal{S}_{\text{eq}}$ the set of all possible  4-tuples of the form $(\mathcal{G},\mathbb{P}_{\bm{N}}, F^{\mathcal{G}},\mathcal{S}_{\text{eq}}(\mathbb{P}_{\bm{N}}, F^{\mathcal{G}}))$. We also allow to represent such reparameterized standard SCM as only $\mathcal{S}_{\text{eq}}(\mathbb{P}_{\bm{N}}, F^{\mathcal{G}})$. Let us now show the following simple result.
\begin{proposition}
There exists a bijection between $\mathcal{S}$ and $\mathcal{S}_{\text{eq}}$, and we denote such equivalence relation as $\mathcal{S}\longleftrightarrow \mathcal{S}_{\text{eq}}$.
\end{proposition}
\begin{proof}
The function that maps $(\mathcal{G}, \mathbb{P}_{\bm{N}}, (f_1,\dots,f_d), \mathcal{S}(\mathcal{G}, \mathbb{P}_{\bm{N}}, (f_1,\dots,f_d)))$ to $(\mathcal{G},\mathbb{P}_{\bm{N}}, F^{\mathcal{G}},\mathcal{S}_{\text{eq}}(\mathbb{P}_{\bm{N}}, F^{\mathcal{G}}))$ as build above is clearly bijective. Indeed, the surjection comes from the definition of $\mathcal{S}_{\text{eq}}$, and the injection is trivial: if $(\mathcal{G}_1,\mathbb{P}_{\bm{N}_1}, F^{\mathcal{G}_1},\mathcal{S}_{\text{eq}}(\mathbb{P}_{\bm{N}_1}, F^{\mathcal{G}}_1))=(\mathcal{G}_2,\mathbb{P}_{\bm{N}_2}, H^{\mathcal{G}_2},\mathcal{S}_{\text{eq}}(\mathbb{P}_{\bm{N}_2}, H^{\mathcal{G}_2}))$, it follows that 
$\mathcal{G}_1=\mathcal{G}_2$, $\mathbb{P}_{\bm{N}_1}=\mathbb{P}_{\bm{N}_2}$, and $F^{\mathcal{G}_1}=H^{\mathcal{G}_2}$, from which follows that $f_i=h_i$ and so for all $i\in\{1,\dots,d\}$, where $f_i=F_i^{\mathcal{G}_1}$ and $h_i=H_i^{\mathcal{G}_2}$.
\end{proof}

While $\mathcal{S}_{\text{eq}}( \mathbb{P}_{\bm{N}}, F^{\mathcal{G}})$ allows to represent bijectively standard SCMs as fixed-point problems, observe that the functional relationships $F^{\mathcal{G}}$ involved in the fixed-point problems still depend on a DAG $\mathcal{G}$. We will see in the following that by simply augmenting the representation of standard SCMs with a topological ordering, we obtain another representation that is independent of any DAG. 

\subsection{An Augmented Represention of Standard SCMs}

The goal of this section is to obtain a new representation of SCMs that is able to distinguish between two same standard SCMs when ordered with two different topological orderings. More formally, given a reparametrized  standard SCM $\mathcal{S}_{\text{eq}}( \mathbb{P}_{\bm{N}}, F^{\mathcal{G}})$, and two valid topological ordering (if they exist) $\pi_1\neq \pi_2$ associated to $\mathcal{G}$, we want to be able to distinguish between $\mathcal{S}_{\text{eq}}( \mathbb{P}_{\bm{N}}, F^{\mathcal{G}})$  when ordered according to $\pi_1$ with the same SCM $\mathcal{S}_{\text{eq}}( \mathbb{P}_{\bm{N}}, F^{\mathcal{G}})$ when ordered according to $\pi_2$. To do so, we propose to simply augment the previous representation of standard SCMs by adding to the parameterization a valid topological ordering. 

Let $\pi$ a topological ordering associated to the DAG $\mathcal{G}$ and let us define the 5-tuple $(\mathcal{G}, \pi, F^{\mathcal{G}}, \mathbb{P}_{\bm{N}}, \mathcal{S}_{\text{eq}}( \mathbb{P}_{\bm{N}}, F^{\mathcal{G}}))$. Observe that this 5-tuple a simple augmentation of the previous representation of standard SCMs $(\mathcal{G}, F^{\mathcal{G}}, \mathbb{P}_{\bm{N}}, \mathcal{S}_{\text{eq}}( \mathbb{P}_{\bm{N}}, F^{\mathcal{G}}))$ where we add a valid topological ordering into the parameterization.

Now, recall that for a given DAG, there (always exists at least one and) might exist multiple valid topological orderings. To clarify this, we denote $\Pi(\mathcal{G}):=\{\pi~\text{s.t.} \pi \text{ is a valid topological ordering of } \mathcal{G}\}$. Then, we define $\mathcal{S}_{\text{aug}}$, the set of augmented standard SCMs defined as the set of all possible 5-tuples of the form $(\mathcal{G}, \pi, F^{\mathcal{G}}, \mathbb{P}_N, S_{\text{eq}}(\mathbb{P}_{\bm{N}},F^{\mathcal{G}}))$ with $\pi\in\Pi(\mathcal{G})$. 

While this augmentation does not remove the dependency between the functional relationship $F^{\mathcal{G}}$ and the graph $\mathcal{G}$, it has the essential advantage to allow distinguishing between two same standard SCMs but with two different topological orderings, a property that does not enjoy the standard formulation of SCMs (or its equivalent reparameterizations). In the following, we will see that by simply reparameterizing this augmented representation, we finally obtain a representation for SCMs independent of DAGs.

\subsection{Fixed-Point SCMs Without DAGs}
In this section, we will show that by reparameterizing the augmented representation of standard SCMs obtained above, we can define an equivelent (i.e.  bijective) representation of the augmented SCMs without having to instantiate a DAG in its parameterization.

Let $(\mathcal{G}, \pi, F^{\mathcal{G}}, \mathbb{P}_{\bm{N}}, S_{\text{eq}}(\mathbb{P}_{\bm{N}},F^{\mathcal{G}}))$ an augmented SCM, and let us denote $P_{\pi}\in\Sigma_d$ the permutation matrix associated to the topological ordering $\pi$. Then by defining $F_{\pi}^{\mathcal{G}} \colon \mathbb{R}^d\times  \mathbb{R}^d  \to \mathbb{R}^d$ such that for all $x,n\in\mathbb{R}^d$ 
\begin{equation}
\label{func_ordered}
\begin{gathered}
F_{\pi}^{\mathcal{G}}(x,n):=\\
[F_{\pi^{-1}(1)}^{\mathcal{G}}(P_\pi^{T} x, P_\pi^{T} n),\dots,F_{\pi^{-1}(d)}^{\mathcal{G}}(P_\pi^{T} x, P_\pi^{T} n)],
\end{gathered}
\end{equation}
we obtain an equivalent formulation of~\eqref{scm_fp} defined as the following fixed-point problem on $\bm{X}$: 
\begin{align}
\label{fp_ordered}
\bm{X}=P_{\pi}^TF_\pi^{\mathcal{G}}(P_{\pi}\bm{X}, P_\pi\bm{N}).
\end{align} 
Observe that here, we have made explicit the dependence between the function $F_\pi^{\mathcal{G}}$ and both the graph $\mathcal{G}$ and the TO $\pi$, due to the definition of the proposed reparameterization. In the following we denote $\mathcal{S}_{\text{fp}}$ the set of all possible reparameterized and augmented SCMs, that is the set of all possible 5-tuples of the form $(\mathcal{G}, \pi, F_{\pi}^{\mathcal{G}}, \mathbb{P}_{\bm{N}}, S_{\text{fp}}(P_\pi,\mathbb{P}_{\bm{N}},F_{\pi}^{\mathcal{G}}))$ where $ S_{\text{fp}}(P_\pi,\mathbb{P}_{\bm{N}},F_{\pi}^{\mathcal{G}})$ is the fixed-point problem  $\bm{X}=P_{\pi}^TF_\pi^{\mathcal{G}}(P_{\pi}\bm{X}, P_\pi\bm{N})$. Let us now show the following simple result.

\begin{proposition}
$\mathcal{S}_{\text{aug}}$ is in bijection with $\mathcal{S}_{\text{fp}}$, that is $\mathcal{S}_{\text{aug}}\longleftrightarrow\mathcal{S}_{\text{fp}}$.
\end{proposition}
\begin{proof}
The result is clear. It follows directly from the bijectivity of the mapping built above. Again the surjection of the construction come from the definition of $\mathcal{S}_{\text{fp}}$ and the injection comes from the identifiability of all the objects involved in the tuple from the mapping.
\end{proof}

The main advantage of this reparameterization of augmented SCMs, is that $F_\pi^{\mathcal{G}}$ must satisfy a very simple structure.
\begin{lemma}
\label{lem-existence-fp}
Let $F_\pi^{\mathcal{G}}$ as defined in~\eqref{func_ordered}, then it satisfies for all $x,n\in\mathbb{R}^d$:
\begin{equation*}
\begin{aligned}
[Jac_1 F_\pi^{\mathcal{G}}(x,n)]_{i,j} &= 0,\quad \text{if}\quad j\geq i,\quad \text{and} \\ 
[Jac_2 F_\pi^{\mathcal{G}}(x,n)]_{i,j} &= 0,\quad \text{if}\quad i\neq j.
\end{aligned}
\end{equation*}
\end{lemma}

\begin{proof}
Under assumption of~\ref{ass-diff}, $F^{\mathcal{G}}$ is therefore differentiable, and we obtain directly that for all $x,n\in\mathbb{R}^d$, and $k\in\{1,2\}$, $Jac_k F_\pi^{\mathcal{G}}(x,n) = P_\pi Jac_k F^{\mathcal{G}}(P_\pi^Tx,P_\pi^Tn)P_\pi^T$. In addition, we have that
for all $w,z\in\mathbb{R}^d$, $[Jac_1 F^{\mathcal{G}}(w,z)]_{i,j}=0$ if $j$ is not a parent of $i$ and $[Jac_2 F^{\mathcal{G}}(w,z)]_{i,j}=0$ if $i\neq j$, then using the rearrangement given by $P_{\pi}$, we deduce the result.
\end{proof}

Therefore, the function $F_\pi^{\mathcal{G}}$, defining the new fixed-point problem~\eqref{fp_ordered}, has to admit a strictly lower-triangular Jacobian w.r.t $x$ and a diagonal Jacobian w.r.t $n$. However, we still have a dependency between the functional relationship $F_\pi^{\mathcal{G}}$, and the graph $\mathcal{G}$ (as well as the topological odering $\pi$). We now ask the following question: 
\begin{question}
Is the above condition sufficient to remove the graphs from the parameterization of elements in $\mathcal{S}_{\text{fp}}$? 
\end{question}

We answer positively to the above question in the following proposition.
\begin{proposition}
Let us denote $\mathcal{\tilde{S}}_{\text{fp}}$ the set of all 4-tuples of the form $(P, H,\mathbb{P}_{\bm{N}},\mathcal{S}_{\text{fp}}(P,\mathbb{P}_{\bm{N}}, H))$ where $P\in\Sigma_d$ is a permutation matrix, $\mathbb{P}_{\bm{N}}\in\mathcal{P}(\mathbb{R})^{\otimes d}$ a jointly independent distribution over $\mathbb{R}^d$ and $H\in\mathcal{F}_d$. Then we have $$\mathcal{\tilde{S}}_{\text{fp}}\longleftrightarrow \mathcal{S}_{\text{fp}}\longleftrightarrow\mathcal{S}_{\text{aug}} $$
\end{proposition}
\begin{proof}
Refer to the proof of Proposition~\ref{prop:equivalence} in Appendix~\ref{sec:proofs}.
\end{proof}
Therefore by transitivity, the above proposition ensures that one can equivalently represent any augmented SCM, i.e. any element in $\mathcal{S}_{\text{aug}}$, using the parameterization of $\mathcal{\tilde{S}}_{\text{fp}}$, that is using a 4-tuple of the form $(P, H,\mathbb{P}_{\bm{N}},\mathcal{S}_{\text{fp}}(P,\mathbb{P}_{\bm{N}}, H))$, which is now completely independent of any DAG. 

\begin{remark}
Note that $\mathcal{\tilde{S}}_{\text{fp}}$ and more precisely the 4-tuple $(P, H,\mathbb{P}_{\bm{N}},\mathcal{S}_{\text{fp}}(P,\mathbb{P}_{\bm{N}}, H))$, is exactly the parameterization introduced in definition~\ref{def:fp-scm}.
\end{remark}

To conclude, we obtain that by augmenting the representation of standard SCMs using a topological ordering, we can distinguish between two same standard SCMs with different topological orderings, which allows us to represent any augmented standard SCM (that is any element of $\mathcal{S}_{\text{aug}}$) independently of its DAG, by using an element of $\mathcal{\tilde S}_{\text{fp}}$, that is using a fixed-point SCM as defined in definition~\ref{def:fp-scm}.  

\subsection{Equivalence of Partial Recovery}
Recall that the partial recovery problem of interest in this paper is defined as follows.

\textbf{Partial Recovery Problem.} Let $\mathcal{S}_{\text{fp}}(P, \mathbb{P}_{\bm{N}}, H)$ a fixed-point SCM generating $\gamma(P, \mathbb{P}_{\bm{N}}, H)$ with left marginal $\mathbb{P}_{\bm{X}}:=p_1\#\gamma(P, \mathbb{P}_{\bm{N}}, H)$. Given $P$ and $\mathbb{P}_{\bm{X}}$, can we recover uniquely the generating fixed-point SCM $\mathcal{S}_{\text{fp}}(P, \mathbb{P}_{\bm{N}}, H)$? Or more formally, is $\mathcal{A}_P(\mathbb{P}_{\bm{X}})$ a singleton?

\begin{remark}
$A_P(\mathbb{P}_{\bm{X}})$ is a singleton if and only if there exists a unique fixed-point SCM with topological order $P$ generating $\mathbb{P}_{\bm{X}}$ .
\end{remark}

In the next proposition, we show that the partial recovery of fixed-point SCMs guarantees that of standard SCMs.

\begin{proposition}
\label{prop:rec-enough}
 Let $P_{\bm{X}}\in\mathcal{P}(\mathbb{R}^d)$, $P\in\Sigma_d$ and assume that $\mathcal{A}_P(\mathbb{P}_{\bm{X}})$ is a singleton. Then, there exists a unique standard SCM of the form $\mathcal{S}(\mathcal{G},\mathbb{P}, F)$ generating $\mathbb{P}_{\bm{X}}$ such that $P$ is a valid topological ordering of $\mathcal{G}$.
\end{proposition}

\begin{proof}
Let us assume that there exists two standard SCMs $\mathcal{S}(F_1,\mathbb{P}_{\bm{N}_1})$  and $\mathcal{S}(F_2,\mathbb{P}_{\bm{N}_2})$ generating $\mathbb{P}_{\bm{X}}$. As $P$ is a valid topological ordering for both SCMs, then thanks to proposition~\ref{prop:equivalence}, there exists $H_1,H_2\in\mathbb{\mathcal{F}_d}$ such that $\mathcal{S}_{\text{fp}}(P,\mathbb{P}_{\bm{N}_1}, H_1)$ $\mathcal{S}_{\text{fp}}(P,\mathbb{P}_{\bm{N}_2}, H_2)$ generate $P_{\bm{X}}$ and satisfy for $i\in\{1,2\}$ and $x,n\in\mathbb{R}^d$
\begin{align*}
    P^T H_i(P x,P n) = F_i(\textbf{PA}_i(x),n)\; .
\end{align*}
Then because $\mathcal{A}_P(\mathbb{P}_{\bm{X}})$ is a singleton, we must have $H_1 = H_2$ and $\mathbb{P}_{\bm{N}_1}=\mathbb{P}_{\bm{N}_2}$, from which follows that 
\begin{align*}
    F_1(\textbf{PA}_1(x),n)= F_2(\textbf{PA}_2(x),n)\; .
\end{align*}
and by minimality assumption~\ref{ass-minimality}, we obtain that $F_1=F_2$ from which the results follows.
\end{proof}

\subsection{Links with Normalizing Flows}

As a by-product of proposition~\ref{prop:equivalence}, our fixed-point formulation can also recover the normalizing flow induced by a standard SCM $\mathcal{S}(F,\mathbb{P}_{\bm{N}})$ as defined in section~\ref{sec:scm}. More formally, let $\mathcal{S}_{\text{fp}}(P,\mathbb{P}_{\bm{N}},H)$ an equivalent fixed-point SCM according to proposition~\ref{prop:equivalence} and let $T:n\in\mathbb{R}^d\to H(\cdot, n)^{\circ d}(0_d)\in\mathbb{R}^d$. Observe now that $T$ is a triangular map that pushes forward the ordered noise distribution $P\#\mathbb{P}_{\bm{N}}$ towards the ordered observational one $P\#\mathbb{P}_{\bm{X}}$ and therefore its inverse (if it exists), defines the normalizing flow of the SCM. The map $T$, describing the static generative process of the SCM, is not restricted to be a TMI, and therefore cannot be recovered in general by the framework of~\cite{NEURIPS2023_b8402301}.
\section{Zero-Shot TO Inference}

Here we detail the zero-shot TO inference obtained by our model $\mathcal{M}$ on a new test dataset $\mathcal{D}_{\text{test}}\in\mathbb{R}^{n_{\text{test}}\times d_{\text{test}}}$. Note that the TO inferred coincides exactly with the one obtained by Algorithm~\ref{alg:one-forward-pass} when all the predicted leaves $\hat{\ell}$ defined in \textbf{line 4} of Algorithm~\ref{alg:inference} (or \textbf{line 6} of Algorithm~\ref{alg:one-forward-pass}) are true (sequential) leaves of the associated graph $\mathcal{G}_{\text{test}}$.

\begin{algorithm}[!h]
   \caption{TO Inference of $\mathcal{M}$}
   \label{alg:inference}
\begin{algorithmic}[1]
\STATE {\bfseries Input:} $\mathcal{M}$, $\mathcal{D}_{\text{test}}$
\STATE Initialize $\text{TO} = [\text{ }]$.
   \FOR{$k=1$ {\bfseries to} $d$}
   \STATE  $\hat{\ell} \gets \argmax_{i} 
   [\mathcal{M}(\mathcal{D}_{\text{test}})]_i$,\quad $\text{TO}\text{.append}(\hat{\ell})$
   \STATE $\mathcal{D}_{\text{test}}\gets \mathcal{R}_1(\mathcal{D}_{\text{test}},\hat{\ell})$
   \ENDFOR
\STATE Return $\text{TO}$
\end{algorithmic}
\end{algorithm}

\textbf{Improving TO Inference.} Assume that $\mathcal{M}$ has been trained to predict TOs  of various datasets of the form $\mathcal{D}_{\text{train}}\in\mathbb{R}^{n_{\text{train}}\times d_{\text{train}}}$, where $n_{\text{train}}$ and $d_{\text{train}}$ are the number of samples and the dimension respectively of each training dataset. Then the trained model can in principle take as input a test dataset $\mathcal{D}_{\text{test}}\in\mathbb{R}^{n_{\text{test}}\times d_{\text{test}}}$ of any size, that is with any $n_{\text{test}}\geq 1$ and $d_{\text{test}}\geq 1$ and return a permutation $\hat{P}\in\Sigma_{d_{\text{test}}}$ of the variables that should (ideally) correspond to a TO of the nodes in $\mathcal{D}_{\text{test}}$. When at test time, we have access to more than $n_{\text{train}}$ samples, that is $n_{\text{test}}\geq n_{\text{train}}$, we propose an assembling strategy to improve the prediction of our inferred TO. More precisely, leveraging the parallelism of the model as well as all the operators involved in Algorithm~\ref{alg:inference} w.r.t the number of datasets, we propose to build from $\mathcal{D}_{\text{test}}$,  $B_{\text{test}}$ smaller datasets $\mathcal{D}_{\text{test}}^{(1)},\dots,\mathcal{D}_{\text{test}}^{(B_{\text{test}})}\in\mathbb{R}^{n_{\text{train}}\times d_{\text{test}}}$ where $B_{\text{test}}:= n_{\text{test}} \div n_{\text{train}} $ and $ \div$ refers to the Euclidean division. Then we propose the procedure presented in Algorithm~\ref{alg:inference-improved} that assembles the predictions of each smaller datasets at each step to predict the most likely leaf. We precise that the operator $\textbf{vote}$ introduced in \textbf{line 5} of Algorithm~\ref{alg:inference-improved}, simply counts the number of apparition of each unique index that are present in the current list $[\hat{\ell}^{(1)},\dots,\hat{\ell}^{(B_{\text{test}})}]$ and returns one that has the maximum count.

\begin{algorithm}[!h]
   \caption{Improved TO Inference of $\mathcal{M}$}
   \label{alg:inference-improved}
\begin{algorithmic}[1]
\STATE {\bfseries Input:} $\mathcal{M}$, $\mathcal{D}_{\text{test}}^{(1)},\dots,\mathcal{D}_{\text{test}}^{(B_{\text{test}})}$ 
\STATE Initialize $\text{TO} = [\text{ }]$.
   \FOR{$k=1$ {\bfseries to} $d$}
   \STATE  $[\hat{\ell}^{(1)},\dots,\hat{\ell}^{(B_{\text{test}})}] \gets [\argmax_{i_1} [\mathcal{M}(\mathcal{D}_{\text{test}}^{(1)})]_{i_1},\dots, \argmax_{i_{B_{\text{test}}}} [\mathcal{M}(\mathcal{D}_{\text{test}}^{(B_{\text{test}})})]_{i_{B_{\text{test}}}}]$
   \STATE $\hat{\ell}\gets \textbf{vote}([\hat{\ell}^{(1)},\dots,\hat{\ell}^{(B_{\text{test}})}])$
   \STATE  $\text{TO}\text{.append}(\hat{\ell})$
   \STATE $[\mathcal{D}_{\text{test}}^{(1)},\dots,\mathcal{D}_{\text{test}}^{(B_{\text{test}})}]\gets [\mathcal{R}_1(\mathcal{D}_{\text{test}}^{(1)},\hat{\ell}),\dots,\mathcal{R}_1(\mathcal{D}_{\text{test}}^{(B_{\text{test}})},\hat{\ell})]$
   \ENDFOR
\STATE Return $\text{TO}$
\end{algorithmic}
\end{algorithm}

\section{Detailed Parametrization of the Fixed-Point SCM}
\label{sec:arch}
\subsection{Proposed Architecture}
In this section, we present a more in-depth explanation of our architectural approach for learning fixed-point SCMs.

\textbf{Causal Embedding.} The purpose of this layer is to embed into a higher dimensional space the samples without modifying the causal structure. Recall first that for $(\bm{X},\bm{N})\sim \gamma(P,\mathbb{P}_{\bm{N}},H)$, we have
\begin{align}
\label{eq:sample-struct}
    P\bm{X} = H(P\bm{X},P\bm{N}).
\end{align}
Let us now introduce two embedding functions for respectively $\bm{X}$ and $\bm{N}$. Let $D\gg 1$, and for $k\in\{1,2\}$, $E_k:\mathbb{R}^d\to\mathbb{R}^{d\times D}$ a differentiable function. Let also assume that for $k\in\{1,2\}$, $E_k$ is bijective (on its image space) and its inverse is also differentiable. In the following Proposition, we show a sufficient condition on these embedding functions to preserve the causal structure.
\begin{proposition}
\label{prop:enc}
Let us denote $\bm{X}_{\text{emb}} :=E_1(P\bm{X})$, and 
$\bm{N}_{\text{emb}} :=E_2(P\bm{N})$ the embedded random variables where $(\bm{X},\bm{N})\sim \gamma(P,\mathbb{P}_{\bm{N}},H)$. If $E_{1}$ and $E_{2}$ satisfy for all $x,n\in\mathbb{R}^d$, $i,j\in\{1,\dots,d\}$ and $k\in\{1,\dots,D\}$
\begin{align}
\label{cond-emb}
\begin{aligned}
[Jac~E_{1}(x)]_{i,k,j} &= 0,\quad \text{if}\quad i\neq j,\quad \text{and} \\ 
[Jac~E_{2}(n)]_{i,k,j} &= 0,\quad \text{if}\quad i\neq j,
\end{aligned}
\end{align} 
Then there exists a differentiable function $F:\mathbb{R}^{d\times D}\times\mathbb{R}^{d\times D} \to \mathbb{R}^{d\times D}$ such that 
\begin{align}
\label{eq:embed-fp}
  \bm{X}_{\text{emb}} =  F(\bm{X}_{\text{emb}}, \bm{N}_{\text{emb}})
\end{align}
and satisfying for all $i,j\in\{1,\dots,d\}$, $k,l\in\{1,\dots,D\}$:
\begin{equation}
\label{eq-conditions}
\begin{aligned}
[Jac_1 F(\cdot,\cdot)]_{i,k,j,l} &= \bm{0},~\text{if}~ [Jac_1 H(\cdot,\cdot)]_{i,j} &= \bm{0} \\ 
[Jac_2 F(\cdot,\cdot)]_{i,k,j,l} &= \bm{0},~\text{if}~[Jac_2 H(\cdot,\cdot)]_{i,j} &= \bm{0}.
\end{aligned}
\end{equation}
\end{proposition} 

\begin{proof}
Let $x,n\in\mathbb{R}^d$ such that $Px = H(Px,Pn)$, now observe that
\begin{align*}
E_1(Px) = E_1\circ H(Px, Pn) = E_1 \circ H (E_1^{-1} \circ E_1(Px), E_2^{-1}\circ E_2(Pn))
\end{align*}
Let us now define for all ($w,v\in\mathbb{R}^{d\times D}$, $F(w,v):= E_1 \circ H(E_1^{-1}(w), E_2^{-1}(v))$ and observe now that we have:
$$ E_1(Px) = F(E_1(Px),E_1(Pn))$$
Therefore for $(\bm{X},\bm{N})\sim \gamma(P,\mathbb{P}_{\bm{N}},H)$, we obtain that $(\bm{X}_{\text{emb}}:=E_1(P\bm{X}),\bm{N}_{\text{emb}}:=P\bm{N})$ is solving the following fixed-point problem:
$$ \bm{X}_{\text{emb}} = F(\bm{X}_{\text{emb}},\bm{N}_{\text{emb}})$$

Let us now show that this fixed-point problem defines a fixed-point SCM that satisfies the causal structure of the generating SCM. Now by definition of $F$, we obtain that
that for all $x,n\in E_1(\mathbb{R}^d)\times E_2(\mathbb{R}^d)$, 
\begin{align*}
Jac_1 (F)(x,n) &= Jac (E_1)(H(E_1^{-1}(x), E_2^{-1}(n))) Jac_1 (H)(E_1^{-1}(x), E_2^{-1}(n)) Jac (E_1^{-1})(x)\\
Jac_2 (F)(x,n) &= Jac (E_1)(H(E_1^{-1}(x), E_2^{-1}(n))) Jac_2 (H)(E_1^{-1}(x), E_2^{-1}(n)) Jac (E_2^{-1})(n)
\end{align*} 
where for a function $G$ and a point $z$ we denote $Jac(G)(z)$ the Jacobian of $G$ evaluated in $z$. Then, under condition~\ref{cond-emb}, as both $E_1, E_2$ and their restricted inverses $E_1^{-1}, E_2^{-1}$ are diagonal maps, we recover~\eqref{eq-conditions}.
\end{proof}

Therefore, as soon as~\eqref{cond-emb} is satisfied, the law of $(\bm{X}_{\text{emb}},\bm{N}_{\text{emb}})$ becomes the solution of a new fixed-point SCM induced by $F$ and~\eqref{eq-conditions} guarantees that its causal structure is the same as $H$. To satisfy~\eqref{cond-emb}, we propose simple embedding of the form:
\begin{align*}
    E_{i}(w):=[w_1 * \theta_{i,1},\dots,w_d *\theta_{i,d}]\in\mathbb{R}^{d\times D}
\end{align*}
where $w:=[w_1,\dots,w_d]\in\mathbb{R}^d$, $\bm{\theta}_i:=[\theta_{i,1},\dots,\theta_{i,d}]^T\in\mathbb{R}^{d\times D}$ with $\theta_{i,q}\in\mathbb{R}^{D}$  are learnable parameters. We also leverage the fact that the topological ordering is known to add a common positional encoding to these embedding. More formally we define 
\begin{align}
\label{eq:causal-emb-final}
    E_{i,\text{pos}}(w):=E_{\theta_i}(w) + \textbf{Pos}\in\mathbb{R}^{d\times D}
\end{align}
where $\textbf{Pos}\in\mathbb{R}^{d\times D}$ is a learnable parameter that encode the position. 
 
\textbf{Causal Attention.} Let us now introduce our new causal attention mechanism in order to model causal relationships. In classical attention,  given a key, and a query, denoted respectively $K,Q\in\mathbb{R}^{d\times D}$, where $d$ is the sequence length, and $D$ is the hidden dimension, the attention matrix is defined as:
\begin{align*}
 A(Q,K):=\text{softmax}(QK^{T}/\sqrt{D}),
\end{align*}
where for $M\in\mathbb{R}^{d\times d}$, $$[\text{softmax}(M)]_{i,j}:=\frac{\exp(M_{i,j})}{\sum_{k}\exp(M_{i,k})}.$$
In order to obtain a triangular mapping, it is common to add a causal masking to the attention weights. Generally the latter is obtained by defining a mask $M\in\{0,+\infty\}$ satisfying for all $i\leq j$, $M_{i,j}=0$ and for all $j>i$, $M_{i,j}=\infty$, and considering the following attention matrix:
\begin{align}
\label{eq:attn-classic}
 A_M(Q,K):=\text{softmax}((QK^{T} - M)/\sqrt{D}).
\end{align}
The main issue with the standard attention as defined in~\eqref{eq:attn-classic} is that the $\text{softmax}$ operator forces all the rows to sum to $1$, which means that all the nodes are forced to have at least one parent. In order to alleviate this issue and model correctly the root nodes, we propose to relax the definition of the attention layer, viewed as the solution of a specific (partial) optimal transport problem, in order to remove the constraints on the rows of the attention matrix. For that purpose let us denote $\Pi_{\mathbf{1}_d}:=\{W\in\mathbb{R}_{+}^{d\times d}\colon W\mathbf{1}_d=\mathbf{1}_d\}$ and $\Pi_{\mathbf{1}_d}^{\leq}:=\{W\in\mathbb{R}_{+}^{d\times d}\colon W\mathbf{1}_d\leq\mathbf{1}_d\}$, where $\mathbf{1}_d=[1,\dots,1]\in\mathbb{R}^d$. Let us now show that $A_M$ is the solution of a specific optimal transport problem.
\begin{proposition}
$A_M$ defined in~\eqref{eq:attn-classic} is the solution of the following (partial) and entropic optimal transport problem:
\begin{align}
\label{def:partial-eot}
\argmin_{W\in\Pi_{\mathbf{1}_d}}\langle W,C_M(Q,K)\rangle -\sqrt{D}H(W)
\end{align}
where $H(W):=-\sum_{i,j} W_{i,j}(\log(W_{i,j)}-1)$ is the generalized entropy and $C_M(Q,K):=-(QK^T - M)$.
\end{proposition}
\begin{proof}
This result is a direct consequence of the first order condition. Indeed at optimality, there exists $\bm{\lambda}\in\mathbb{R}^d$ such that
\begin{align*}
C_M(Q,K) + \sqrt{D}\log(W) + \bm{\lambda}\mathbf{1}_d=\bm{0}
\end{align*}
From which follows that $W= \exp((-C_M(Q,K) -\bm{\lambda}\mathbf{1}_d)/\sqrt{D})$ and as $W$ must satisfies the constraint the result follows.
\end{proof}
Let us now consider another masking $M\in\mathbb{R}^{d\times d}$, that is the matrix satisfying for all $i<j$, $[M]_{i,j}=0$ and for $j\geq i$, $[M]_{i,j}=\infty$. Note that the only difference with the traditional masking, is that here we also mask the diagonal in order to remove the edges from a node to itself. We are now ready to present our new causal attention mechanism.
\begin{definition}[Causal Attention]
For $Q,K\in\mathbb{R}^{d\times D}$, we define the causal attention matrix $\text{CA}_{M}(Q,K)$
as the solution of the following relaxed, (partial) and entropic optimal transport problem:
\begin{align}
\label{def:relaxed-partial-eot}
\argmin_{W\in\Pi_{\mathbf{1}_d}^{\leq}}\langle W,C_{M}(Q,K)\rangle -\sqrt{D}H(W).
\end{align}
where $C_{M}(Q,K):=-(QK^T - M_1)$.
\end{definition}
It happens that the solution of~\eqref{def:relaxed-partial-eot} is unique and can be derived in closed form:
\begin{align*}
    \text{CA}_{M}(Q,K)=\frac{\exp((QK^T - M)/\sqrt{D})}{\mathcal{V}(\exp((QK^T - M)/\sqrt{D})\mathbf{1}_d)}
\end{align*}
where for $v:=[v_1,\dots,v_d]\in\mathbb{R}^d_{+}$ and $i\in\{1,\dots,d\}$,

$$
[\mathcal{V}(v)]_i = \begin{cases}
    v_i,& \text{if } v_i\geq 1\\
    1,              & \text{otherwise.}
\end{cases}
$$
By relaxing the constraint of the optimal transport problem, we obtain a new attention mechanism that handles the existence of roots in a causal graph, which cannot be captured by the standard attention.

\textbf{Causal Encoder.} Let us now present the main building block of our proposed architecture. Here, we aim at parametrizing the function $F$ introduced in~\eqref{eq:embed-fp}. To do so, let us consider $(\bm{X}_{\text{emb}}, \bm{N}_{\text{emb}})\in\mathbb{R}^{d\times D}\times \mathbb{R}^{d\times D}$ some inputs, $W_Q, W_K, W_V\in\mathbb{R}^{D\times D}$, three learnable parameters and let us define:
\begin{align*}
&Q(\bm{N}_{\text{emb}}):=\bm{N}_{\text{emb}}W_Q,~K(\bm{X}_{\text{emb}}):=\bm{X}_{\text{emb}}W_K,\\
&V(\bm{X}_{\text{emb}}):=\bm{X}_{\text{emb}}W_V.
\end{align*} 
Let us also denote $h:\mathbb{R}^{D}\to\mathbb{R}^{D}$a parametric function. The proposed encoder layer $\mathcal{C}$ is defined as the following operator:
\begin{equation}
\label{def:DL}
\begin{gathered}
\mathcal{C}(\bm{X}_{\text{emb}}, \bm{N}_{\text{emb}}):=\\
h(\text{CA}_{M}(Q(\bm{N}_{\text{emb}}),K(\bm{X}_{\text{emb}})) V(\bm{X}_{\text{emb}}) + \bm{N}_{\text{emb}})
\end{gathered}
\end{equation}
where we define for $W:=[W_1,\dots,W_d]^T\in\mathbb{R}^{d\times D}$ with $W_i\in\mathbb{R}^{D}$, $h(W):=[h(W_1),\dots,h(W_d)]^T\in\mathbb{R}^{d\times D}$. We omit the dependence of $\mathcal{C}$ with the parameters to simplify the notation.
As in a classical transformer, we consider $h$ the composition of a layer norm operator ($\text{LN}$) and a multi-layer perceptron ($\text{MLP}$):
\begin{align*}
    h(x)=\text{LN}\circ (\text{I}_{D} + \text{MLP})\circ \text{LN}(x).
\end{align*}
Let us now show that the proposed layer satisfies the constraints of a fixed-point SCM.
\begin{proposition}
\label{prop:single-enc}
Let $\mathcal{C}$ as defined in~\eqref{def:DL}. Then for all $x,n\in\mathbb{R}^{d\times D}$, $i,j\in\{1,\dots,d\}$ and $k,l\in\{1,\dots,D\}$, we have
\begin{equation}
\label{eq:causal-enc-1-layer}
\begin{aligned}
[Jac_1 \mathcal{C}(x,n)]_{i,k,j,l} &= 0,\quad \text{if}\quad j\geq i,\quad \text{and} \\ 
[Jac_2 \mathcal{C}(x,n)]_{i,k,j,l} &= 0,\quad \text{if}\quad i\neq j.
\end{aligned}
\end{equation}
\end{proposition}
\begin{proof}
First observe that $h$ is applied coordinate-wise, therefore when viewed as an operator from $\mathbb{R}^{d\times D}\to \mathbb{R}^{d\times D}$, its Jacobian is diagonal w.r.t to the first dimension. Now let us define for $x,n\in\mathbb{R}^{d\times D}$
$$g(x,n):=\text{CA}_{M}(Q(n),K(x)) V(x) + n$$
and observe that the $i$-th row of $\text{CA}_{M}(Q(n),K(x))$ only depends on the $i$-th row $Q(n)$ and the $i-1$ first rows of $K(x)$. In addition, because $\text{CA}_{M}(Q(n),K(x))$ is strictly lower-triangular, $\text{CA}_{M}(Q(n),K(x)) V(x)$ has exactly the same dependencies. Then we deduce directly that for all $x,n\in\mathbb{R}^{d\times D}$, $i,j\in\{1,\dots,d\}$ and $k,l\in\{1,\dots,D\}$,
\begin{equation*}
\begin{aligned}
[Jac_1 g(x,n)]_{i,k,j,l} &= 0,\quad \text{if}\quad j\geq i,\quad \text{and} \\ 
[Jac_2 g(x,n)]_{i,k,j,l} &= 0,\quad \text{if}\quad i\neq j.
\end{aligned}
\end{equation*}
from which the result follows.
\end{proof}
Therefore the proposed causal encoder layer can parameterize a whole family of fixed-point SCM in a latent space where the nodes are represented by vector of dimension $D$. In order to further increase its complexity, we propose now to compose them. Starting with an embedded sample $\bm{X}_{\text{emb}}:=E_{1,\textbf{Pos}}(P\bm{X})\in\mathbb{R}^{d\times D}$ and an embedding of noise, that is  $\bm{N}^{(0)}_{\text{emb}}:=E_{2,\textbf{Pos}}(P\bm{N})$, we compute for $k\in\{1,\dots,L-1\}$:
\begin{align*}
    \bm{N}^{(k+1)}_{\text{emb}}:=\mathcal{C}( \bm{X}_{\text{emb}}, \bm{N}^{(k)}_{\text{emb}}).
\end{align*}
Let us now show that this composition is still a valid fixed-point SCM.
\begin{proposition}
\label{prop-causal-enc}
Let $\mathcal{C}$ as defined in~\eqref{def:DL}, $x,n\in\mathbb{R}^{d\times D}$ and let us define $\mathcal{C}_L(x,n):=\mathcal{C}(x,\cdot)^{\circ L}(n)$. Then we have for all $i,j\in\{1,\dots,d\}$ and $k,l\in\{1,\dots,D\}$:
\begin{equation*}
\begin{aligned}
[Jac_1 \mathcal{C}_L(x,n)]_{i,k,j,l} &= 0,\quad \text{if}\quad j\geq i,\quad \text{and} \\ 
[Jac_2 \mathcal{C}_L(x,n)]_{i,k,j,l} &= 0,\quad \text{if}\quad i\neq j.
\end{aligned}
\end{equation*}
\end{proposition}
\begin{proof}
This results is a direct consequence of proposition~\ref{prop:single-enc}. Indeed because $Jac_2 [\mathcal{C}(x,n)]_{i,k,j,l} = 0,\quad \text{if}\quad i\neq j$, then composing w.r.t the second variable does not modify the structure of the Jacobian. More formally let $\mathcal{C}:\mathbb{R}^{d\times D}\times \mathbb{R}^{d\times D}\to \mathbb{R}^{d\times D}$ a function satisfying~\eqref{eq:causal-enc-1-layer}, then we can show by recursion that~\eqref{eq:causal-enc-1-layer} still holds when composing them w.r.t the second variable. Indeed assume it is the case after $k$ composition then we have for $x,n\in\mathbb{R}^{d\times D}\times\mathbb{R}^{d\times D}$:
\begin{align*}
    \mathcal{C}_{k+1}(x,n) = \mathcal{C}(x,\mathcal{C}_k(x,n))
\end{align*}
Then taking the jacobian we obtain:
\begin{align*}
Jac_1 (\mathcal{C}_{k+1})(x,n) = Jac_1(\mathcal{C})(x,\mathcal{C}_k(x,n)) + Jac_2(\mathcal{C})(x,\mathcal{C}_k(x,n)) Jac_1(\mathcal{C}_k)(x,n)
\end{align*}
however as $Jac_1(\mathcal{C})$ and $ Jac_1(\mathcal{C}_k)$ have the same structure, and $Jac_2(\mathcal{C})$ is diagonal, this structure is preserved. Similarly, we have that
\begin{align*}
Jac_2 (\mathcal{C}_{k+1})(x,n) = Jac_2(\mathcal{C})(x,\mathcal{C}_k(x,n)) Jac_2(\mathcal{C}_k)(x,n))
\end{align*}
which is still diagonal  w.r.t the first dimension as both $Jac_2(\mathcal{C})$ and $Jac_2(\mathcal{C}_k)$ are diagonals. Then the result follows.
\end{proof}

\textbf{Causal Decoder.} Let us now present our causal decoder that aims at bringing back the output of the causal encoder into the original space $\mathbb{R}^d$ without affecting the causal structure. For that purpose we aim at designing a function $\mathcal{J}:\mathbb{R}^{d\times D}\to\mathbb{R}^d$ that conserves the structure of $\mathcal{C}_L$. More formally $\mathcal{J}$ has to satisfy the following constraints for $i,j\in\{1,\dots,d\}$ and $l\in\{1,\dots,D\}$:
\begin{equation}
\label{dec:cond}
\begin{aligned}
[Jac_1 \mathcal{J}\circ\mathcal{C}_L(\cdot,\cdot)]_{i,j,l} &= \bm{0},~\text{if}~ \forall k,q, ~[Jac_1 \mathcal{C}_L(\cdot,\cdot)]_{i,k,j,q} = \bm{0} \\ 
[Jac_2 \mathcal{J}\circ\mathcal{C}_L(\cdot,\cdot)]_{i,j,l} &= \bm{0},~\text{if}~\forall k,q,~[Jac_2 \mathcal{C}_L(\cdot,\cdot)]_{i,k,j,q} = \bm{0}.
\end{aligned}
\end{equation}
Let us now present a sufficient condition to satisfy the above conditions.
\begin{proposition}
\label{prop:cond-dec}
If $\mathcal{J}$ satisfies for all $x\in\mathbb{R}^{d\times D}$, $i,j\in\{1,\dots,d\}$, and $l\in\{1,\dots,D\}$:
\begin{equation*}
    [Jac \mathcal{J}(x)]_{i,j,l}=0,~\text{if}~i\neq j
\end{equation*}
then $\mathcal{J}$ satisfies~\eqref{dec:cond} and preserves the structure of $\mathcal{C}_L$ .
\end{proposition}
\begin{proof}
This simply follows the composition of Jacobians. Indeed, under the assumption, we obtain for all $x,n\in\mathbb{R}^{d\times D}$,
\begin{align*}
Jac_1 \mathcal{J}\circ\mathcal{C}_L(x,n)= Jac (\mathcal{J})(\mathcal{C}_L(x,n))Jac_1 (\mathcal{C}_L)(x,n)\\
Jac_2 \mathcal{J}\circ\mathcal{C}_L(x,n)= Jac (\mathcal{J})(\mathcal{C}_L(x,n))Jac_2 (\mathcal{C}_L)(x,n)
\end{align*}
but as $\mathcal{J}$ is diagonal, the result follows.
\end{proof}

In order to satisfy the condition obtained in proposition~\ref{prop:cond-dec}, we propose a simple decoder layer defined for $x:=[x_1,\dots,x_d]^T\in\mathbb{R}^{d\times D}$ as:
\begin{align}
\label{def:dec}
\mathcal{J}(x):=[\langle x_1,w_1\rangle,\dots,\langle x_d,w_d\rangle]\in\mathbb{R}^d
\end{align}
where $w_i\in\mathbb{R}^D$ are learnable parameters.

\textbf{Final Parameterization.} The final proposed architecture $\mathcal{T}$ obtained can is defined for $x,n\in\mathbb{R}^d$ as:
 \begin{align*}
     \mathcal{T}(x,n):=\mathcal{J}\circ\mathcal{C}_L(E_{1,\textbf{Pos}}(x),E_{2,\textbf{Pos}}(n))\in\mathbb{R}^d.
 \end{align*}
Using the proposed parameterization for $\mathcal{D}, \mathcal{C}_L$, and $E_{i}$ we have the following corollary showing that $\mathcal{T}$ satisfies the structural constraints.
\begin{corollary}
For all $x,n\in\mathbb{R}^d$ and $i,j\in\{1,\dots,d\}$, we have
\begin{equation*}
\begin{aligned}
[Jac_1 \mathcal{T}(x,n)]_{i,j} &= 0,\quad \text{if}\quad j\geq i,\quad \text{and} \\ 
[Jac_2 \mathcal{T}(x,n)]_{i,j} &= 0,\quad \text{if}\quad i\neq j\; ,
\end{aligned}
\end{equation*}
therefore $\mathcal{T}\in\mathcal{F}_d$.
\end{corollary}

To summarize, our architecture allows to embed into a higher dimensional space an SCM while conserving its structure using the causal embedding, parameterize the set of valid SCM in the latent space using the causal encoder, and then bring back the encoded SCM into the original space without modifying its structure.

 \subsection{Training, Generation and Inference}
 
\textbf{Training.} Recall that in this work, we only focus in a restricted case of ANM, and we consider models of the form $\mathcal{T}_{\text{ANM}}(x,n):=\mathcal{T}(x,0_d) + n$.
Our goal now is to learn $\mathcal{T}_{\text{ANM}}$ such that we recover a generating fixed-point SCM of $\mathbb{P}_{\bm{X}}$ given $n$ samples $\bm{X}_i$ and the topological ordering $P$. To do so, we propose to minimize the mean squared error (MSE), that is:
\begin{align}
\label{eq:MSE}
    \mathbb{E}_{x\sim \mathbb{P}_{\bm{X}}}\Vert x - P^T\mathcal{T}_{\text{ANM}}(Px,0_d)\Vert_2^2.
\end{align} 
We show that if the generating fixed-point SCM is an ANM, then we can recover it uniquely by minimizing~\eqref{eq:MSE}.
\begin{proposition}
\label{prop:anm-id-train}
Let $\mathbb{P}_{\bm{X}}\in\mathcal{P}_2(\mathbb{R}^d)$ and $P\in\Sigma_d$. Assume that there exists $(P,\mathbb{P},H)\in\mathcal{A}_P^{\text{ANM}}(\mathbb{P}_{\bm{X}})$ such that $\mathbb{P}\in\mathcal{P}_2(\mathbb{R}^d)$. Let us also denote $\mathcal{H}_d:=\{h:\mathbb{R}^d\to\mathbb{R}^d\colon
\text{$h$ is differentiable and } [Jac~h(x)]_{i,j}=0~\text{if}~j\geq i\}$.
Then the problem 
\begin{align}
\label{eq:MSE-obj-gene}
    \min_{h\in\mathcal{H}_d}  \mathbb{E}_{x\sim \mathbb{P}_{\bm{X}}}\Vert x - P^Th(Px)\Vert_2^2
\end{align}

admits $\mathbb{P}_{P\bm{X}}$ a.s. a unique solution $h^{*}$ and by denoting $\mathbb{P}:=P^T\circ(\text{I}_d - h^{*})\#\mathbb{P}_{P\bm{X}}$ and $H(x,n)=h^{*}(x)+n$, we have that $(P,\mathbb{P}, H)$ is $\mathbb{P}_{P\bm{X}}$ a.s. the unique element of $\mathcal{A}_P^{\text{ANM}}(\mathbb{P}_{\bm{X}})$.
\end{proposition}

\begin{proof}
This results follows directly from the resolution of the MSE minimization. We can assume without loss of generality that $P=I_d$. Then for all $i\in\{1,\dots,d\}$, by conditioning on the previous $\bm{X}_{<i}:=[X_1,\dots,X_{i-1}]$ and taking the expectation, a simple calculation gives that the optimal solution is unique $\mathbb{P}_{\bm{X}}$ a.s. and satisfy for all $i$ $h_i(x):=\mathbb{E}(X_i|\bm{X}_{<i}=x_{<i})$ where $x_{<i}:=[x_1,\dots,x_{i-1}]$ and $h:=[h_1,\dots,h_d]$. Now because $\mathbb{P}_{\bm{X}}$ is generated by a fixed-point SCM with the same topological ordering, and thanks to Proposition~\ref{prop:anm}, we deduce directly that $H(x,n)=h(x)+n$. Then, the exogenous distribution follows directly.
\end{proof}

\begin{remark}
It it worth noting that in Proposition~\ref{prop:anm-id-train}, we show that, given the observations of any generative SCM (ANM or not) and a topological ordering, the optimization problem defined in~\eqref{eq:MSE-obj-gene} admits a unique solution that is the conditional expectancies w.r.t the parents for each nodes. More formally, the minimizer of our optimization problem is always the function $h^*:=[h_1^*,\dots,h_d^*]$ where $h_i^*(x):=E_{\bm{X}\sim P_{\bm{X}}}(X_i|Pa(X_i)=Pa(x_i))$. Therefore, whatever the generative SCM is (ANM or not), our procedure is always able to recover the ground truth graph under the assumption that the $h_i^*$’s satisfy the structural minimality assumption. In addition, if the generative SCM happens to be an ANM, then we show that solving~\eqref{eq:MSE-obj-gene}  recovers it uniquely, that is $(x,n)\to h^*(x)+n$ is the only possible choice for the generative fixed-point SCM.
\end{remark}

By construction, $x\to\mathcal{T}_{\text{ANM}}(x,0_d)$ is an element of $\mathcal{H}_d$ and therefore can be used to recover the causal ANM. In practice we would rather minimize  $\mathbb{E}_{x\sim \hat{\mathbb{P}}_{\bm{X}}}\Vert Px - \mathcal{T}(Px,0_d)\Vert_2^2$ where $\hat{\mathbb{P}}_{\bm{X}}:=1/n\sum_{i=1}^n\delta_{\bm{X}_i}$ is the empirical distribution. 

Once $\mathcal{T}_{\text{ANM}}$ is trained by minimizing~\eqref{eq:MSE}, we can define the exogenous distribution of the model as $\mathbb{P}_{\bm{N}}:=P^{T}\circ (\text{I}_d - \mathcal{T}_{\text{ANM}}(\cdot,0_d))\#\mathbb{P}_{P\bm{X}}$. However as we only have access to samples from $\mathbb{P}_{\bm{X}}$, we can only generates the associated samples of $\mathbb{P}_{\bm{N}}$. In order to build a complete generative model, we need also to be able to sample according to $\mathbb{P}_{\bm{N}}$. To do so, we propose to estimate simple functions using the associated samples of the exogenous distribution, that are the $P\bm{N}_i:=P\bm{X}_i - \mathcal{T}(P\bm{X}_i,0_d)$. First recall that $\mathbb{P}_{\bm{N}}$ is a jointly independent distribution and let us denote it $\mathbb{P}_{\bm{N}}=\otimes_{i=1}^d\mathbb{P}_{N_i}$. In order to learn a generating process to obtain new samples from this distribution, we propose to solve these following 1-d problems:
\begin{align*}
\text{find } g_i:\mathbb{R}\to\mathbb{R}\colon g_i\#\mathbb{U}=\mathbb{P}_{N_i}~\forall~i\in\{1,\dots,d\},
\end{align*}
where $\mathbb{U}$ is the 1-d uniform distribution on [0,1]. To do so we propose to minimize the optimal transport distance~\cite{villani2009optimal} and we show that it allows us to recover the exogenous distribution.
\begin{proposition}
\label{prop:generation-noise}
Let $\mathbb{P}_{\bm{N}}\in\mathcal{P}(\mathbb{R})^{\otimes d}$ and assume it is continuous. Then
\begin{align}
\label{eq:simple-1-d-appendix}
\min_{g_k:\mathbb{R}\to\mathbb{R}} \text{OT}( g_k\#\mathbb{U},\mathbb{P}_{N_k}),\quad \forall~k\in\{1,\dots,d\},
\end{align}
admits a solution and by defining for $x\in\mathbb{R}^d$, $g(x):=[g_1(x_1),\dots,g_d(x_d)]$, we have
$$g\#\otimes_{i=1}^d \mathbb{U}=\mathbb{P}_{\bm{N}}.$$
\end{proposition}
\begin{proof}
The existence of the solution follows directly from~\cite{villani2009optimal} thanks to the continuity of both the source and the target probability measures. Then using the independence of $\mathbb{P}_{\bm{N}}$, the second equality follows directly.
\end{proof}
The resolution~\eqref{eq:simple-1-d-appendix} is a well-studied problem and can be solved for example by estimating the quantile functions of each $\mathbb{P}_{N_i}$. Then once the $g_i$ are estimated, we can use them to obtain new sample of $\mathbb{P}_{\bm{N}}$ by drawing $d$ i.i.d samples from $\mathbb{U}$.

\textbf{Causal Inference.} Once the model is trained by minimizing~\eqref{eq:MSE}, that is $\mathcal{T}_{\text{ANM}}$ is learned, one has now access to the full SCM and can perform any causal operations. More precisely, for any $\mathcal{T}\in\mathcal{F}_d^{\text{ANM}}$, we can define  $\text{NS}(\mathcal{T}):\mathbb{R}^d\to\mathbb{R}^d$ that transforms a noise sample to the associated data sample and defined as
\begin{align*}
    \text{NS}(\mathcal{T})(n)=x_{\mathcal{T}}(n)=\mathcal{T}(\cdot,n)^{\circ d}
\end{align*}
and $\text{SN}(\mathcal{T}):\mathbb{R}^d\to\mathbb{R}^d$ that transforms a data sample to the noise associated and defined as:
\begin{align*}
\text{SN}(\mathcal{T})(x)=x_{\mathcal{T}}^{-1}(x):=x - \mathcal{T}(x,0_d).
\end{align*}
Given these two operators, we can now generate new samples
and perform counterfactual computations. More formally, $\text{NS}(\mathcal{T}_{\text{ANM}})$ allows us to generate a sample from the observational distribution given a sample from the exogenous one. In addition, by modifying directly $\mathcal{T}_{\text{ANM}}$, one can define a new function $\mathcal{T}_{\text{ANM}}^{\text{do}}$ that also induces a fixed-point SCM living in $\mathcal{F}_d^{\text{ANM}}$, and compute $\text{NS} (\mathcal{T}_{\text{ANM}}^{\text{do}})\circ \text{SN}(\mathcal{T}_{\text{ANM}})(x)$ on a data point $x$ in order to obtain the counterfactual sample of $x$ associated to the operation $\text{do}$.

\subsection{Discussion on the Fixed-Points of FiP} 

Observe that the model is designed to ensure that the observations are the fixed-points of the model and so for all the iterations of the training (even at initialization). Therefore, the fixed-point error of our model is $0$. In fact, it is rather the functional relationship $\mathcal{T}_{\text{ANM}}$ as well as the noise distribution induced by $\mathcal{T}_{\text{ANM}}$ that improves over time during training.

More formally, assume for simplicity that $P=I_d$. Then at each iteration of the algorithm, we can define the noise distribution induced by our current model $\mathcal{T}_{\text{ANM}}$ as the law $\tilde{\bm{N}} := \bm{X} - \mathcal{T}_{\text{ANM}}(\bm{X},0)$. During training, we are minimizing (w.r.t the parameters of $\mathcal{T}_{\text{ANM}}$) the second moment of this distribution, that is~\ref{eq:MSE-training}, which forces it to get closer to the true noise distribution under additive assumption. However, for any network $\mathcal{T}_{\text{ANM}}$ and associated noise distribution of the form $\tilde{\bm{N}} = \bm{X} - \mathcal{T}_{\text{ANM}}(\bm{X},0)$ where $\bm{X}$ follows the observational distribution, we always have that $\bm{X} = \mathcal{T}(\bm{X},0)+ \tilde{\bm{N}}$, in addition $\bm{X}$ is the only fixed-point of $x\to \mathcal{T}_{\text{ANM}}(x,0)+ \tilde{\bm{N}}$. This is ensured thanks to our inductive bias that forces $\mathcal{T}_{\text{ANM}}$ to be strictly lower-triangular w.r.t $\bm{X}$, and therefore guarantees that the fixed-point problem  on $z$, $z=\mathcal{T}(z,0) + \tilde{\bm{N}}$, admits a unique solution that must be $\bm{X}$.

\subsection{Discussion on the Approximation Power of FiP}

In~\cite{yun2019transformers}, the authors show that transformers can universally approximate arbitrary continuous sequence-to-sequence functions on a compact domain. In our setting, by viewing an observable sample $\bm{X}=[X_1,\dots,X_d]$ as a finite sequence of size $d$ ranked in the topological order, an SCM can be seen as a conditional sequence-to-sequence map that given the noise $\bm{N}$, maps the sequence $\bm{X}$ to itself. While the extension of their results to our transformer-based network is out of scope of this paper, we demonstrate experimentally that our architecture is still able to efficiently approximate SCMs on various problems.

\section{Additional Information on Experiments}
\label{sec:add-experiment}

\textbf{Synthetically Generated Datasets.} To obtain the two SCM distribution $\mathbb{P}_{\text{IN}}$ and $\mathbb{P}_{\text{OUT}}$  described in~\ref{sec:eval-amortization}, we reproduce exactly the setting of~\cite{lorch2022amortized} in Appendix A Table 3, with the difference that we do not consider Cauchy distributions for the exogenous variables, as they are not integrable and therefore the MSE minimization problem is not well defined, as well as the geometric random graphs distributions, as they tend to produce graphs without any edges when setting them with a small radius.

\paragraph{Optimization of $\mathcal{M}$.} Recall that $\mathcal{M}$ uses the exact same encoder as the one proposed in~\cite{lorch2022amortized}, on the top of which we add a simple linear layer to classify the encoded nodes whether they are leaves or not. The description of their encoder can be found in Appendix C.2 of~\cite{lorch2022amortized}. To train the model, we use the Adam implementation of Pytorch~\cite{paszke2017automatic} with a learning rate of $1e-4$ with a weight decay of $5e-9$. We run the training for 2000 epochs, where each epochs contains 96 newly generated datasets from $\mathbb{P}_{\text{IN}}$. More precisely, for each configuration of distributions of the training $\mathbb{P}_{\text{IN}}$, we sample 16 SCMs, from which we obtain 16 pairs $(\mathcal{D}_{\text{tr}},\mathcal{G}_{\text{train}})$. Each dataset is of size $200\times 100$, where $n_{\text{train}}=200$ and $d_{\text{train}}=100$ denote the sample size and the dimension (or the number of nodes) respectively. We use 4 A100 GPUs with a total of 320 GiB of memory and 85 CPUs to train our architecture $\mathcal{M}$. The total batch size (cross GPUs) used is 8.

\textbf{Optimization of $\mathcal{T}_{\text{ANM}}$}. As explained in section~\ref{sec:eval-fp-scm}, we consider a small architecture with only 2 layers $L=2$, $d_{\text{head}}=32$, with $8$ heads, an a latent dimension of $D=128$. Also recall that our causal encoder uses a MLP, that is as in the classical transformer a fully connected network with two layers and a ReLU activation where the hidden dimension is set to $d_{\text{hidden}}:=128$. Given a dataset $\mathcal{D}\in\mathbb{R}^{n_{\text{tot}}
\times d}$ where $n_{\text{tot}}$ is the total number of samples, we split it into three datasets w.r.t the sample size, with ratio $0.8$, $0.1$, $0.1$ for training, validation and testing respectively. We consider a batch size of $n_{\text{batch}}:=\min(1024, 0.8 * n_{\text{tot}})$ to train $\mathcal{T}_{\text{ANM}}$. Finally, we consider the same optimizer as the one used to train $\mathcal{M}$.

\textbf{Additional Datasets.} In addition of our test metadatasets defined in~\ref{sec:eval-amortization}, we consider C-Suite~\cite{geffner2022deep}, SynTReN~\cite{van2006syntren} and the real-world dataset of protein measurements from~\cite{sachs2005causal}. C-suite consists of various discrete, mixed and continuous datasets but we only consider the continuous ones that are: lingauss, linexp, nonlingauss, nonlin simpson, symprod simpson, large backdoor, and weak arrows. These datasets admit different size of variables ranging from $d=2$ to $d=9$ nodes and test specific structures to assert the performance of models. For these datasets we generate $n_{\text{tot}}=10000$ samples. SynTReN creates
synthetic transcriptional regulatory networks and produces simulated gene expression data that mimics experimental data. We use the datasets generated by~\cite{lachapelle2019gradient} that consists of 5 datasets of $n_{\text{tot}}=500$ samples with $d=20$ nodes. Finally, Proteins cells consists of one true world dataset of $n_{\text{tot}}=853$ samples with $d=11$. Following~\cite{geffner2022deep}, we generate multiple of them by randomly sub-sampling $800$ samples and create 5 datasets from it.

\textbf{Baselines.} In this work, we compare our methods with the following baselines:
\begin{itemize}
    \item AVICI~\cite{lorch2022amortized} using the trained models available at \hyperlink{}{https://github.com/larslorch/avici/tree/main}. When apply on LIN $\textbf{OUT}$, we use their model trained only on linear functional relationships, while on RFF $\textbf{OUT}$, we use the specific model trained for RFF functions.
    \item GES~\cite{chickering2002optimal}, GOLEM~\cite{ng2020role}, DAG-GNN~\cite{yu2019dag}, GraN-DAG~\cite{lachapelle2019gradient} using the implementations of~\cite{zhang2021gcastle} available at 
    \hyperlink{}{https://github.com/huawei-noah/trustworthyAI/tree/master/gcastle}.
    \item   DP-DAG~\cite{charpentier2022differentiable}, using their implementation available at 
        ~\hyperlink{}{https://github.com/sharpenb/Differentiable-DAG-Sampling}.
    \item DECI~\cite{geffner2022deep}, using their implementation available at
        ~\hyperlink{}{https://github.com/microsoft/causica}.
    \item Dowhy~\cite{dowhy_gcm}, using their implementation available at 
        ~\hyperlink{}{https://github.com/py-why/dowhy}.
\end{itemize}
While some of these methods should in principle be able to compute counterfactual samples, the only implementations that offer such computations are DECI and DoWhy and we therefore only compare to them for counterfactual predictions.

\textbf{Computation of the Causal Graph.} To obtain a binary graph from $\mathcal{T}_{\text{ANM}}$, we first estimate a continuous graph defined as the mean of its  absolute value Jacobian over samples, i.e. $\hat{\mathcal{G}_c}:=\mathbb{E}_{\bm{X}\sim \hat{P}_{\bm{X}}}|\text{Jac}_1 P^T \mathcal{T}_{\text{ANM}}(P\bm{X},0_d)|$ where $\hat{P}_{\bm{X}}$ are the train samples. Then, we obtain the binary graph by applying a naive uniform thresholding $\tau=0.1$, i.e. $\hat{\mathcal{G}}(\tau):=(\hat{\mathcal{G}_c}>\tau)$, thus discarding values smaller than $\tau$.

\textbf{Counterfactual Generation.} To evaluate the causal inference of our model, we propose to predict counterfactual samples. To measure the quality of these predicted samples, we consider only settings where we have access to the simulators in order to generate new counterfactual samples. Therefore we only consider the test datasets introduced in~\ref{sec:eval-amortization}, that are LIN \textbf{IN}, LIN \textbf{OUT}, RFF \textbf{IN}, RFF \textbf{OUT} and C-suite. For datasets of size $\mathcal{D}^{n\times d}$, we repeat $d$ times the following procedure: (1) we select randomly a node $k\in\{1,\dots,d\}$ among the $d$ nodes, (2) then we randomly sample a value in the range of this variable by drawing a sample from the uniform distribution of $[\min(X_k),\max(X_k)]$ where the $\min$ ans $\max$ are taken w.r.t the available samples. (3) Finally we generate 100 new samples according to this interventions using new observations sampled according the true generative SCM.

\section{Additional Results}

\subsection{Graph Discovery}

\textbf{More details on the Causal Discovery Experiments.} In tables~\ref{benchmark-table-f1-rff-out}, \ref{benchmark-table-f1-lin-out}, we show more detailed versions of the results presented in table~\ref{table:f1-pred-baseline}, where we report the mean and the standard deviation of the oriented F1 scores obtained by our method and other baselines on the out-of-distributions metadatasets RFF \textbf{OUT} and LIN \textbf{OUT} respectively and so for each dimension $d\in\{10, 20, 50\}$ of the problems considered. Additionally, in tables~\ref{benchmark-table-f1-rff-in}, \ref{benchmark-table-f1-lin-in}, we also compare the oriented F1 score obtained by our method with other baselines on the in-distribution metadatasets RFF \textbf{IN} and LIN \textbf{IN}. Note that in all these tables, we add the results obtained by FiP when either the true topological ordering $P$ or the true graph $\mathcal{G}$ is given.

\textbf{A More Precise Evaluation of FiP for Causal Discovery.} While our naive threshold rule, that is $\hat{\mathcal{G}}(\tau):=(\hat{\mathcal{G}_c}>\tau)$ with $\tau=0.1$, to obtain the causal graph allows to recover most of the edges, we observe in tables~\ref{benchmark-table-f1-rff-out}, \ref{benchmark-table-f1-lin-out},\ref{benchmark-table-f1-rff-in}, \ref{benchmark-table-f1-lin-in} and in figure~\ref{fig:f1-score-eval} that when the true graph $\mathcal{G}$ is given to FiP, we do not obtain a perfect F1 score. This is because some of the true edges $(i,j)$ in the continous graph $\hat{\mathcal{G}_c}$ have a score lower than the prescribed threshold $\tau =0.1$ and therefore are discarded by our naive rule. With the aim of achieving greater precision in the evaluation of FiP's performance, we report the directed scores to the
ground-truth binary causal graphs by using the area under the receiver operating characteristic curve (AUC-ROC) and the area under the curve of precision-recall-gain (AUC-PRG)~\cite{flach2015precision}. Note that the latter is more adapted to measure the F1 score performance rather than the more usual AUC-PR, as AUC-PRG respects the harmonic scale of the F1 score~\cite{flach2015precision} and therefore is a better measure to assert the performances of FiP in term of F1 scores. Both AUC-ROC and AUC-PRG have the
important advantage to be independent of a threshold choice $\tau$~\cite{vowels2022d} and therefore offer a better understanding on the performances of our method. In table~\ref{eval-auc-roc-fip-rff-out}, \ref{eval-auc-roc-fip-lin-out}, \ref{eval-auc-roc-fip-rff-in}, \ref{eval-auc-roc-fip-lin-in}, we report the directed AU-ROC scores obtained by our method on the metadatasets presented in section~\ref{sec:eval-amortization}, and in table~\ref{eval-auc-pr-fip-rff-out}, \ref{eval-auc-pr-fip-lin-out}, \ref{eval-auc-pr-fip-rff-in}, \ref{eval-auc-pr-fip-lin-in}, we report the directed AUC-PRG obtained by our method on the same metadatasets. We show that FiP is able to obtain almost consistently very high accuracy both in term of AUC-ROC and AUC-PRG. Note that when the true graph $\mathcal{G}$ is provided to FiP, we obtain, as expected, a perfect score in term of AUC-PRG, therefore the model still learns the fixed-point SCM using all the available (true) edges.

\begin{table*}[!t]
\caption{We compare the directed F1 scores obtained by our model against various baselines on the out-of-distribution test metadataset RFF \textbf{OUT} introduced in section~\ref{sec:eval-amortization}. For each dimension $d\in\{10, 20, 50\}$, we report the mean and the standard deviation of the F1 scores obtained over all the datasets.}
\label{benchmark-table-f1-rff-out}
\begin{center}
\begin{small}
\begin{sc}
\begin{tabular}{lcccc}
\toprule
\multirow{2}{*}{Methods} & \multicolumn{4}{c}{RFF \textbf{OUT}}\\
 & $d=10$ & $d=20$ & $d=50$ &  Ave.  \\
\midrule
PC  & 0.38$\pm$ 0.16& 0.48$\pm$ 0.069& 0.32$\pm$ 0.041& 0.40$\pm$ 0.12\\
GES  & 0.41$\pm$ 0.06& 0.37$\pm$ 0.04& 0.32$\pm$ 0.05& 0.37$\pm$ 0.06 \\
GOLEM  & 0.36$\pm$ 0.14& 0.34$\pm$ 0.12& 0.24$\pm$ 0.084& 0.31$\pm$ 0.13 \\
DECI  & 0.68$\pm$ 0.16& 0.70$\pm$ 0.14& \textbf{0.84$\pm$ 0.039}& 0.74$\pm$ 0.14 \\
GraN-DAG  & 0.34$\pm$ 0.32& 0.55$\pm$ 0.21& 0.62$\pm$ 0.080& 0.50$\pm$ 0.26 \\
DAG-GNN  & 0.50$\pm$ 0.12& 0.50$\pm$ 0.08& 0.34$\pm$ 0.17& 0.44$\pm$ 0.15 \\
DP-DAG  & 0.19$\pm$ 0.087& 0.17$\pm$ 0.054& 0.12$\pm$ 0.039& 0.16$\pm$ 0.067 \\
AVICI  & 0.61$\pm$ 0.21& \textbf{0.88$\pm$ 0.062}& 0.72$\pm$ 0.068& 0.74$\pm$ 0.17 \\
FiP  &\textbf{0.80$\pm$ 0.22}& 0.88$\pm$ 0.085& 0.78$\pm$ 0.054& \textbf{0.81$\pm$ 0.15} \\
\midrule
FiP w. $P$  &0.97$\pm$ 0.038& 0.94$\pm$ 0.044& 0.86$\pm$ 0.058& 0.92$\pm$ 0.068 \\
FiP w. $\mathcal{G}$ &0.97$\pm$ 0.037& 0.98$\pm$ 0.014& 0.98$\pm$ 0.012& 0.98 $\pm$ 0.026 \\
\bottomrule
\end{tabular}
\end{sc}
\end{small}
\end{center}
\vskip -0.1in
\end{table*}

\begin{table*}[!t]
\caption{We compare the directed F1 scores obtained by our model against various baselines on the out-of-distribution test metadataset LIN \textbf{OUT} introduced in section~\ref{sec:eval-amortization}. For each dimension $d\in\{10, 20, 50\}$, we report the mean and the standard deviation of the F1 scores obtained over all the datasets.}
\label{benchmark-table-f1-lin-out}
\begin{center}
\begin{small}
\begin{sc}
\begin{tabular}{lcccc}
\toprule
\multirow{2}{*}{Methods} & \multicolumn{4}{c}{LIN \textbf{OUT}}\\
 & $d=10$ & $d=20$ & $d=50$ &  Ave.\\
\midrule
PC  & 0.53$\pm$ 0.16& 0.53$\pm$ 0.05& 0.36$\pm$ 0.12& 0.47$\pm$ 0.14\\
GES  & 0.51$\pm$ 0.098& 0.56$\pm$ 0.12& 0.60$\pm$ 0.13& 0.56$\pm$ 0.12 \\
GOLEM  & 0.77$\pm$ 0.16& \textbf{0.91$\pm$ 0.054}& 0.49$\pm$ 0.36& 0.73$\pm$ 0.29 \\
DECI  & 0.41$\pm$ 0.15& 0.37$\pm$ 0.074& 0.29$\pm$ 0.12& 0.36$\pm$ 0.13 \\
GraN-DAG  & 0.18$\pm$ 0.13& 0.46$\pm$ 0.13& 0.22$\pm$ 0.17& 0.29$\pm$ 0.19 \\
DAG-GNN  & 0.71$\pm$ 0.16& 0.66$\pm$ 0.069& 0.44$\pm$ 0.19& 0.61$\pm$ 0.19 \\
DP-DAG  & 0.22$\pm$ 0.081& 0.20$\pm$ 0.043& 0.098$\pm$ 0.029& 0.17$\pm$ 0.074 \\
AVICI  & 0.77$\pm$ 0.16& 0.78$\pm$ 0.098& \textbf{0.64$\pm$ 0.16}& 0.73$\pm$ 0.16 \\
FiP &\textbf{0.95$\pm$ 0.076}& 0.77$\pm$ 0.084& 0.56$\pm$ 0.18& \textbf{0.76$\pm$ 0.20}\\
\midrule
FiP w. $P$  &0.96$\pm$ 0.048& 0.78$\pm$ 0.11& 0.58$\pm$ 0.20& 0.77$\pm$ 0.21 \\
FiP w. $\mathcal{G}$ &0.98$\pm$ 0.036& 0.91$\pm$ 0.059& 0.82$\pm$ 0.11& 0.92 $\pm$ 0.090 \\
\bottomrule
\end{tabular}
\end{sc}
\end{small}
\end{center}
\vskip -0.1in
\end{table*}

\begin{table*}[!t]
\caption{We compare the directed F1 scores obtained by our model against various baselines on the in-distribution test metadataset RFF \textbf{IN} introduced in section~\ref{sec:eval-amortization}. For each dimension $d\in\{10, 20, 50\}$, we report the mean and the standard deviation of the F1 scores obtained over all the datasets.}
\label{benchmark-table-f1-rff-in}
\begin{center}
\begin{small}
\begin{sc}
\begin{tabular}{lcccc}
\toprule
\multirow{2}{*}{Methods} & \multicolumn{4}{c}{RFF \textbf{IN}}\\
 & $d=10$ & $d=20$ & $d=50$ &  Ave. \\
\midrule
PC  & 0.48$\pm$ 0.14& 0.41$\pm$ 0.18& 0.30$\pm$ 0.15& 0.39$\pm$ 0.17\\
GES  & 0.38$\pm$ 0.16& 0.48$\pm$ 0.17& 0.47$\pm$ 0.11& 0.44$\pm$ 0.16 \\
GOLEM  & 0.42$\pm$ 0.10& 0.24$\pm$ 0.087& 0.21$\pm$ 0.066& 0.29$\pm$ 0.12 \\
DECI  & 0.59$\pm$ 0.094& 0.64$\pm$ 0.044& 0.73$\pm$ 0.087   & 0.65$\pm$ 0.097 \\
GraN-DAG  & 0.55$\pm$ 0.19& 0.46$\pm$ 0.12& 0.49$\pm$ 0.14& 0.50$\pm$ 0.16 \\
DAG-GNN  & 0.48$\pm$ 0.16& 0.53$\pm$ 0.12& 0.35$\pm$ 0.14& 0.45$\pm$ 0.16 \\
DP-DAG  & 0.30$\pm$ 0.14& 0.20$\pm$ 0.087& 0.10$\pm$ 0.038& 0.20$\pm$ 0.13 \\
AVICI  & \textbf{0.92$\pm$ 0.097}& \textbf{0.89$\pm$ 0.079}& \textbf{0.80$\pm$ 0.14
}& \textbf{0.87$\pm$ 0.12} \\
FiP &0.88$\pm$ 0.10& 0.83$\pm$ 0.089& 0.75$\pm$ 0.10& 0.82$\pm$ 0.11\\
\midrule
FiP w. $P$  &0.94$\pm$ 0.060& 0.87$\pm$ 0.072& 0.86$\pm$ 0.10& 0.90$\pm$ 0.084 \\
FiP w. $\mathcal{G}$ &0.97$\pm$ 0.042& 0.96$\pm$ 0.047& 0.91$\pm$ 0.11& 0.94 $\pm$ 0.077 \\
\bottomrule
\end{tabular}
\end{sc}
\end{small}
\end{center}
\vskip -0.1in
\end{table*}

\begin{table*}[!t]
\caption{We compare the directed F1 scores obtained by our model against various baselines on the in-distribution test metadataset LIN \textbf{IN} introduced in section~\ref{sec:eval-amortization}. For each dimension $d\in\{10, 20, 50\}$, we report the mean and the standard deviation of the F1 scores obtained over all the datasets.}
\label{benchmark-table-f1-lin-in}
\begin{center}
\begin{small}
\begin{sc}
\begin{tabular}{lcccc}
\toprule
\multirow{2}{*}{Methods} & \multicolumn{4}{c}{LIN \textbf{IN}}\\
 & $d=10$ & $d=20$ & $d=50$ &  Ave. \\
\midrule
PC  & 0.56$\pm$ 0.23& 0.48$\pm$ 0.22& 0.42$\pm$ 0.21& 0.48$\pm$ 0.22\\
GES  & 0.54$\pm$ 0.23& 0.57$\pm$ 0.18& 0.51$\pm$ 0.25& 0.54$\pm$ 0.22 \\
GOLEM  & \textbf{0.88$\pm$ 0.049}& 0.86$\pm$ 0.15& \textbf{0.83$\pm$ 0.14}& \textbf{0.86$\pm$ 0.12} \\
DECI  & 0.28$\pm$ 0.18& 0.32$\pm$ 0.16& 0.37$\pm$ 0.20& 0.32$\pm$ 0.18 \\
GraN-DAG  & 0.43$\pm$ 0.23& 0.37$\pm$ 0.19& 0.30$\pm$ 0.16& 0.37$\pm$ 0.21 \\
DAG-GNN  & 0.74$\pm$ 0.11& 0.60$\pm$ 0.14& 0.52$\pm$ 0.21& 0.62$\pm$ 0.18 \\
DP-DAG  & 0.34$\pm$ 0.16& 0.19$\pm$ 0.097& 0.095$\pm$ 0.035& 0.21$\pm$ 0.081 \\
AVICI  & 0.84$\pm$ 0.19& \textbf{0.92$\pm$ 0.10}& 0.77$\pm$ 0.17& 0.84$\pm$ 0.16 \\
FiP & 0.87$\pm$ 0.12& 0.77$\pm$ 0.19& 0.71$\pm$ 0.18& 0.78$\pm$ 0.17\\
\midrule
FiP w. $P$  &0.89$\pm$ 0.10& 0.77$\pm$ 0.19& 0.71$\pm$ 0.17& 0.79$\pm$ 0.17 \\
FiP w. $\mathcal{G}$ &0.91$\pm$ 0.096& 0.81$\pm$ 0.14& 0.83$\pm$ 0.12& 0.85 $\pm$ 0.13\\
\bottomrule
\end{tabular}
\end{sc}
\end{small}
\end{center}
\vskip -0.1in
\end{table*}

\begin{table*}[!t]
\caption{We show the directed AUC-ROC scores obtained by our model on the out-of-distribution test metadataset RFF \textbf{OUT} introduced in section~\ref{sec:eval-amortization}. For each dimension $d\in\{10, 20, 50\}$, we report the mean and the standard deviation of the scores obtained over all the datasets. Note that this metric cannot be computed when we are given the true graph as the False Positive Rate (FPR) is always 0.}
\label{eval-auc-roc-fip-rff-out}
\begin{center}
\begin{small}
\begin{sc}
\begin{tabular}{lcccc}
\toprule
\multirow{2}{*}{Methods} & \multicolumn{4}{c}{RFF \textbf{OUT}}\\
 & $d=10$ & $d=20$ & $d=50$ &  Ave.  \\
\midrule
FiP  &0.91$\pm$ 0.052& 0.97$\pm$ 0.040& 0.94$\pm$ 0.022& 0.94$\pm$ 0.045 \\
FiP w. $P$  &0.98$\pm$ 0.0044& 0.99$\pm$ 0.0065& 0.99$\pm$ 0.0031& 0.99$\pm$ 0.0063 \\
\bottomrule
\end{tabular}
\end{sc}
\end{small}
\end{center}
\vskip -0.1in
\end{table*}

\begin{table*}[!t]
\caption{We show the directed AUC-ROC scores obtained by our model on the out-of-distribution test metadataset LIN \textbf{OUT} introduced in section~\ref{sec:eval-amortization}. For each dimension $d\in\{10, 20, 50\}$, we report the mean and the standard deviation of the scores obtained over all the datasets. Note that this metric cannot be computed when we are given the true graph as the False Positive Rate (FPR) is always 0.}
\label{eval-auc-roc-fip-lin-out}
\begin{center}
\begin{small}
\begin{sc}
\begin{tabular}{lcccc}
\toprule
\multirow{2}{*}{Methods} & \multicolumn{4}{c}{LIN \textbf{OUT}}\\
 & $d=10$ & $d=20$ & $d=50$ &  Ave.  \\
\midrule
FiP  &0.97$\pm$ 0.028& 0.90$\pm$ 0.038& 0.73$\pm$ 0.11& 0.87$\pm$ 0.11 \\
FiP w. $P$  &0.98$\pm$ 0.011& 0.92$\pm$ 0.065& 0.82$\pm$ 0.12& 0.91$\pm$ 0.10 \\
\bottomrule
\end{tabular}
\end{sc}
\end{small}
\end{center}
\vskip -0.1in
\end{table*}

\begin{table*}[!t]
\caption{We show the directed AUC-ROC scores obtained by our model on the in-distribution test metadataset RFF \textbf{IN} introduced in section~\ref{sec:eval-amortization}. For each dimension $d\in\{10, 20, 50\}$, we report the mean and the standard deviation of the scores obtained over all the datasets. Note that this metric cannot be computed when we are given the true graph as the False Positive Rate (FPR) is always 0.}
\label{eval-auc-roc-fip-rff-in}
\begin{center}
\begin{small}
\begin{sc}
\begin{tabular}{lcccc}
\toprule
\multirow{2}{*}{Methods} & \multicolumn{4}{c}{RFF \textbf{IN}}\\
 & $d=10$ & $d=20$ & $d=50$ &  Ave.  \\
\midrule
FiP  &0.93$\pm$ 0.053& 0.93$\pm$ 0.055& 0.93$\pm$ 0.052& 0.93$\pm$ 0.053 \\
FiP w. $P$  &0.97$\pm$ 0.017& 0.98$\pm$ 0.030& 0.96$\pm$ 0.052& 0.97$\pm$ 0.033 \\
\bottomrule
\end{tabular}
\end{sc}
\end{small}
\end{center}
\vskip -0.1in
\end{table*}

\begin{table*}[!t]
\caption{We show the directed AUC-ROC scores obtained by our model on the in-distribution test metadataset LIN \textbf{IN} introduced in section~\ref{sec:eval-amortization}. For each dimension $d\in\{10, 20, 50\}$, we report the mean and the standard deviation of the scores obtained over all the datasets.  Note that this metric cannot be computed when we are given the true graph as the False Positive Rate (FPR) is always 0.}
\label{eval-auc-roc-fip-lin-in}
\begin{center}
\begin{small}
\begin{sc}
\begin{tabular}{lcccc}
\toprule
\multirow{2}{*}{Methods} & \multicolumn{4}{c}{LIN \textbf{IN}}\\
 & $d=10$ & $d=20$ & $d=50$ &  Ave.  \\
\midrule
FiP  &0.93$\pm$ 0.062& 0.88$\pm$ 0.13& 0.90$\pm$ 0.094& 0.91$\pm$ 0.098 \\
FiP w. $P$  &0.95$\pm$ 0.039& 0.89$\pm$ 0.12& 0.91$\pm$ 0.097& 0.92$\pm$ 0.096 \\
\bottomrule
\end{tabular}
\end{sc}
\end{small}
\end{center}
\vskip -0.1in
\end{table*}

\begin{table*}[!t]
\caption{We show the directed AUC-PRG scores obtained by our model on the out-of-distribution test metadataset RFF \textbf{OUT} introduced in section~\ref{sec:eval-amortization}. For each dimension $d\in\{10, 20, 50\}$, we report the mean and the standard deviation of the scores obtained over all the datasets.}
\label{eval-auc-pr-fip-rff-out}
\begin{center}
\begin{small}
\begin{sc}
\begin{tabular}{lcccc}
\toprule
\multirow{2}{*}{Methods} & \multicolumn{4}{c}{RFF \textbf{OUT}}\\
 & $d=10$ & $d=20$ & $d=50$ &  Ave.  \\
\midrule
FiP  &0.99$\pm$ 0.0069& 1.0$\pm$ 0.0040& 1.0$\pm$ 0.0022& 0.99$\pm$ 0.0056 \\
FiP w. $P$  &1.0$\pm$ 0.00078& 1.0$\pm$ 0.00034& 1.0 $\pm$ 0.00014& 1.0$\pm$ 0.0012 \\
FiP w. $\mathcal{G}$  &1.0$\pm$ 0.0& 1.0$\pm$ 0.0& 1.0$\pm$ 0.0& 1.0 $\pm$ 0.0 \\
\bottomrule
\end{tabular}
\end{sc}
\end{small}
\end{center}
\vskip -0.1in
\end{table*}

\begin{table*}[!t]
\caption{We show the directed AUC-PRG scores obtained by our model on the out-of-distribution test metadataset LIN \textbf{OUT} introduced in section~\ref{sec:eval-amortization}. For each dimension $d\in\{10, 20, 50\}$, we report the mean and the standard deviation of the scores obtained over all the datasets.}
\label{eval-auc-pr-fip-lin-out}
\begin{center}
\begin{small}
\begin{sc}
\begin{tabular}{lcccc}
\toprule
\multirow{2}{*}{Methods} & \multicolumn{4}{c}{LIN \textbf{OUT}}\\
 & $d=10$ & $d=20$ & $d=50$ &  Ave.  \\
\midrule
FiP  &1.0$\pm$ 0.0026& 0.99$\pm$ 0.013& 0.96$\pm$ 0.034& 0.87$\pm$ 0.11 \\
FiP w. $P$  &1.0$\pm$ 0.00073& 0.99$\pm$ 0.0099& 0.98$\pm$ 0.016& 0.99$\pm$ 0.012 \\
FiP w. $\mathcal{G}$  &1.0$\pm$ 0.0& 1.0$\pm$ 0.0& 1.0$\pm$ 0.0& 1.0 $\pm$ 0.0 \\
\bottomrule
\end{tabular}
\end{sc}
\end{small}
\end{center}
\vskip -0.1in
\end{table*}

\begin{table*}[!t]
\caption{We show the directed AUC-PRG scores obtained by our modelon the in-distribution test metadataset RFF \textbf{IN} introduced in section~\ref{sec:eval-amortization}. For each dimension $d\in\{10, 20, 50\}$, we report the mean and the standard deviation of the scores obtained over all the datasets.}
\label{eval-auc-pr-fip-rff-in}
\begin{center}
\begin{small}
\begin{sc}
\begin{tabular}{lcccc}
\toprule
\multirow{2}{*}{Methods} & \multicolumn{4}{c}{RFF \textbf{IN}}\\
 & $d=10$ & $d=20$ & $d=50$ &  Ave.  \\
\midrule
FiP  &0.98$\pm$ 0.034& 0.99$\pm$ 0.0066& 1.0$\pm$ 0.0038& 0.99$\pm$ 0.023 \\
FiP w. $P$  &1.0$\pm$ 0.0023& 1.0$\pm$ 0.0044& 1.0$\pm$ 0.0024& 1.0$\pm$ 0.0033 \\
FiP w. $\mathcal{G}$  &1.0$\pm$ 0.0& 1.0$\pm$ 0.0& 1.0$\pm$ 0.0& 1.0 $\pm$ 0.0 \\
\bottomrule
\end{tabular}
\end{sc}
\end{small}
\end{center}
\vskip -0.1in
\end{table*}

\begin{table*}[!t]
\caption{We show the directed AUC-PRG scores obtained by our model on the in-distribution test metadataset LIN \textbf{IN} introduced in section~\ref{sec:eval-amortization}. For each dimension $d\in\{10, 20, 50\}$, we report the mean and the standard deviation of the scores obtained over all the datasets.}
\label{eval-auc-pr-fip-lin-in}
\begin{center}
\begin{small}
\begin{sc}
\begin{tabular}{lcccc}
\toprule
\multirow{2}{*}{Methods} & \multicolumn{4}{c}{LIN \textbf{IN}}\\
 & $d=10$ & $d=20$ & $d=50$ &  Ave.  \\
\midrule
FiP  &0.99$\pm$ 0.012& 0.98$\pm$ 0.031& 0.99$\pm$ 0.0083& 0.99$\pm$ 0.020 \\
FiP w. $P$  &0.99$\pm$ 0.0046& 0.98$\pm$ 0.029& 0.99$\pm$ 0.0082& 0.99$\pm$ 0.019 \\
FiP w. $\mathcal{G}$  &1.0$\pm$ 0.0& 1.0$\pm$ 0.0& 1.0$\pm$ 0.0& 1.0 $\pm$ 0.0 \\
\bottomrule
\end{tabular}
\end{sc}
\end{small}
\end{center}
\vskip -0.1in
\end{table*}

\subsection{Counterfactual Predictions}

\textbf{More details on the Counterfactual Prediction Experiments.} In tables~\ref{benchmark-table-cf-rff-out}, \ref{benchmark-table-cf-lin-out}, we show more detailed versions of the results presented in table~\ref{table:cf-pred-benchmark}, where we report the median, the mean and the standard deviation of the re-scaled $\ell_2$ errors to the ground truth obtained by our method and other baselines on the out-of-distributions metadatasets RFF \textbf{OUT} and LIN \textbf{OUT} respectively for each dimension $d\in\{10, 20, 50\}$ of the problems considered. Additionally, in tables~\ref{benchmark-table-cf-rff-in}, \ref{benchmark-table-cf-lin-in}, we also compare the re-scaled $\ell_2$ errors to the ground truth obtained by our method with other baselines on the in-distribution metadatasets RFF \textbf{IN} and LIN \textbf{IN}. 

\textbf{Additional Benchmark for Counterfactual Predictions.} Here, we reproduce the experiment presented in Table 2 of~\cite{NEURIPS2023_b8402301} to compare FiP with Causal NF~\cite{NEURIPS2023_b8402301}, CAREFL~\cite{khemakhem2021causal} and VACA~\cite{sanchez2022vaca}. The datasets considered are: i) TRIANGLE~\cite{sanchez2022vaca}, a 3-node SCM with a dense causal graph, ii) LARGE BD~\cite{geffner2022deep}, a 9-node SCM with non-Gaussian noise and made out of two chains with common initial and final nodes, and
iii) SIMPSON~\cite{geffner2022deep}, a 4-node SCM simulating a Simpson’s paradox~\cite{simpson1951interpretation}, where the relation between
two variables changes if the SCM is not properly approximated. Note that the two latter datasets are also used in the C-Suite dataset~\cite{geffner2022deep}. For each dataset, FiP, Causal NF and CAREFL are learned with the knowledge of the true topological ordering $P$, while VACA is learned knowing the true graph $\mathcal{G}$. In table~\ref{benchmark-table-cf-causal-nf}, we compare the $\ell_2$ distance to the ground truth counterfactual samples, defined as $\ell_2(x,\hat{x}):=\sqrt{\sum_{i=1}^d \left(x_i - \hat{x}_i\right)^2}$. The only difference with the re-scaled $\ell_2$ distance considered in tables~\ref{table:cf-pred-self}, \ref{table:cf-pred-benchmark} is that here we do not re-scale the distance by the standard deviations of the observable variables $X_i$. However for these particular datasets, the variances of the random variables $X_i$'s are very close (or exactly equal) to 1, and therefore this metric is very similar to the one considered in the main paper.  We observe that FiP and Causal NF manages to obtain similar performances when it comes to predict counterfactual samples. Note also that for these datasets, both FiP and Causal NF manages to recover with high accuracy the counterfactual samples.

\begin{table*}[!t]
\caption{We compare 
the counterfactual predictions obtained by our model against other baselines on the O.O.D metadataset RFF \textbf{OUT}. We measure the re-scaled $\ell_2$ distance between the predicted counterfactual samples and the ground truth ones. For each dimension $d\in\{10, 20, 50\}$, the results are presented in the form $x/y~(z)$ where $x$ is the median, $y$ the mean and $z$ the standard deviation (std) w.r.t the number of datasets of the averaged errors.}
\label{benchmark-table-cf-rff-out}
\begin{center}
\begin{small}
\begin{sc}
\begin{tabular}{lcccc}
\toprule
\multirow{2}{*}{Methods} & \multicolumn{4}{c}{RFF \textbf{OUT}}\\
 & $d=10$ & $d=20$ & $d=50$ &  Ave.  \\
\midrule
DECI & 0.10 / 0.17  (0.17)& 0.22 / 0.20  (0.093)& 0.16 / 0.15 (0.085)& 0.16 / 0.18 (0.12) \\
DoWhy - AVICI & 0.066 / 0.094  (0.068)& 0.16 / 0.16  (0.094)& \textbf{0.20 / 0.20 (0.095)}& 0.16 / 0.16 (0.096) \\
FiP  &\textbf{0.032 / 0.044  (0.032)}& \textbf{0.16 / 0.14  (0.073)}& 0.21 / 0.21 (0.082)& \textbf{0.11 / 0.13 (0.096)} \\
\midrule
FiP w. $\mathcal{G}$ & 0.0078 / 0.0087  (0.0070)& 0.041 / 0.056  (0.040)& 0.080 / 0.065 (0.037)& 0.033 / 0.042 (0.040) \\
DoWhy w. $\mathcal{G}$ & 0.018 / 0.028  (0.030)& 0.10 / 0.11  (0.054)& 0.14 / 0.13 (0.075)& 0.08 / 0.088 (0.072) \\
\bottomrule
\end{tabular}
\end{sc}
\end{small}
\end{center}
\vskip -0.1in
\end{table*}

\begin{table*}[!t]
\caption{We compare 
the counterfactual predictions obtained by our model against other baselines on the O.O.D metadataset LIN \textbf{OUT}. We measure the re-scaled $\ell_2$ distance between the predicted counterfactual samples and the ground truth ones. For each dimension $d\in\{10, 20, 50\}$, the results are presented in the form $x/y~(z)$ where $x$ is the median, $y$ the mean and $z$ the standard deviation (std) w.r.t the number of datasets of the averaged errors.}
\label{benchmark-table-cf-lin-out}
\begin{center}
\begin{small}
\begin{sc}
\begin{tabular}{lcccc}
\toprule
\multirow{2}{*}{Methods} & \multicolumn{4}{c}{LIN \textbf{OUT}}\\
 & $d=10$ & $d=20$ & $d=50$ &  Ave.\\
\midrule
DECI & 0.39 / 0.42  (0.42)& 0.44 / 0.45  (0.22)& 0.26 / 0.27 (0.12)& 0.34 / 0.38 (0.29) \\
DoWhy - AVICI & 0.12 / 0.19  (0.22)& 0.13 / 0.20 (0.17)&\textbf{ 0.15 / 0.21 (0.14)}& 0.13 / 0.20 (0.18) \\
FiP  &\textbf{0.030 / 0.048  (0.045)}& \textbf{0.15 / 0.14  (0.065)}& 0.23 / 0.22 (0.081)& \textbf{0.12 / 0.13 (0.10)} \\
\midrule
FiP w. $\mathcal{G}$ & 0.0084 / 0.025  (0.041)& 0.021 / 0.043  (0.060)& 0.030 / 0.034 (0.023)& 0.018 / 0.034 (0.048) \\
DoWhy w. $\mathcal{G}$ & 26 / 29  (23) $\times$ 1e-4 & 11 / 14  (7.5)  $\times$ 1e-4 & 7.1 / 7.2 (2.7) $\times$ 1e-4& 0.0010  / 0.0017 (0.0017) \\
\bottomrule
\end{tabular}
\end{sc}
\end{small}
\end{center}
\vskip -0.1in
\end{table*}

\begin{table*}[!t]
\caption{We compare 
the counterfactual predictions obtained by our model against other baselines on the I.D metadataset RFF \textbf{IN}. We measure the re-scaled $\ell_2$ distance between the predicted counterfactual samples and the ground truth ones. For each dimension $d\in\{10, 20, 50\}$, the results are presented in the form $x/y~(z)$ where $x$ is the median, $y$ the mean and $z$ the standard deviation (std) w.r.t the number of datasets of the averaged errors.}
\label{benchmark-table-cf-rff-in}
\begin{center}
\begin{small}
\begin{sc}
\begin{tabular}{lcccc}
\toprule
\multirow{2}{*}{Methods} & \multicolumn{4}{c}{RFF \textbf{IN}}\\
 & $d=10$ & $d=20$ & $d=50$ &  Ave. \\
\midrule
DECI & 0.34 / 0.31  (0.14)& 0.13 / 0.21  (0.14)& 0.12 / 0.13 (0.057)& 0.16 / 0.21 (0.14) \\
DoWhy - AVICI & 0.091 / 0.13  (0.11)& 0.099 / 0.14  (0.14)& \textbf{0.087 / 0.10 (0.069)}& \textbf{0.091 / 0.12 (0.11)} \\
FiP  &\textbf{0.096 / 0.12  (0.072)}& \textbf{0.073 / 0.13  (0.13)}& 0.13 / 0.12 (0.060)& 0.12 / 0.13 (0.093) \\
\midrule
FiP w. $\mathcal{G}$ & 0.038 / 0.051  (0.046)& 0.030 / 0.077  (0.11)& 0.030 / 0.025 (0.051)& 0.033 / 0.059 (0.075) \\
DoWhy w. $\mathcal{G}$ & 0.082 / 0.094  (0.078)& 0.096 / 0.12  (0.13)& 0.068 / 0.073 (0.059)& 0.071 / 0.097 (0.096) \\
\bottomrule
\end{tabular}
\end{sc}
\end{small}
\end{center}
\vskip -0.1in
\end{table*}

\begin{table*}[!t]
\caption{We compare 
the counterfactual predictions obtained by our model against other baselines on the I.D metadataset LIN \textbf{IN}. We measure the re-scaled $\ell_2$ distance between the predicted counterfactual samples and the ground truth ones. For each dimension $d\in\{10, 20, 50\}$, the results are presented in the form $x/y~(z)$ where $x$ is the median, $y$ the mean and $z$ the standard deviation (std) w.r.t the number of datasets of the averaged errors.}
\label{benchmark-table-cf-lin-in}
\begin{center}
\begin{small}
\begin{sc}
\begin{tabular}{lcccc}
\toprule
\multirow{2}{*}{Methods} & \multicolumn{4}{c}{LIN \textbf{IN}}\\
 & $d=10$ & $d=20$ & $d=50$ &  Ave. \\
\midrule
DECI & 0.66 / 0.61  (0.21)& 0.27 / 0.34  (0.18)& 0.22 / 0.26 (0.15)& 0.36 / 0.41 (0.24) \\
DoWhy - AVICI & 0.033 / 0.23  (0.27)& \textbf{0.024 / 0.043 (0.058)}& 0.053 / 0.11 (0.11)& 0.035/ 0.13 (0.19) \\
FiP  &\textbf{0.024 / 0.082  (0.078)}&0.076 / 0.081  (0.068)& \textbf{0.036 / 0.090 (0.010)}& \textbf{0.037 / 0.083 (0.084)} \\
\midrule
FiP w. $\mathcal{G}$ & 0.012 / 0.049  (0.066)& 0.012 / 0.024  (0.029)& 0.0075 / 0.039 (0.070)& 0.012 / 0.039 (0.060) \\
DoWhy w. $\mathcal{G}$ & 17 / 24  (28) $\times$ 1e-4& 8.2 / 9.5  (5.4) $\times$ 1e-4 & 5.8 / 7.7 (5.6) $\times$ 1e-4 & 8.0 / 14 (18) $\times$ 1e-4 \\
\bottomrule
\end{tabular}
\end{sc}
\end{small}
\end{center}
\vskip -0.1in
\end{table*}

\begin{table*}[!t]
\caption{We reproduce the exact same experiment as the one proposed in Table 2 of~\cite{NEURIPS2023_b8402301} to compare FiP with other baselines, namely  Causal NF~\cite{NEURIPS2023_b8402301}, CAREFL~\cite{khemakhem2021causal} and VACA~\cite{sanchez2022vaca} on counterfactual predictions. We measure the $\ell_2$ distance between the predicted counterfactual samples and the ground truth ones. We average the results obtained over 5 runs and they are presented in the form $y~(z)$ where $y$ is the mean and $z$ the standard deviation (std) w.r.t the number of runs of the averaged errors. For other methods, we report the same numbers as the one obtainen in Table 2 of~\cite{NEURIPS2023_b8402301}.}
\label{benchmark-table-cf-causal-nf}
\begin{center}
\begin{small}
\begin{sc}
\begin{tabular}{llc}
\toprule
Dataset & Model & CF Error (RMSE)\\
\midrule
 \multirow{4}{*}{TRIANGLE} &  Causal NF & 0.13 (0.02) \\
 &  CAREFL & 0.17 (0.03) \\
& VACA &  4.19 (0.04) \\
& FiP & $\bm{0.094 (0.021)}$ \\
\midrule
\multirow{4}{*}{LARGE BD} & Causal NF & $\bm{0.01 (0.00)}$ \\
 &  CAREFL & 0.08 (0.01)  \\
& VACA &  0.82 (0.02) \\
& FiP &  0.024 (0.0019) \\
\midrule
\multirow{4}{*}{SIMPSON} & Causal NF &  $\bm{0.12 (0.02)}$ \\ 
 &  CAREFL & 0.17 (0.04)  \\
& VACA & 1.50 (0.04) \\
& FiP &  $\bm{0.12 (0.0089)}$ \\
\bottomrule
\end{tabular}
\end{sc}
\end{small}
\end{center}
\vskip -0.1in
\end{table*}

\subsection{On the Noise Reconstruction} 

In this experiment, we evaluate the noise prediction of FiP on the four test metadatasets presented in section~\ref{sec:eval-amortization}, that are  RFF \textbf{OUT}, LIN \textbf{OUT},  RFF \textbf{IN}, and LIN \textbf{IN}. More formally, we evaluate the error between the predicted noise $\tilde{\bm{N}}:=\bm{X} - \mathcal{T}_{\text{ANM}}(P\bm{X},0_d)$ and the true noise $\bm{N}$ generating $\bm{X}$. As we learn $\mathcal{T}_{\text{ANM}}$ on the standardized data, we report the error on this space, that is we report the re-scaled $\ell_2$ distance to the ground truth defined as $r-\ell_2(x,\hat{x}):=\sqrt{\frac{1}{d}\sum_{i=1}^d \left(\frac{n_i - \hat{n}_i}{\sigma_i}\right)^2}$ where $\sigma_i$ are the standard deviations of the observable data $X_i$. Note that we divide the metric by $\sqrt{d}$ as we compare the results across various dimension choices. In table~\ref{table:cf-pred-self}, we present the results obtained, and we show that our method is able to recover more than $90\%$ (in the worst case) of the noise signal in the re-scaled space when trained with the predicted topological ordering obtained by $\mathcal{M}$, while it can reach $96\%$ at worst when trained with the true causal graph.

\begin{table}[!h]
\caption{We report the noise prediction errors of $\mathcal{T}_{\text{ANM}}$ when trained with the predicted topological ordering from $\mathcal{M}$, the true TO, or the true graph on the four test metadatasets introduced in section~\ref{sec:eval-amortization}. We measure the re-scaled $\ell_2$ distance between the predicted noise samples and the ground truth ones. The results presented are of the form $x / y~(z)$ where $x$ is the median, $y$ is the mean and $z$ is the standard deviation (std) w.r.t the number of datasets of the averaged errors.}
\label{table:cf-pred-self}
\begin{center}
\begin{small}
\begin{sc}
\begin{tabular}{lccc}
\toprule
Datasets & FiP & FiP w. True $P$ & FiP w. True $\mathcal{G}$  \\
\midrule
LIN \textbf{IN} &  0.030 / 0.032 (0.018) &  0.029 / 0.030 (0.015) & 0.013 / 0.013 (0.0027)\\
LIN \textbf{OUT} &   0.034 / 0.089 (0.11) &  0.033 / 0.034 (0.0090) & 0.016 / 0.015 (0.0030)\\
RFF \textbf{IN} &  0.083 / 0.11 (0.074) &  0.045 / 0.063 (0.047) & 0.023 / 0.036 (0.034)\\
RFF \textbf{OUT} & 0.13/ 0.14 (0.078)
&  0.072/ 0.094 (0.062) & 0.027 / 0.037 (0.020)\\
\bottomrule
\end{tabular}
\end{sc}
\end{small}
\end{center}
\vskip -0.1in
\end{table}

\subsection{Analysis of the Complexity}

\paragraph{On the number of parameters.} Given observations living in a $d$-dimensional space, an embedding dimension of $D$, a number of heads in the attention $n_h$, a dimension per head $d_{\text{head}}$, a hidden dimension $d_{\text{hidden}}$ for the MLP, and a number of layer $L$, the number of parameters of our transformer architecture is  of order 
$$\mathcal{O}(d D + L (D^2 +  D (n_h * d_{\text{head}}) + D d_h))$$ 
In practice we use $L=2$, and $D=d_{\text{hidden}}=128$, $n_h=8$, and $d_{\text{head}}=32$ for all experiments, giving an architecture with an order of 100k learnable parameters (when $d=50$), and therefore requires an order of 1MB of memory using a float32 precision (which corresponds to 4 bytes per parameter).

\textbf{On the forward pass FLOPs.} The number of floating-point operations involved in the forward pass is the sum of the following complexities: 

\begin{itemize}
    \item Causal Embeddings: $2 d \times D$
    \item Causal Attention: \begin{itemize}
        \item Key, query, and value projection: $3 d D^2$
        \item Key @ Query:  $n_h d_{\text{head}}d^2$
        \item Causal normalization of attention: $n_h d^2$
        \item Final Projection:  $d n_h d_{\text{head}} D$
    \end{itemize}
 \item Function $h$ in the Causal Encoder layer: $d \times (D d_{\text{hidden}} + D d_{\text{hidden}})$
 \item Causal Decoder: $d D$
\end{itemize}

Overall, the total forward pass FLOPs is: 
$$3 dD + L \times (3 dD^2 + n_h (d_{\text{head}}+1)d^2+n_h d_{\text{head}}dD+ 2 Dd_{\text{hidden}})$$
When replacing the variables with the actual values used in the experiments, we obtain a total forward pass FLOPs of around $1e7$ when $d=50$, and therefore the forward pass can be computed on any CPU/GPU with more than a 10 MFLOPs, which is most likely the case as standard CPUs and GPUs can now perform an order of hundreds of GFLOPs per second. The backward pass has (around) twice the FLOPs of the forward pass, which is still of an order of 10 MFLOPs.

\textbf{Experimental Evaluation of the Complexities.} In this experiment, we present the time and memory complexities obtained in practice for FiP. More precisely, we consider the architecture used in all the experiments, that is the transformer-based model $\mathcal{T}_{\text{ANM}}$ with $L=2$ layers, $d_{\text{head}}=32$, $8$ heads, a latent dimension of $D=128$, and a the hidden dimension used in the MLP of $d_{\text{hidden}}:=128$, and we measure its time and memory complexities when varying the dimension $d$ of the problem (i.e. the number of nodes) and the number of samples $n$. More precisely we vary $d\in\{10,50,100\}$ and $n\in\{1k,10k\}$. We also extrapolate the results obtained for the case when $n=1$ using the linear relationships between the number of samples and the time and memory complexities in order to obtain the complexities of FiP per sample. Indeed, we cannot simply evaluate accurately the complexities of the model on a sample basis because our architecture is too small, and the overhead caused by pytorch lightning~\cite{falcon2019pytorch} is greater than the time and memory that needs our model to process one sample. In table~\ref{table:num-param-fip}, we report the total number of parameters of FiP when varying $d$, and in tables~\ref{table:time-fip},~\ref{table:mem-fip}, we report the computational and memory complexities respectively. We observe that our model  requires less than 2 MiB of memory and less than 50 MFLOPs per forward call to be trained and and tested.

\begin{table}[!h]
\caption{We report the number of parameters of FiP when varying $d\in\{10,50,100\}$.} 
\label{table:num-param-fip}
\begin{center}
\begin{small}
\begin{sc}
\begin{tabular}{lccc}
\toprule
$d$ & $d=10$ & $d=50$ & $d=100$   \\
\midrule
Num. Parameters & 340k & 375k & 420k\\
\bottomrule
\end{tabular}
\end{sc}
\end{small}
\end{center}
\vskip -0.1in
\end{table}

\begin{table}[!h]
\caption{We report the computational time per forward call of FiP during both training and testing phases, when varying both the dimension $d$ of the problem (i.e. the number of nodes) and the number of samples $n$. We measure the time complexity of FiP using the same architecture as the one considered for all the experiments, that is the transformer-based model $\mathcal{T}_{\text{ANM}}$ with $L=2$ layers, $d_{\text{head}}=32$, $8$ heads, a latent dimension of $D=128$, and a the hidden dimension used in the MLP of $d_{\text{hidden}}:=128$. We report the mean time (over the batches) in second when varying $d\in\{10,50,100\}$ and $n\in\{1k,10k\}$. We also report the time that FiP would obtain if $n=1$ when varying $d$ by exploiting the linear relationship between $n$ and the computational time. Note that the case where $n=1$ cannot be estimated directly because our architecture is too small, and the overhead caused by pytorch lightning~\cite{falcon2019pytorch} is greater than the time that needs our model to process one sample. Finally, we also show (a rough estimate of) the number of FLOPs needed per forward call, by multiplying the time obtained by the number of FLOPs per second that the machine can perform. Here all the runs have been realized on a A100 with 32-bit float precision, which can process up to 19.5 TFLOPs per second. The results are presented in the form $x / y$ where $x$ represents the time in second and $y$ the number of GFLOPs.} 
\label{table:time-fip}
\begin{center}
\begin{small}
\begin{sc}
\begin{tabular}{c| c c| c c|c c}
\toprule
$n/d$ & \multicolumn{2}{|c}{$n=1k$} & \multicolumn{2}{|c}{$n=10k$} & \multicolumn{2}{|c}{$n=1$}  \\
  & Train & Test & Train & Test & Train & Test \\
\midrule
$d=10$  & 0.0094 / 184 & 0.0037 / 72 & 0.026 / 507  & 0.013 / 254 & $1.8 \times 1e-6$ / 0.035   &  $1.03 \times 1e-6$ / 0.020 \\
$d=50$ & 0.012 / 234 & 0.0047 / 92 &  0.027 / 526 & 0.016 / 312 & $1.6 \times 1e-6$ / 0.031   &  $1.25 \times 1e-6$ / 0.024 \\
$d=100$ &  0.012 / 234 & 0.0048 / 94 & 0.035 / 683 & 0.025 / 488 & $2.6 \times 1e-6$ / 0.051   &  $2.24 \times 1e-6$ / 0.044 \\
\bottomrule
\end{tabular}
\end{sc}
\end{small}
\end{center}
\vskip -0.1in
\end{table}

\begin{table}[!h]
\caption{We report the memory usage of FiP at training time, when varying both the dimension $d$ of the problem (i.e. the number of nodes) and the number of samples $n$. We measure the memory complexity of FiP using the same architecture as the one considered for all the experiments, that is the transformer-based model $\mathcal{T}_{\text{ANM}}$ with $L=2$ layers, $d_{\text{head}}=32$, $8$ heads, a latent dimension of $D=128$, and a the hidden dimension used in the MLP of $d_{\text{hidden}}:=128$. We report the memory used in GiB when varying $d\in\{10,50,100\}$ and $n\in\{1k,10k\}$. We also report the memory usage that FiP would use at training time if $n=1$ when varying $d$ by exploiting the linear relationship between $n$ and the memory usage. Indeed, the case where $n=1$ cannot be estimated directly because our architecture is too small, and the overhead caused by pytorch lightning~\cite{falcon2019pytorch} is much greater than the memory that needs our model to process one sample.} 
\label{table:mem-fip}
\begin{center}
\begin{small}
\begin{sc}
\begin{tabular}{c| c| c|c}
\toprule
$n/d$ & $n=1k$ & $n=10k$ & $n=1$  \\
\midrule
$d=10$  & 1.4 & 3.4 & 0.00022 \\
$d=50$ & 2.6 &  15 & 0.0014\\
$d=100$ &  5.1 & 41 & 0.0040 \\
\bottomrule
\end{tabular}
\end{sc}
\end{small}
\end{center}
\vskip -0.1in
\end{table}

\subsubsection{Scalability of FiP}

In all experiments considered, we only apply FiP to problems with at most $d=50$ nodes when training our fixed-point SCM learner. For such problems, we observe that a 1- or 2-layer(s) network is enough to obtain already high accuracy in recovering the generative SCMs from observations. 

When it comes to studying problems involving thousands of nodes, or even hundreds of thousands of nodes, scaling the model adequately can require (much) higher computational and memory costs, which can represent a significant limitation. However, the scalability of transformers to large sequences with more than 1 million tokens, corresponding to a problem with 1 million nodes in our setting, has been largely studied in the literature\cite{dao2022flashattention, liu2023ring}. Therefore, as our model is an attention-based one, we could exploit similar techniques to scale our transformer-based network for these large-scale problems. However, this extension to large scale problems is out of the scope of this paper, and we leave this question for future work.


\clearpage
\newpage
\section{Appendix: proofs}
\label{sec:proofs}
\subsection{Proof of Proposition \ref{prop:unique-gamma}}

\begin{proof}
Let us define $T:n\to H(\cdot,n)^{\circ d}(0_d)$,  $\mathbb{P}_{\bm{X}}:=(P^T\circ T\circ P)\#\mathbb{P}$ and let us denote $F:(x,n)\to P^T H(Px,Pn)$. Then thanks to the structure of $H$, we obtain that for any $n\in\mathbb{R}^d$, $F(P^T\circ T\circ P (n), n) = P^T\circ T\circ P (n)$  from which follows that $( P^T\circ T\circ P, \text{I}_d)\#\mathbb{P}\in\Pi_{2,\mathbb{P}}$ solves~\eqref{eq:dfp_scm}. In addition, if $\gamma$ solves ~\eqref{eq:dfp_scm}, then for $(\bm{X},\bm{N})\sim\gamma$, we obtain that $$\bm{X} = P^T T(P\bm{N})\; ,$$ from which follows that $\gamma=( P^T\circ T\circ P, \text{I}_d)\#\mathbb{P}\in\Pi_{2,\mathbb{P}}$ which conclude the proof.
\end{proof}

\subsection{Proof of Proposition~\ref{prop:equivalence}}

\begin{proof}
Let $\mathcal{S}(F,\mathbb{P}_{\bm{N}})$ a standard SCM and $P_{\pi}$ a topological ordering associated. Let now $H_1$ and $H_2$ satisfying~\eqref{prop:equivalence}. Then we obtain for all $i$ and $x,n$ that 
\begin{align*}
[P_\pi^T H_1(P_\pi x,P_\pi n)]_i = [P_\pi^T H_2(P_\pi x,P_\pi n)]_i\; .
\end{align*}
from which follows directly that $H_1 = H_2$. To show existence, let us first define $\tilde{F}:=[\tilde{F}_1,\dots,\tilde{F}_d]$ the extended version of $F$ such that for all $i$, $\tilde{F}_i:\mathbb{R}^d\times \mathbb{R}^d\to\mathbb{R}$ satisfies for all $x,n\in\mathbb{R}^d$ $\tilde{F}_i(x,n)=F_i(\textbf{PA}(x_i),n_i)$. Then we can simply define $H:=[H_1,\dots,H_d]$ as $H_i(x,n)=P\tilde{F}_i(P^Tx,P^Tn)$. Reciprocally, let $\mathcal{S}_{\text{fp}}(P,\mathbb{P},H)$ a fixed-point SCM and let us denote $\mathcal{G}$ its graph associated as defined in~\ref{def:causal-graph}. Let now $\mathcal{S}(F^{(1)},\mathbb{P})$ and $\mathcal{S}(F^{(2)},\mathbb{P})$ two standard SCM with DAG associated $\mathcal{G}_1$ and $\mathcal{G}_2$ respectively such that they satisfy~\eqref{eq:unique-equi} and $P$ is a topological order of both. Then, we obtain that for all $i$ and $x, n$
\begin{align*}
F_i^{(1)}(\textbf{PA}_1(x_i),n_i) = F_i^{(2)}(\textbf{PA}_2(x_i),n_i) \; .
\end{align*}
Now using the minimality assumption~\ref{ass-minimality}, we deduce that the set of parents are the necessarily the same, and from which follows that $F_i^{(1)}=F_i^{(2)}$. The existence follows the construction obtained in section~\ref{sec:random-var-fp}.
\end{proof}

\subsection{Proof of Proposition~\ref{prop:anm}}
Here we prove a more general result than the one presented in the main text, where we only need to assume additionally that the endegenous and exogenous distribution are squared integrable and the exogenous distribution admits a variance of $1_d$. The proof of the result presented in Proposition~\ref{prop:anm} can be directly deduced from the proof below by assuming that $g=1$ and without the need to assume that both the endogenous and exogenous distributions are squared integrable and that the variance of the exogenous distribution is $1_d$. This is because the result below requires "extra" assumptions, that we made the choice to present a simpler version in the main text.

\begin{proposition}
\label{prop:anm_gene}
Let $P\in\Sigma_d$ and $\mathbb{P}_{\bm{X}}\in\mathcal{P}_2(\mathbb{R}^d)$. Let us also denote $\mathcal{F}_d^{ANM}:=\{H\in\mathcal{F}_d\colon H(x,n)=h(x)+ g(x) \odot n,~g> 0\}$  and $\mathcal{A}_P^{\text{ANM}}(\mathbb{P}_{\bm{X}}):=\{(P,\mathbb{P},H)\in \mathcal{A}_P(\mathbb{P}_{\bm{X}})\colon \mathbb{P}\in\mathcal{P}_2(\mathbb{R})^{\otimes d},~ H\in\mathcal{F}_d^{ANM},~\mathbb{E}_{\bm{N}\sim \mathbb{P}}(\bm{N})=0_d, \mathbb{E}_{\bm{N}\sim \mathbb{P}}(\bm{N}^2)=1_d\}$. Then $\mathcal{A}_P^{\text{ANM}}(\mathbb{P}_{\bm{X}})$ admits at most 1 element $\mathbb{P}_{P\bm{X}}$ a.s.
\end{proposition}

\begin{proof}
Let $(P,\mathbb{P},H)\in\mathcal{A}_P^{\text{ANM}}(\mathbb{P}_{\bm{X}})$. Then we have that for $(\bm{X},\bm{N})\sim\gamma(P,\mathbb{P},H)$
\begin{align*}
    P\bm{X} &= H(P\bm{X},P\bm{N})\\
        & = h(P\bm{X}) + g(P\bm{X}) \odot P\bm{N}.
\end{align*}
Now let us denote $\bm{Y}=P\bm{X}$. Because $H$ has to satisfies~\eqref{eq-struc} and the $N_i$-s are independent and have $0$ mean, we deduce by taking the conditional expectancy that
\begin{align*}
    h_i(y) = \mathbb{E}(Y_i|(Y_1,\dots,Y_{i-1})=(y_1,\dots,y_{i-1})),~\mathbb{P}_{\bm{Y}}~\text{a.s.}
\end{align*}
and so for all $i$ where $h=[h_1,\dots,h_d]$, therefore the $h_i$'s are $\mathbb{P}_{\bm{Y}}$ almost surely unique. Now by considering the conditional second moment of the residual, we obtain for all $i$ that:
\begin{align*}
    \mathbb{E}((\bm{Y}_i - h_i(\bm{Y}))^2|(Y_1,\dots,Y_{i-1})=(y_1,\dots,y_{i-1})) = g_i(y)^2,~\mathbb{P}_{\bm{Y}}~\text{a.s.}
\end{align*}
as the variances of $N_i$ are 1. Therefore $y\to g_i(y)^2$ are $\mathbb{P}_{\bm{Y}}$ almost surely unique and thanks to the positivity of $g$, we deduce that $y\to g_i(y)$ are $\mathbb{P}_{\bm{Y}}$ a.s. unique, from which follows that
$$P_{P\bm{N}} = \frac{I_d - h}{g}\#\mathbb{P}_{\bm{Y}}$$ is uniquely defined (as $g>0$) and that concludes the proof.
\end{proof}

\subsection{Proof of Theorem~\ref{thm:id-noise}}
Before proving theorem~\ref{thm:id-noise}, let us first show the following important lemma.

\begin{lemma}
\label{lem:identification}
Let $H_1,H_2\in\mathcal{F}_d$ and let $\mathbb{P}_{\bm{N}}\in\mathcal{P}(\mathbb{R})^{\otimes d}$ a jointly independent distribution. Then if 
\begin{align}
\label{eq:equality-h}
(H^{(1)}(\cdot,n))^{\circ d} = (H^{(2)}(\cdot,n))^{\circ d},\quad \mathbb{P}_{\bm{N}}~\text{a.s.}
\end{align}
and by denoting $\mathbb{P}_{\bm{X}} = (H^{(1)}(\cdot,n))^{\circ d}\# \mathbb{P}_{\bm{N}}$, then we have $H_1(x,n)=H_2(x,n)~~\mathbb{P}_{\bm{X}} \otimes\mathbb{P}_{\bm{N}}~\text{a.s.}.$
\end{lemma}

\begin{proof}
Let us assume that
\begin{align}
\label{eq:equality}
(H^{(1)}(\cdot,n))^{\circ d} = (H^{(2)}(\cdot,n))^{\circ d} = h(n),\quad \mathbb{P}_{\bm{N}}~\text{a.s.}
\end{align}
Now, using the structure of $H^{(i)}$ induced by~\eqref{eq-struc}, observe that we have for all $x,n\in\mathbb{R}^d$ and $k\in\{1,\dots,d\}$:
\begin{align*}
    H_k^{(i)}(x,n)=H_k^{(i)}([x_1,\dots,x_{k-1},0,\dots,0],[0,\dots,n_k,\dots,0])
\end{align*}
where $H^{(i)}(x,n)=[H_1^{(i)}(x,n),\dots,H_d^{(i)}(x,n)]$, $x=[x_1,\dots,x_d]$, and $n=[n_1,\dots,n_d]$. In the following we denote for all $k\geq 1$, $x,n\in\mathbb{R}^d$, $(H^{(i)}(\cdot,n))^{\circ k}(x) = [[(H^{(i)}(\cdot,n))^{\circ k}(x)]_1,\dots,[(H^{(i)}(\cdot,n))^{\circ k}(x)]_d]\in\mathbb{R}^d$. Now from~\eqref{eq:equality-h} and using the triangular structure of $H^{(i)}$, observe that for all $k\in\{1,\dots,d\}$ and $x\in\mathbb{R}^d$, we have that $\mathbb{P}_{\bm{N}}$ a.s.
\begin{align*}
&[\tilde{h}_1(n_1),\dots, \tilde{h}_k(n_1,\dots,n_k)]=\\
&[[(H^{(i)}(\cdot,n))^{\circ 1}(x)]_1,\dots,[(H^{(i)}(\cdot,n))^{\circ k}(x)]_k].
\end{align*}
where we denote for all $j\in\{1,\dots,d\}$, 
$\tilde{h}_j(n_1,\dots,n_j):=h_j(n_1,\dots,n_d)$ and $h=[h_1,\dots,h_d]$ which are well defined as $h$ is a triangular map. 
Our goal now is to show that for $k\in\{1,\dots,d\}$, 
$$
H_k^{(1)}(x,n)=H_k^{(2)}(x,n),\quad \mathbb{P}_{\bm{X}}\otimes\mathbb{P}_{\bm{N}}~\text{a.s.},
$$
which will conclude the proof. First Observe that for all $x\in\mathbb{R}^d$
$$
[(H^{(i)}(\cdot,n))^{\circ d}(x)]_1 = H_1^{(i)}(x,n)=h_1(n),\quad \mathbb{P}_{\bm{N}}~\text{a.s.}
$$
Let us now develop the expression $H^{(i)}(\cdot,n))^{\circ d}$. For that purpose, let us denote for $x,n\in\mathbb{R}^d$ and $k\in\{1,\dots, d\}$, $x_k(n):=[(H^{(i)}(\cdot,n))^{\circ 1}(x)]_1,\dots,[(H^{(i)}(\cdot,n))^{\circ k}(x)]_k,0,\dots,0]$, $n_{1,k}:=[n_1,\dots,n_k]\in\mathbb{R}^k$ and $\tilde{h}_{1,k}:=[\tilde{h}_1,\dots,\tilde{h}_k]$. Then we obtain that for all  $x,n\in\mathbb{R}^d$ and $k\in\{2,\dots, d\}$
\begin{align*}
&[(H^{(i)}(\cdot,n))^{\circ d}(x)]_k \\
&= H_k^{(i)}([\tilde{x}_{k-1}(n),0,\dots,0][0,\dots,n_k,\dots,0])
\end{align*}
from which follows that for all $x\in\mathbb{R}^d$ we have $\mathbb{P}_{\bm{N}}$ a.s. that
\begin{align*}
 &[(H^{(i)}(\cdot,n))^{\circ d}(x)]_k \\   
 &= H_k^{(i)}([\tilde{h}_{1,k-1}(n_{1,k-1}),0,\dots,0],[0,\dots,n_k,\dots,0]).
\end{align*}
Therefore we deduce from~\eqref{eq:equality}, that $\mathbb{P}_{\bm{N}}~\text{a.s.}$ we have
\begin{align*}
    &H_k^{(1)}([\tilde{h}_{1,k-1}(n_{1,k-1}),0,\dots,0],[0,\dots,n_k,\dots,0])\\
    &=H_k^{(2)}([\tilde{h}_{1,k-1}(n_{1,k-1}),0,\dots,0],[0,\dots,n_k,\dots,0])
\end{align*}
Now using the jointly independence of $\mathbb{P}_{\bm{N}}$, and by denoting $\mathbb{P}_{\bm{N}_{1,k}}:=\mathbb{P}_{N_{1}}\otimes \dots \mathbb{P}_{N_{k}}$, we obtain that $\mathbb{P}_{\bm{N}_{1,k-1}} \otimes \mathbb{P}_{\bm{N}}$
\begin{align*}
    &H_k^{(1)}([\tilde{h}_{1,k-1}(n_{1,k-1}),0,\dots,0],[0,\dots,n_k,\dots,0])\\
    &=H_k^{(2)}([\tilde{h}_{1,k-1}(n_{1,k-1}),0,\dots,0],[0,\dots,n_k,\dots,0])
\end{align*}
and as $\tilde{h}_{1,k-1}\#\mathbb{P}_{\bm{N}_{1,k-1}}=\mathbb{P}_{P\bm{X}_{1,k-1}}$, we deduce that $\mathbb{P}_{P\bm{X}} \otimes \mathbb{P}_{\bm{N}}$ a.s.
\begin{align*}
    &H_k^{(1)}([x_1,\dots,x_k,0,\dots,0],[0,\dots,n_k,\dots,0])\\
    &=H_k^{(2)}([x_{1},\dots,x_k,0,\dots,0],[0,\dots,n_k,\dots,0])
\end{align*}
from which follows that 
$$
H_k^{(1)}(x,n)=H_k^{(2)}(x,n),\quad \mathbb{P}_{P\bm{X}} \otimes\mathbb{P}_{\bm{N}}~\text{a.s.}.
$$
\end{proof}

We are now ready to prove the theorem below.
\begin{proof}
Let $(P,\mathbb{P}_{\bm{N}},H)\in \mathcal{A}_P^{\text{MON}}(\mathbb{P}_{\bm{N}},\mathbb{P}_{\bm{X}})$. Let us define $h: n\in\mathbb{R}^d\to x(n):=H(\cdot,n)^{\circ d}$
where $x(n)$ is the solution of the equation $x = H(x,n)$. The solution always exists and is unique. Observe now that $h$ is a triangular and monotonic map thanks to the structure imposed on $H$ and satisfies $h\#\mathbb{P}_{P\bm{N}}=\mathbb{P}_{P\bm{X}}$. As both $\mathbb{P}_{P\bm{N}}$, and  $\mathbb{P}_{P\bm{X}}$ are a.c. w.r.t the Lebesgue measure, then $\mathbb{P}_{P\bm{N}}$ a.s. there exists a unique increasing triangular $T$ satisfying $T\#\mathbb{P}_{P\bm{N}} = \mathbb{P}_{P\bm{X}}$~\cite{rosenblatt1952remarks}. Therefore we have that $\mathbb{P}_{P\bm{N}}$ a.s.  $h=T$ and $h$ is unique $\mathbb{P}_{P\bm{N}}$ a.s. Now let $H^{(1)}, H^{(2)}\in\mathcal{F}_d^{MON}$ such that $(P,\mathbb{P}_{\bm{N}},H^{(1)})\in \mathcal{A}_P^{\text{MON}}(\mathbb{P}_{\bm{N}},\mathbb{P}_{P\bm{X}})$ and $(P,\mathbb{P}_{\bm{N}},H^{(2)})\in \mathcal{A}_P^{\text{MON}}(\mathbb{P}_{\bm{N}},\mathbb{P}_{P\bm{X}})$. Because $h$ is unique $\mathbb{P}_{P\bm{N}}$ a.s. we have that
\begin{align}
\label{eq:equality}
(H^{(1)}(\cdot,n))^{\circ d} = (H^{(2)}(\cdot,n))^{\circ d} = h(n),\quad \mathbb{P}_{P\bm{N}}~\text{a.s.}
\end{align}
Then thanks to lemma~\ref{lem:identification}, we deduce directly the result.
\end{proof}

\subsection{On the Existence and Non-Uniqueness of Monotonic Fixed-Point SCMs.}

\begin{proposition}
\label{prop:existence-scm}
Let $P\in\Sigma_d$, and $\mathbb{P}_{\bm{X}}\in\mathcal{P}(\mathbb{R}^d)$. Let us assume $\mathbb{P}_{\bm{X}}$ is continuous. In addition, let us assume that there exists a jointly independent and continuous distribution $\mathbb{Q}\in\mathcal{P}(\mathbb{R})^{\otimes d}$ with continuous density such that $\mathcal{A}_P^{\text{MON}}(\mathbb{Q},\mathbb{P}_{\bm{X}})$ is not empty. Then for any continuous distribution $\mathbb{P}_{\bm{N}}\in\mathcal{P}(\mathbb{R})^{\otimes d}$ with continuous density, $\mathcal{A}_P^{\text{MON}}(\mathbb{P}_{\bm{N}},\mathbb{P}_{\bm{X}})$ is a singleton $\mathbb{P}_{P\bm{X}}\otimes \mathbb{P}_{P\bm{N}}$ a.s. 
In particular, when $\mathbb{P}_{\bm{N}} = \mathcal{N}(\bm{0}_d,\text{I}_d)$, that is the standard (Multivariate) Gaussian distribution.
\end{proposition}

\begin{proof}
    Let $\mathbb{Q}$ such that $\mathcal{A}_P^{\text{MON}}(\mathbb{Q},\mathbb{P}_{P\bm{X}})$ is not empty. As $\mathbb{Q}$ is assumed to be continuous and jointly independent, we obtain from theorem~\ref{thm:id-noise} that  $\mathbb{P}_{P\bm{X}}\otimes \mathbb{Q}$ a.s. there exists a unique $H\in\mathcal{F}_d^{MON}$ satisfying $p_1\#\gamma(P,\mathbb{Q},H) = \mathbb{P}_{\bm{X}}$. Let us denote it $H_{\mathbb{Q}}$. Now because $\mathbb{P}_{P\bm{N}}$ is also continuous, there exists $\mathbb{P}_{P\bm{N}}$ a.s. a unique triangular and increasing map satisfying $h\#\mathbb{P}_{P\bm{N}}=\mathbb{Q}$. In addition, because both $\mathbb{P}_{\bm{N}}$ and $\mathbb{Q}$ are jointly independent, $h$ is in fact a \emph{diagonal} and increasing map. Finally because both densities of $\mathbb{P}_{\bm{N}}$ and $\mathbb{Q}$ are continuous, then $h$ can be chosen differentiable. Now let us define $H^*(x,n):= H_{\mathbb{Q}}(x,h(n))$. Now because $h$ is differentiable and due to its structure, we obtain that $H^*\in\mathcal{F}_d^{MON}$. Observe also that $p_1\#\gamma(P,\mathbb{P}_{\bm{N}}, H^{*})=\mathbb{P}_{\bm{X}}$, therefore $(P,\mathbb{P}_{\bm{N}}, H^{*})\in\mathcal{A}_P^{\text{MON}}(\mathbb{P}_{\bm{N}},\mathbb{P}_{P\bm{X}})$. Then applying again theorem~\ref{thm:id-noise}, we obtain the desired result.
\end{proof}

The above Proposition has two important consequences: it shows (i) that as long as $\mathbb{P}_{\bm{X}}$ has been generated using a "monotonic" fixed-point SCM, then there exists a unique "monotonic" fixed-point SCM with standard Gaussian noise and the same topological ordering that can explain it. And (ii) it shows that if $\mathbb{P}_{\bm{X}}$ has been generated using a "monotonic" fixed-point SCM then, there exists infinite "monotonic" fixed-point SCMs with the same topological ordering that can explain it. Therefore for such a class of SCMs, it is sufficient and necessary to specify the exogenous distribution in order to obtain full recovery given the topological order. We also deduce directly the three following clarifying corollaries from the above results.
\begin{corollary}
Under the assumption of proposition~\ref{prop:existence-scm}, let $H_{\mathbb{Q}}\in\mathcal{F}_d^{\text{MON}}$ such that $\mathcal{S}_{\text{fp}}(P,\mathbb{Q},H_Q)$ generates $\mathbb{P}_{\bm{X}}$. Then for any $\mathbb{P}_{\bm{N}}\in\mathcal{P}(\mathbb{R})^{\otimes d}$ continuous, with continuous density, there exists a unique diagonal, monotonic and differentiable map $h$, $\mathbb{P}_{\bm{N}}$ a.s. such that $\mathcal{S}_{\text{fp}}(P,\mathbb{P}_{\bm{N}},(x,n)\to H_{\mathbb{Q}}(x,h(n)))$ generates $\mathbb{P}_{\bm{X}}$.
\end{corollary}
The above corollary characterizes the form of all monotonic fixed-point SCMs generating the same observational distribution given a reference one.

\begin{corollary}
\label{coro:equivalence}
Under the assumption of proposition~\ref{prop:existence-scm}, if $H_1, H_2\in\mathcal{F}_d^{\text{MON}}$ such that there exists $\mathbb{P}_1$ and  $\mathbb{P}_2$ both in $\mathcal{P}(\mathbb{R})^{\otimes d}$ continuous with continuous density and satisfying $(P,\mathbb{P}_1,H_1)$  and $(P,\mathbb{P}_2,H_2)$ are elements of  $$\mathcal{A}_P^{\text{MON}}(\mathbb{P}_{\bm{X}}):=\bigcup_{\mathbb{P}\in\mathbb{P}(R)^{\otimes d}} \mathcal{A}_P^{\text{MON}}(\mathbb{P},\mathbb{P}_{\bm{X}})$$, then there exists $P\#\mathbb{P}_1$ a.s. a diagonal, differentiable and monotonic map $h:\mathbb{R}^d\to\mathbb{R}^d$ such that
$$ H_1(x,n)=H_2(x,h(n))~~ \mathbb{P}_{\bm{X}}\otimes (P\#\mathbb{P}_1)~\text{a.s.}$$
\end{corollary}

The above corollary shows the functional relationships between two monotonic fixed-point SCMs with the same TO that generate the same observational distribution.

\begin{corollary}
Under the assumption of proposition~\ref{prop:existence-scm}, for any continuous distribution $\mathbb{P}_{\bm{N}}\in\mathcal{P}(\mathbb{R})^{\otimes d}$ with continuous density, $\mathcal{A}_P^{\text{MON}}(\mathbb{P}_{\bm{N}},\mathbb{P}_{\bm{X}})$ is a singleton $\mathbb{P}_{P\bm{X}}\otimes \mathbb{P}_{P\bm{N}}$ a.s., and all these fixed-point SCMs admit the exact same causal graphs.
\end{corollary}

Finally, due to the fact that two generating fixed-point SCMs only differ from each others by a diagonal map on the exogenous variables, the causal graphs are therefore the same.

\subsection{Proof of Theorem~\ref{thm-gene-identification}}

Before showing the result, let us first show three important Lemmas, from which the result will follows.

\begin{lemma}
\label{lemma-bij}
Let $H$ satisfying condition~\ref{cond:counterfactual}, and let us denote $H:=[H_1,\dots,H_d]$. Then for any $i\in\{1,\dots,d\}$, and $x\in\mathbb{R}^d$,
$$n_i\in\mathbb{R}\to H_i(x,[0,\dots,n_i,0,\dots])\in\mathbb{R}$$
is bijective from $\mathbb{R}$ to $\mathbb{R}$, and thus strictly monotonic.
\end{lemma}

\begin{proof}
To show the result, we will show by recursion on $1\leq i\leq d$, that for all $x\in\mathbb{R}^d$, $n_i\in\mathbb{R}\to H_i(x,n_i)\in\mathbb{R}$ is bijective. Before doing so, let us introduce some notations. Using the structure of $H$, we can define for all $i\in\{1,\dots,d\}$, $\tilde{H}_i(x_1,\dots,x_{i-1},n_i):=H_i(x,n)$. Therefore showing the bijectivity of  $n_i\to H_i(x,n)$ for any $x$ is equivalent to show the bijectivity of  $n_i\to \tilde{H}_i(x_1,\dots,x_{i-1},n_i)$ and so for all $[x_1,\dots,x_{i-1}]\in\mathbb{R}^{i-1}$. Let us also define recursively the following sequence, starting with $[H^{\circ 1}(n)]_1=\tilde{H}_1(n_1)$, and 
for all $i\in\{2,\dots,d\}$, 
$$[H^{\circ i}(n)]_i=\tilde{H}_i([H^{\circ 1}(n)]_1,\dots,[H^{\circ (i-1)}(n)]_{i-1},n_i)$$
Now observe that $H^{\circ d}(n)=[[H^{\circ 1}(n)]_1, \dots, [H^{\circ d}(n)]_d]$. We are now ready to show the desired recursion. For $k=1$, first observe that $[H^{\circ d}(n)]=\tilde{H}_1(n_1)$, thus as $H^{\circ d}$ is bijective, it means that $\tilde{H}_1$ must describe $\mathbb{R}$ and therefore it is surjective. Now assume it is not injective. Then there exists $n_1\neq n_1'$ such that $\tilde{H}_1(n_1)=\tilde{H}(n_1')$. However, using the above construction, we deduce that for any $n_2,\dots,n_d\in\mathbb{R}$, we have $H^{\circ d}(n_1,n_2,\dots,n_d)=H^{\circ d}(n_1',\dots,n_d)$ which contredicts the injectivity of $H^{\circ d}$, therefore $n_1\to \tilde{H}_1(n_1)$ is bijective. Now assume the result holds for $i\leq k\leq d-1$ and let us show that the result hold in $k+1$. Using the above construction, we have that
$$[H^{\circ (k+1)}(n)]_{k+1}=\tilde{H}_{k+1}([H^{\circ 1}(n)]_1,\dots,[H^{\circ (k)}(n)]_{k},n_{k+1}).$$

Let now $x=[x_1,\dots,x_{k}]\in\mathbb{R}^k$.
First, using the bijectivity of $n_i\to \tilde{H}(z_1,\dots,z_{i-1},n_i)$ for $[z_1,\dots,z_{i-1}]\in\mathbb{R}^{i-1}$ for $1\leq i\leq k$, we can choose $n_1(x), \dots, n_k(x)\in\mathbb{R}$, such that for all $i\in\{1,\dots,k\}$, $[H^{\circ i}(n)]_i=x_i$. For such $n_1(x),\dots,n_k(x)\in\mathbb{R}$ and $n_{k+1}\in\mathbb{R}$, we obtain that  

$$[H^{\circ (k+1)}(n_1(x),\dots,n_k(x),n_{k+1},0,\dots]_{k+1}=\tilde{H}_{k+1}(x_1,\dots,x_k,n_{k+1}).$$
Now again, the sujertivity of $n_{k+1}\to \tilde{H}_k(x_1,\dots,x_k,n_{k+1})$ follows directly from the surjectivity of $H^{\circ d}$, and the injectivity can be obtained again by contradiction. Finally, because $H$ is differentiable, thus continuous, then for all the $n_i\to H_i(x,n)$ are bijective and continous from which follows the strict monotonicity.

\end{proof}

\begin{lemma}
\label{lem-sign}
Let $H$ satisfying~\ref{cond:counterfactual},  and let us denote $H:=[H_1,\dots,H_d]$. Then for all $i\in\{1,\dots,d\}$,
$$(x,n_i)\in\mathbb{R}^d\times\mathbb{R}\to \frac{\partial H_i}{\partial n_i} (x,[0,\dots,n_i,0,\dots])$$
is of constant sign, that is either for all $(x,n_i)\in\mathbb{R}^d\times\mathbb{R}$, we have $\frac{\partial H_i}{\partial n_i} (x,[0,\dots,n_i,0,\dots])\geq 0$, or $\frac{\partial H_i}{\partial n_i} (x,[0,\dots,n_i,0,\dots])\leq 0$.
\end{lemma}

\begin{proof}
Thanks to Lemma~\ref{lemma-bij}, we know that for any $x$ the functions  $n_i\in\mathbb{R}^d\times\mathbb{R}\to H_i (x,[0,\dots,n_i,0,\dots])$ are strictly monotonic, therefore for any $x$, $n_i\to \frac{\partial H_i}{\partial n_i} (x,[0,\dots,n_i,0,\dots])$ is of constant sign. Now we want to show that the sign is also constant for any $x$. To do so let us assume that there exists $x_1\neq x_2$ such that $n_i\to \frac{\partial H_i}{\partial n_i} (x_1,[0,\dots,n_i,0,\dots])$ is non-negative and $n_i\to \frac{\partial H_i}{\partial n_i} (x_2,[0,\dots,n_i,0,\dots])$ is non-positive. Now because we have assumed that $H$ is $C^1$, we have that $(x,n_i)\in\mathbb{R}^d\times\mathbb{R}\to \frac{\partial H_i}{\partial n_i} (x,[0,\dots,n_i,0,\dots])$ is continuous. Let now $\mathcal{I}_1\subset\mathbb{R}$ a finite interval such that $n_i\to \frac{\partial H_i}{\partial n_i} (x_2,[0,\dots,n_i,0,\dots])$  is strictly negative on $\mathcal{I}_1$ and $\mathcal{I}_2\subset\mathbb{R}$ another finite interval such that $n_i\to \frac{\partial H_i}{\partial n_i} (x_1,[0,\dots,n_i,0,\dots])$ is strictly positive on $\mathcal{I}_2$. If such intervals do not exist, then by continuity of $n_i\to \frac{\partial H_i}{\partial n_i} (x_k,[0,\dots,n_i,0,\dots])$ is $0$ everywhere which contradicts the bijectivity previously obtained of such functions. Now by continuity again (w.r.t $x$ this time), we can define  $h(x):=\int_{\mathcal{I}_1\cup\mathcal{I}_2}  \frac{\partial H_i}{\partial n_i} (x,[0,\dots,n_i,0,\dots]) dn_i$. Observe that $h(x_1)>0$ while $h(x_2)<0$. Then by continuity we obtain that the existence of $x_3$ such that  $h(x_3)=0$ using the fact that the image space of the segment $[x_1,x_3]$ (which is a compact) is necessarily an interval of $\mathbb{R}$ by continuity. However, this implies that $\frac{\partial H_i}{\partial n_i} (x_3,[0,\dots,n_i,0,\dots])=0$ and so for all $n_i\in \mathcal{I}_1\cup\mathcal{I}_2$ (again by continuity), which again contradicts the bijectivity of $n_i\to \partial H_i (x_3,[0,\dots,n_i,0,\dots])$. The results is proved.
\end{proof}

\begin{lemma}
\label{lem-eq-dist}
Let  $\mathcal{S}_{\text{fp}}(P, \mathbb{P}, H)$ a fixed-point SCM such that $H$ satisfies cond.~\ref{cond:counterfactual}. Then for any diagonal, bijective, and differentiable map $h:\mathbb{R}^d\to\mathbb{R}^d$, the fixed-point SCM defined as  $\mathcal{S}_{\text{fp}}(P, h^{-1}\#\mathbb{P}, (x,n)\to H(x,h(n)))$ has the exact same observational, interventional and counterfactual distributions.
\end{lemma}
\begin{proof}
Let $T:\mathbb{R}^d\to\mathbb{R}^d$  a differentiable and lower-triangular map. Let us now show that  $\gamma^{\text{do}(T)}(P, \mathbb{P}_{\bm{N}}, H)$ = $\gamma^{\text{do}(T)}(P, h^{-1}\#\mathbb{P}_{\bm{N}}, H^{(2)}:(x,n)\to H(x,h(n)))$. First it is clear that $\mathcal{S}_{\text{fp}}(P, h^{-1}\#\mathbb{P}_{\bm{N}}, H^{(2)})$ is well defined as $H^{(2)}\in\mathcal{F}_d$ and $h^{-1}\#\mathbb{P}_{\bm{N}}\in\mathcal{P}(\mathbb{R})^{\otimes d}$. Now we have by definition that for all $n\in\mathbb{R}^d$
$(H^{(2)})^{\circ d}(\cdot,n)(0_d)=H^{\circ d}(\cdot,h(n))(0_d)=H^{\circ d}(h(n))$. Therefore if we denote $\mathbb{P}_{\bm{X}}=p_1\#\gamma(P, \mathbb{P}_{\bm{N}}, H)$, we have that
$$\mathbb{P}_{\bm{X}}=H^{\circ d}\#\mathbb{P}_{\bm{N}}=H^{\circ d}\circ h \#(h^{-1}\#\mathbb{P}_{\bm{N}}) = (H^{(2)})^{\circ d} \#(h^{-1}\#\mathbb{P}_{\bm{N}}) $$
from which follows that $\mathbb{P}_{\bm{X}}$ is also the observational distribution of $\mathcal{S}_{\text{fp}}(P, h^{-1}\#\mathbb{P}_{\bm{N}}, H^{(2)})$. Now observe that 
$$(H^{(2)}_T)^{\circ d} \circ ((H^{(2)})^{\circ d})^{-1}(x) =(H^{(2)}_T)^{\circ d}(h^{-1}((H^{\circ d})^{-1}(x))) = (H_T)^{\circ d}\circ (H^{\circ d})^{-1}(x)$$
Therefore $(I_d,H^{(2)}_T)^{\circ d} \circ ((H^{(2)})^{\circ d})^{-1})\#\mathbb{P}_{\bm{X}}=(I_d,H_T^{\circ d} \circ (H^{\circ d})^{-1})\#\mathbb{P}_{\bm{X}}$
from which the result follows.
\end{proof}

Let us now show the following Proposition.
\begin{proposition}
\label{prop:eq-mon}
Let  $\mathcal{S}_{\text{fp}}(P, \mathbb{P}, H)$ a fixed-point SCM with $H$ satisfying cond.~\ref{cond:counterfactual}. There exists a function $h:x\in\mathbb{R}^d\to[\pm x_1,\dots,\pm x_d]$, such that 
$$H^{\text{MON}}:(x,n)\in\mathbb{R}^d\times\mathbb{R}^d\to H(x,h(n))\in\mathbb{R}^d$$
satisfies $H^{\text{MON}}\in\mathcal{F}_d^{\text{MON}}$. In addition $\mathcal{S}_{\text{fp}}(P, \mathbb{P}, H)$ and $\mathcal{S}_{\text{fp}}(P, h^{-1}\#\mathbb{P}, H^{\text{MON}})$ have the same observational, interventional and counterfactual distributions.
\end{proposition}

\begin{proof}
Thanks to Lemma~\ref{lem-sign}, we can define for all $i\in\{1,\dots,d\}$ 
$$s_i:=\text{sign}\left((x,n_i)\to \frac{\partial H_i}{\partial n_i} (x,[0,\dots,n_i,0,\dots])\right)\in\{-1,1\}$$
Then by defining $h(n):=[s_1 n_1,\dots,s_d n_d]$, we can define 
$H^{\text{MON}}(x,n) =H(x,h(n))$. Now observe that $h$ is differentiable, bijective, and diagonal, therefore thanks to Lemma~\ref{lem-eq-dist}, we deduce that $\mathcal{S}_{\text{fp}}(P, \mathbb{P}, H)$ and $\mathcal{S}_{\text{fp}}(P, h^{-1}\#\mathbb{P}, H^{\text{MON}})$ have the same observational, interventional and counterfactual distributions. Finally let us show that $H^{\text{MON}}\in \mathcal{F}_d^{\text{MON}}$. To see that, we simply needs to look at the Jacobian and we obtain that
$$\text{Jac}_2 H^{\text{MON}}(x,n) = \text{Jac}_2 H(x,h(n)) \text{Diag}(s_1,\dots,s_d)$$
However, again using Lemma~\ref{lem-sign}, we have that for all $i$,  $$\text{sign}\left([\text{Jac}_2 H(x,h(n))]_{i,i}\right)=s_i $$
from which follows the result.
\end{proof}

We are now ready to show the Theorem. Recall first that by definition,  $(P, \mathbb{P}, H)\in \mathcal{A}_P^{\text{INV}}(\mathbb{P}_{\bm{X}})$, and therefore it is not empty. Let now  $(P, \mathbb{P}^{*}, H^*)\in\mathcal{A}_P^{\text{INV}}(\mathbb{P}_{\bm{X}})$ any triplet in this set. Using Proposition~\ref{prop:eq-mon},  there exists $h_1$ and $h_2$ both $C^1$ diffeomorphisms such that $\mathcal{S}_{\text{fp}}(P, h_1^{-1}\#\mathbb{P}, (x,n)\to H(x,h_1(n)))$ and $\mathcal{S}_{\text{fp}}(P, \mathbb{P}, (x,n)\to H(x,n))$ 
have the same causal distributions and $(x,n)\to H(x,h_1(n))\in\mathcal{F}_d^{\text{MON}}$, and similarly, $\mathcal{S}_{\text{fp}}(P,\mathbb{P}^*, (x,n)\to H^*(x,n))$ and 
$\mathcal{S}_{\text{fp}}(P, h_2^{-1}\#\mathbb{P}^*, (x,n)\to H^*(x,h_2(n)))$ have the same causal distributions and $(x,n)\to H^*(x,h_2(n))\in\mathcal{F}_d^{\text{MON}}$. Now observe that both $h_1^{-1}\#\mathbb{P}$ and $h_2^{-1}\#\mathbb{P}^*$ are a.c. w.r.t Lebesgue with continuous density because $\mathbb{P}, \mathbb{P}^*\in\mathcal{P}_{cc}(\mathbb{R})^{\otimes d}$ and $h_1,h_2$ are $C^1$-diffeomorphisms. Then we obtain that both SCMs $\mathcal{S}_{\text{fp}}(P, h_1^{-1}\#\mathbb{P}, (x,n)\to H(x,h_1(n)))$ and $\mathcal{S}_{\text{fp}}(P, h_2^{-1}\#\mathbb{P}^*, (x,n)\to H^*(x,h_2(n)))$ are elements of  $\mathcal{A}_P^{\text{MON}}(\mathbb{P}_{\bm{X}}$ (as defined in Corollary~\ref{coro:equivalence}) with continuous exogenous distributions admitting continuous densities, and thanks to Proposition~\ref{prop:existence-scm} (or Corollary~\ref{coro:equivalence}), there exist a diagonal, differentiable and monotonic map $h_3$ such that 
$$  H^*(x,h_2(n)) = H(x,h_3(h_1(n)))~~ \mathbb{P}_{\bm{X}}\otimes (P\circ h_2^{-1}\#\mathbb{P}^*)~\text{a.s.}$$
Then applying Lemma~\ref{lem-eq-dist}, we obtain that  $\mathcal{S}_{\text{fp}}(P, h_1^{-1}\#\mathbb{P}, (x,n)\to H(x,h_1(n)))$ and $\mathcal{S}_{\text{fp}}(P, h_2^{-1}\#\mathbb{P}^*, (x,n)\to H^*(x,h_2(n)))$ have the same causal distributions, from which follows that  $\mathcal{S}_{\text{fp}}(P, \mathbb{P}, H)$ and $\mathcal{S}_{\text{fp}}(P, \mathbb{P}^*, (x,n)\to H^*)$ have the same causal distributions.  
\begin{remark}
Here, we allow a slight abuse of use of  Lemma~\ref{lem-eq-dist} which can be easily extended to the case where the functions are only equal on the outer-product of the marginal distributions.
\end{remark}

Finally, using the exact same argument as in Proposition~\ref{prop:existence-scm}, we deduce that for any $\mathbb{P}_{\bm{N}}\in\mathcal{P}_{cc}(\mathbb{R})^{\otimes d}$, $\mathcal{A}_P^{\text{INV}}(\mathbb{P}_{\bm{N}},\mathbb{P}_{\bm{X}})$ is not empty, and of course that any element in $\mathcal{A}_P^{\text{INV}}(\mathbb{P}_{\bm{N}},\mathbb{P}_{\bm{X}})\subset \mathcal{A}_P^{\text{INV}}(\mathbb{P}_{\bm{X}})$ induces an SCM that has the exact same causal distributions as $\mathcal{S}_{\text{fp}}(P, \mathbb{P}, H)$.

\end{document}